\documentclass[pdflatex,sn-mathphys-num,iicol]{sn-jnl} %
 
\usepackage{amsmath,amssymb,amsfonts}%
\usepackage{amsthm}%
\usepackage{mathrsfs}%
\usepackage[title]{appendix}%
\usepackage{xcolor}%
\usepackage{textcomp}%
\usepackage{manyfoot}%
\usepackage{booktabs}%
\usepackage{algpseudocode}%
\usepackage{listings}%
\usepackage{balance}
\usepackage{mathtools}
\usepackage{tablefootnote} 
\usepackage{multirow} 
\usepackage{graphicx}
\usepackage{subfigure}
\usepackage[ruled,vlined,linesnumbered]{algorithm2e}
\usepackage{adjustbox}
\usepackage{pifont}
\makeatletter
\newcommand{\thickhline}{%
	\noalign {\ifnum 0=`}\fi \hrule height 1pt
	\futurelet \reserved@a \@xhline
} 

\theoremstyle{thmstyleone}%
\newtheorem{theorem}{Theorem}
\newtheorem{proposition}[theorem]{Proposition}%

\theoremstyle{thmstyletwo}%
\newtheorem{lemma}{Lemma} 

\theoremstyle{thmstylethree}%

\begin{document}

\title[Article Title]{Universal Pansharpening Model}

\author[1,3]{\fnm{Hebaixu} \sur{Wang}}\email{wanghebaixu@gmail.com} 
\author*[2,3]{\fnm{Jing} \sur{Zhang}}\email{jingzhang.cv@gmail.com}  
\author[3,4]{\fnm{Haonan} \sur{Guo}}\email{haonan.guo@whu.edu.cn}  
\author[2,3]{\fnm{Di} \sur{Wang}}\email{d\_wang@whu.edu.cn}
\author*[1,3,5]{\fnm{Jiayi} \sur{Ma}}\email{jyma2010@gmail.com}
\author*[2,3]{\fnm{Bo} \sur{Du}}\email{dubo@whu.edu.cn} 
\author[4]{\fnm{Liangpei} \sur{Zhang}}\email{zlp62@whu.edu.cn}   
\affil[1]{\orgdiv{School of Electronic Information}, \orgname{Wuhan University}, \city{Hubei}, \postcode{430072}, \state{Wuhan}, \country{China}}
\affil[2]{\orgdiv{School of Computer Science}, \orgname{Wuhan University}, \city{Hubei}, \postcode{430072}, \state{Wuhan}, \country{China}}
\affil[3]{\orgdiv{Zhongguancun Academy}, \city{Beijing}, \postcode{100094},\country{China}} 
\affil[4]{\orgdiv{State Key Laboratory of Information Engineering in Surveying, Mapping and Remote Sensing}, \orgname{Wuhan University}, \city{Hubei}, \postcode{430072}, \state{Wuhan}, \country{China}} 
\affil[5]{\orgdiv{School of Robotics}, \orgname{Wuhan University}, \city{Hubei}, \postcode{430072}, \state{Wuhan}, \country{China}} 

\abstract{Pansharpening generates the high-resolution multi-spectral (MS) image by integrating spatial details from a texture-rich panchromatic (PAN) image and spectral attributes from a low-resolution MS image. Existing methods are predominantly satellite-specific and scene-dependent, which severely limits their generalization across heterogeneous sensors and varied scenes, thereby reducing their real-world practicality. To address these challenges, we present UniPS, a universal pansharpening model for satellite-agnostic and scene-robust fusion. Specifically, we introduce a modality-interleaved transformer that learns band-wise modal specializations to form reversible spectral affine bases, mapping arbitrary-band MS into a unified latent space via tensor multiplication. Building upon this, we construct a latent diffusion bridge model to progressively evolve latent representations, and incorporate bridge posterior sampling to couple latent diffusion with pixel-space observations, enabling stable and controllable fusion. Furthermore, we devise infinite-dimensional pixel-to-latent interaction mechanisms to comprehensively capture the cross-domain dependencies between PAN observations and MS representations, thereby facilitating complementary information fusion. In addition, to support large-scale training and evaluation, we construct a comprehensive pansharpening benchmark, termed PSBench, consisting of worldwide MS and PAN image pairs from multiple satellites across diverse scenes. Extensive experiments verify that UniPS consistently outperforms state-of-the-art methods, exhibiting superior generalization and robustness across a wide range of tasks.}

\maketitle 
 
\vspace*{-3.7em}

\section{Introduction}\label{sec1} 
The substantial demand for high-quality satellite imagery has driven rapid advances in remote sensing systems~\cite{wu2025semantic,xu2025hipandas,ciotola2024hyperspectral}. Nevertheless, fundamental physical and hardware constraints of satellite sensors prevent the simultaneous acquisition of images possessing both high spatial resolution and rich spectral information. Consequently, most contemporary satellites typically adopt dual-sensor imaging systems, acquiring paired low-resolution multi-spectral (MS) images and high-resolution panchromatic (PAN) images. While MS images provide abundant spectral information across multiple bands at the cost of spatial resolution, PAN images capture fine-grained spatial structures by integrating a broad spectral range into a single channel, albeit with limited spectral diversity. To alleviate this inherent trade-off, pansharpening aims to fuse the complementary information from both modalities, thereby generating images with high spectral fidelity and spatial resolution~\cite{zhou2025toward,wang2025deep,arienzo2022full}. Owing to the superior characteristics of fused images, pansharpening has become a critical technique in remote sensing, supporting a wide range of applications across diverse domains~\cite{wang2023hyperspectral,roy2025classical}.

\begin{figure}[t]
	\centering
	\includegraphics[width=1\linewidth]{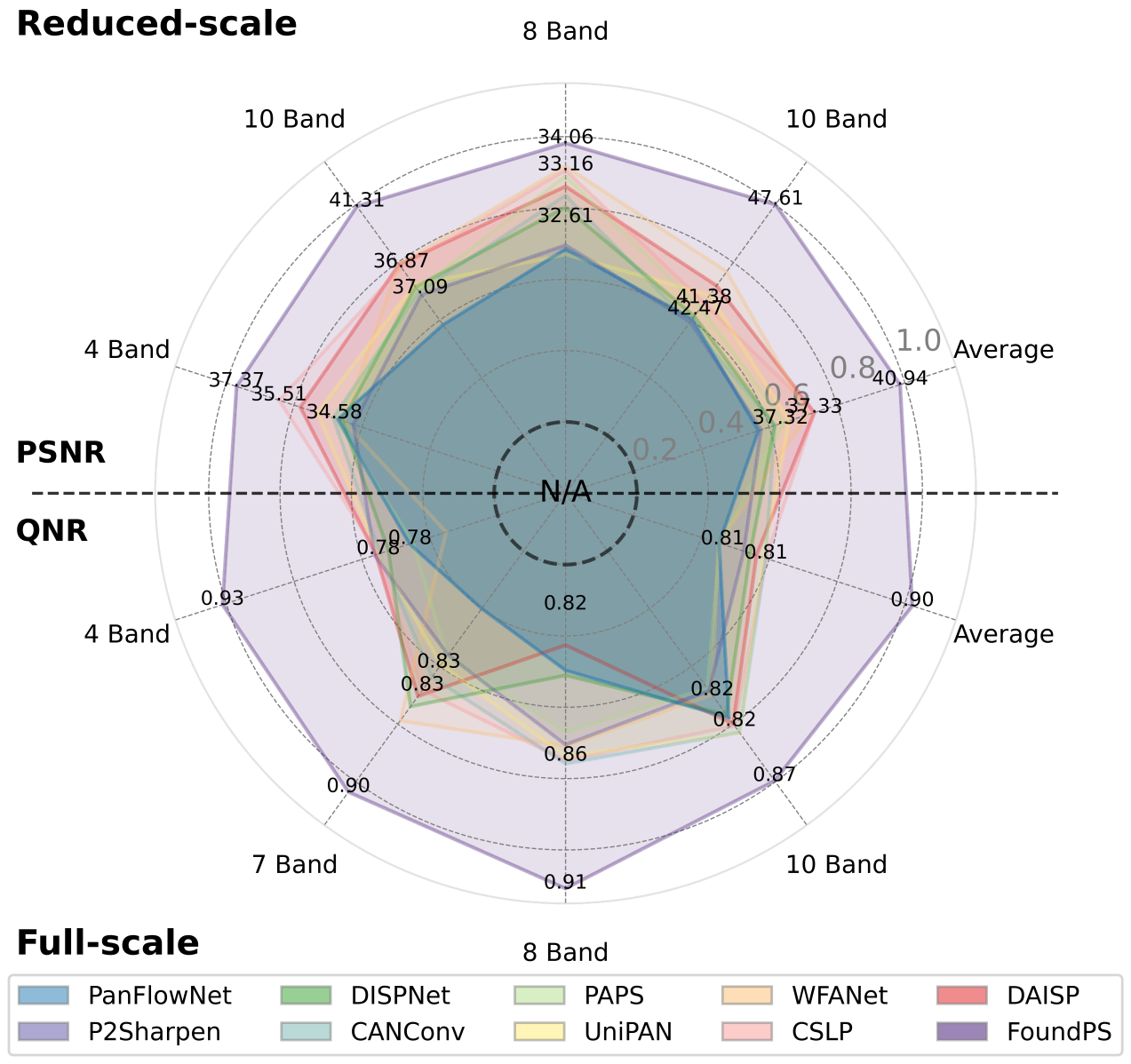}
	
	\caption{Pansharpening performance comparisons on PSBench across representative spectral band configurations (4-, 7-, 8-, and 10-band). The upper and lower panels show the reduced- and full-scale evaluations using PSNR and the non-reference QNR metric, respectively. UniPS consistently outperforms the advanced task-specific models, offering a universal solution for pansharpening tasks. }
	\label{intro:1-0}
\end{figure}

\begin{figure*}[t]
	\centering
	\includegraphics[width=1\linewidth]{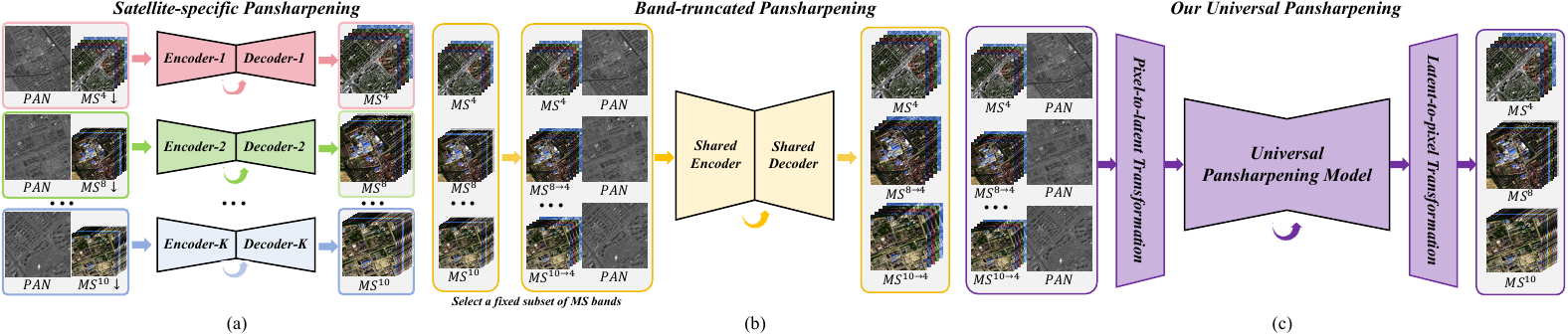}
	
	\caption{Mainstream pansharpening paradigms: (a) Satellite-specific methods employ independent encoder-decoder architectures per spectral configuration, resulting in parameter redundancy and poor cross-satellite scalability. (b) Band-truncated methods unify data format by selecting a fixed band subset, enabling joint training but discarding spectral information and losing the ability to process excluded bands. (c) Our universal method performs band-agnostic fusion by projecting arbitrary-band MS images into a shared latent space, achieving unified modeling without band truncation or parameter duplication, while preserving full cross-band and cross-modal information.}
	\label{intro:1-1}
\end{figure*}

\begin{figure*}[t]
	\centering
	\includegraphics[width=1\linewidth]{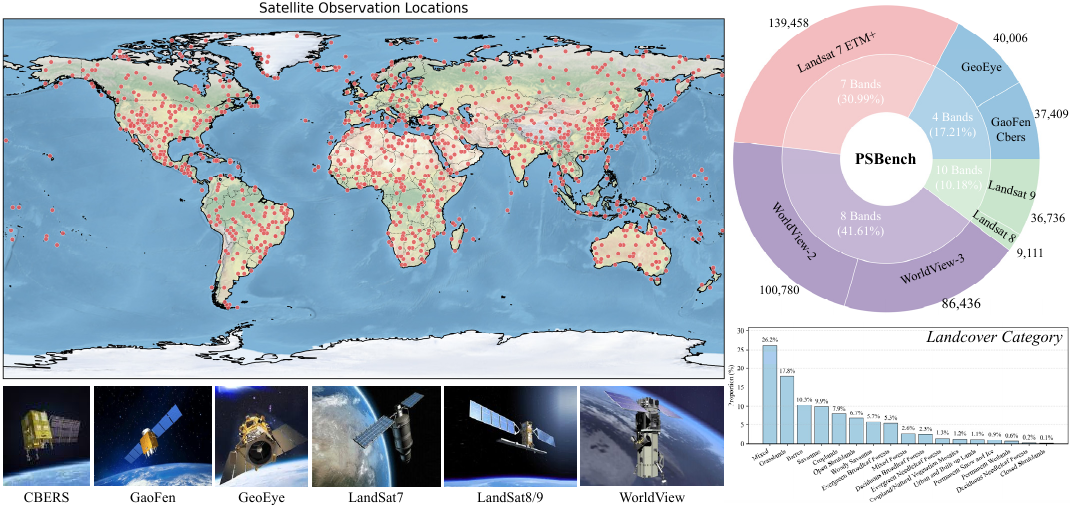}
	
	\caption{Satellite observation locations and distribution of PSBench. It contains four representative spectral configurations for pansharpening tasks with over seventeen landcover categories.} 
	\label{intro:1-2}
\end{figure*}

Traditional pansharpening methods are conventionally categorized into three principal families: component substitution (CS)~\cite{vivone2014critical,vivone2019robust}, multi-resolution analysis (MRA)~\cite{alparone2016spatial}, and variational optimization (VO)~\cite{fang2013variational,fu2019variational,ruder2016overview}. Generally, they all rely on handcrafted rules or priors to model the relationship between modalities, and operate under the assumption that spectral and spatial information can be effectively separated and recombined. Benefiting from explicit formulations, traditional methods typically allow for straightforward adaptation to MS images with varying numbers of spectral bands. However, their performance is fundamentally constrained by the oversimplified cross-modal fusion assumptions, which suffer from inevitable drawbacks in simultaneously preserving spectral fidelity and enhancing spatial details. 

So far, the advent of deep learning has led to significant advancements in pansharpening tasks~\cite{zhang2023panchromatic,guo2025better}. Existing deep-learning-based methods can be generally classified into two paradigms, as illustrated in Fig.~\ref{intro:1-1}. Most approaches adopt independent encoder-decoder architectures built upon convolutional neural networks (CNN)~\cite{zhang2023p2sharpen,wang2025forgotten}, generative adversarial networks (GAN)~\cite{xu2023upangan,shang2024mft}, vision transformers (ViT)~\cite{hou2024linearly}, space state model (SSM)~\cite{he2025pan,hou2025mambamtl}, or their hybrid variants. By incorporating sophisticated spectral-spatial interaction mechanisms, these models have substantially achieved improvements in the preservation of both spatial information and spectral fidelity. Nevertheless, these methods remain predominantly satellite-specific, exhibiting limited generalization capabilities across heterogeneous sensors and varying imaging conditions, as shown in Fig.~\ref{intro:1-1}(a). Specifically, models trained on fixed spectral band configurations are not applicable to different band settings. Moreover, even under identical band numbers, fusion performance frequently deteriorates when applied to unseen scenes or incompatible satellite data. To mitigate these limitations, the zero-shot learning paradigm~\cite{wang2024zero,xiao2025hyperspectral,cao2024zero} reformulates pansharpening as image-pair-specific fusion tasks, wherein a dedicated model is tailored for each source image pair to inherently accommodate arbitrary spectral dimensionalities. However, it generally yields suboptimal performance without large-scale training. Alternatively, the band-truncated pansharpening paradigm~\cite{cui2025enpowering,zhang2025rethinking} attempts to improve cross-sensor generalization by selecting a fixed subset of spectral bands and discarding others to enforce a unified data format, as presented in Fig.~\ref{intro:1-1}(b). While this strategy offers improved generalization ability for sensors sharing the selected bands, it fundamentally fails to adapt to the spectral bands that are excluded. In consequence, its applicability is restricted to a narrow range of sensors and scenarios, lacking true universality and compromising the completeness of spectral information. Beyond the limitations of model architectures, current open-source datasets also impede the assessment and development of generalizable pansharpening models. Most widely used datasets~\cite{xiong2020large,meng2020large,cao2026cross} are collected from a small number of satellites, a specific range of geographic regions, or a limited set of relatively homogeneous scenes, with training and test splits typically constructed via random partitioning, manual truncating, and Wald's protocol~\cite{wald1997fusion}. Such random splitting within a limited data regime induces overlap between training and test distributions, which is likely to inflate in-distribution reconstruction accuracy while offering limited evidence for real-world generalization. Moreover, this constrained data regime further hinders the emergence of robust generalization ability, as models may rely on dataset-specific statistics rather than learning transferable fusion priors. In conclusion, these persistent challenges in cross-sensor generalization and spectral adaptability underscore the urgent necessity for a universal pansharpening model. Such a model should be capable of handling arbitrary spectral bands, generalizing across heterogeneous satellite sensors, and maintaining robust performance on unseen geographic scenes. However, the advancement towards a truly universal pansharpening model is impeded by three crucial challenges. First, the absence of a unified, band-agnostic modeling paradigm prevents existing methods from accommodating arbitrary spectral configurations, rendering them inherently incompatible with different satellites. As a result, practical deployment requires training and maintaining separate models for individual sensors, severely limiting scalability. Second, current methods often suffer from pronounced performance degradation on unseen scenes and incompatible satellite data due to inherent inductive biases even when applied to matched spectral configurations, indicating limited generalization ability. Third, the lack of worldwide large-scale, multi-sensor datasets further constrains the development and assessment of pansharpening, as existing datasets are restricted to specific regions and lack geographic diversity.

To address these challenges, we propose the \textbf{UniPS}, a universal pansharpening  model for satellite-agnostic and scene-robust fusion. Fundamentally, we conceptualize pansharpening into three stages, namely unifying band-agnostic MS representations, enabling effective cross-modality interaction, and restoring desired high-resolution MS images. Firstly, we present a modality-interleaved transformer (MiT) that ensembles mixture-of-experts to generate band-wise modal specializations as a set of spectral affine bases, which are concatenated rather than conventionally weighted aggregation to form reversible mapping matrices. Through tensor multiplication, MS images with arbitrary bands are deterministically projected into a unified latent space, thereby achieving band-agnostic representation with consistent dimensionality. Secondly, we develop the latent diffusion bridge model (LDBM) to progressively evolve the latent representation toward high quality. Furthermore, we incorporate the bridge posterior sampling (BPS) strategy, which tightly couples latent diffusion with pixel observations to provide manifold-constrained fusion guidance. This simultaneously reduces discretization error and shortens the sampling trajectories, and enables training-free adaptation to varied sensors and scenes, thereby enhancing generalization capability. Thirdly, we devise an infinite-dimensional pixel-to-latent interaction mechanism to comprehensively capture the cross-modal dependencies between the pixel-space PAN observation and the latent-space MS representation via Hadamard Product modulated by geometric kernel and exponential kernel, effectively promoting the complementary fusion of spectral and spatial information. In addition, we construct PSBench, a large-scale pansharpening benchmark comprising worldwide MS-PAN image pairs from multiple satellites across diverse geographic scenes (Fig.~\ref{intro:1-2}), to support systematic training and evaluation of universal pansharpening models. Extensive experiments verify that UniPS consistently outperforms state-of-the-art methods (Fig.~\ref{intro:1-0}), exhibiting superior generalization capability and robustness across a wide range of pansharpening tasks. Our contributions are summarized below:

\begin{itemize}
	\item[-] We design a modality-interleaved transformer that generates modal specializations to form spectral affine bases. These bases constitute the reversible mapping matrices that deterministically project MS images with arbitrary spectral bands into a unified latent space, enabling band-agnostic representation learning.
	\item[-] A latent diffusion bridge model equipped with bridge posterior sampling is developed to progressively refine a high-quality representation. This jointly reduces discretization error, accelerates sampling by tightly coupling latent diffusion with pixel observations, and enables training-free adaptation to boost generalization across varied scenes and satellite data. 
	\item[-] To extremely facilitate the cross-modal dependencies between pixel PAN observations and latent MS representations, we devise an infinite-dimensional interaction mechanism. It models feature interactions via a Hadamard product modulated by a geometric kernel and an exponential kernel, enabling complementary fusion of spectral and spatial information.
	\item[-] We construct PSBench, a large-scale pansharpening benchmark comprising 450K worldwide MS-PAN image pairs from multiple satellites across diverse geographic scenes, to fully support the comprehensive training and evaluation of universal pansharpening  models.
\end{itemize}


\section{Related Work}\label{sec:2}

\subsection{Mixture-of-Experts System} 
Mixture of experts (MoE)~\cite{cai2025survey} framework operates on a simple yet powerful principle that different parts of a model specialize in different tasks or aspects of data~\cite{xu2023vitpose++}. Typically, it consists of multiple feed-forward layers with each serving as an expert, and a gating network that dynamically selects and combines expert outputs. Dense MoE activates all experts and aggregates their outputs to obtain favorable task performance but incurs a significant increase in computational overhead, such as EvoMoE~\cite{nie2021evomoe}, MoLE~\cite{wumixture}, and DSMoE~\cite{pan2024dense}. In contrast, sparse MoE computes a weighted combination of only the top-$k$ experts selected by the gating network, achieving a balanced trade-off between performance and efficiency~\cite{rajbhandari2022deepspeed,wei2024skywork} and thus becoming the dominant paradigm. Another focal point of MoE is the design of the gating mechanism. Lewis~\textit{et al.}~\cite{lewis2021base} introduce the balanced assignment of sparse expert layers to equalize workload across experts, improving individual specialization and model stability. Besides, many studies~\cite{jiang2024mixtral,cai2025survey,lenz2025jamba} incorporate an auxiliary loss function to promote a uniform allocation of tokens among experts, thereby enforcing the workload balance. Additionally, several MoE variants have been explored to improve model scalability and adaptability. PanGu-$\Sigma$~\cite{ren2023pangu} presents the random routed experts mechanism, which initially routes tokens to a domain-specific expert group, followed by a random selection within that group. Hmoe~\cite{wang2025hmoe} introduces a heterogeneous MoE framework to address varying token complexities by utilizing different expert networks with diverse capabilities. UniMoE~\cite{li2025uni} employs modality-specific encoders and a sparse MoE architecture to efficiently handle multiple modalities.

Generally, current MoE primarily utilizes data-specific experts to process inputs within a unified format, and aggregates their outputs via weighted addition. Its potential to handle data across heterogeneous modalities with varying formats remains unexplored. In this work, we extend MoE to adaptively align the representations of MS inputs with varying bands into a unified latent space, thereby enabling universal pansharpening.

\subsection{Traditional Pansharpening} 
Traditional pansharpening methods are generally categorized into three groups, namely component substitution, multi-resolution analysis, and variational optimization~\cite{zhang2022progress}. Classical CS-based methods, including the intensity hue-saturation (IHS) fusion method~\cite{tu2001new}, Brovey transformations~\cite{lolli2017haze}, and Gram-Schmidt (GS) orthogonalization method~\cite{aiazzi2007improving}, typically project source image pairs into a shared domain and substitute spatially enhanced components to improve spatial details. Such linear substitutions are computationally efficient, but often induce spectral distortions. MRA-based methods extract the fine-grained spatial information from PAN images using multi-scale decomposition or filtering operations, and inject the high-frequency components into upsampled MS images to preserve spectral information, such as wavelet-based fusion~\cite{alparone2016spatial,otazu2005introduction}, Laplacian pyramid methods~\cite{pradeep2013implementation}, and smoothing filter-based methods~\cite{vivone2014critical}. However, their performance is highly sensitive to the choice of decomposition scales and filters, and spectral inconsistencies may still arise in complex scenes. VO-based methods explicitly model the PAN-MS relationship and formulate pansharpening as an inverse problem solved via iterative optimization under handcrafted priors. For instance, P+XS~\cite{ballester2006variational} assumes the PAN image to be a linear combination of MS bands. SIRF~\cite{chen2015sirf} introduces gradient sparsity as a regularization prior to enhance structural consistency. BAGDC~\cite{lu2021unified} constructs band-adaptive gradient constraints to refine the spatial structures. Despite the strong interpretability, they often require empirical tuning of multiple regularization terms and involve computationally expensive iterative procedures, which can compromise both performance and practical efficiency. Concretely, traditional methods are adaptable to varying spectral bands by tuning hyperparameters, but they are fundamentally constrained by oversimplified fusion assumptions that compromise either spectral or spatial fidelity.

\begin{table*}[t]
	\centering
	\caption{Details of the collected multi-spectral and panchromatic image pairs from different satellites. }
	\label{tab:psbench}	 
	\renewcommand\arraystretch{1.0}
	\resizebox{2.10\columnwidth}{!}{
	\begin{tabular}{ccccccccccc}
		\toprule
		\multirow{2}{*}{{Satellite/Sensor}} & Launch &  Launch &{MS Spatial} & {PAN Spatial} & Multi-Spectral & Panchromatic & Number of  & \multirow{2}{*}{{Order of MS Bands}} &{MS Pixel} & {PAN Pixel}   \\  
		& Date & Country  & Resolution & Resolution & Range & Range & MS Bands &   & Size & Size \\
		\midrule
		GaoFen-1\footnotemark[1]& 2013 & China & 8.0 m & 2.0 m & 450-890 nm & 450-900 nm & 4 & B/G/R/NIR & 256 & 1024   \\
		GaoFen-2\footnotemark[1]& 2014 & China & 3.2 m & 0.8 m & 450-890 nm & 450-900 nm & 4 & B/G/R/NIR & 256 & 1024    \\
		GaoFen-6\footnotemark[1]& 2018 & China & 8.0 m & 2.0 m & 450-900 nm & 450-900 nm  & 4 & B/G/R/NIR & 256 & 1024    \\
		GaoFen-7\footnotemark[1]& 2020 & China & 2.6 m & 0.65 m & 450-900 nm & 450-900 nm & 4 & B/G/R/NIR & 256 & 1024   \\
		GeoEye-1\footnotemark[2]& 2008 & USA  & 1.64 m& 0.41m & 450-920 nm & 450-800 nm & 4 & B/G/R/NIR & 256 & 1024   \\
		CBERS-04A\footnotemark[1]& 2019 & China & 8.0 m & 2.0 m & 450-890 nm & 450-900 nm   & 4 & B/G/R/NIR & 256 & 1024  \\
		Landsat7 ETM+\footnotemark[3] & 1999 & USA & 30 m & 15m & 450-12500 nm & 520-900 nm & 7 & B/G/R/NIR/SWIR/TIR/MIR & 512 & 1024 \\
		WorldView-2\footnotemark[2]&2009 & USA & 1.84 m & 0.46 m & 400-1040 nm & 450-800 nm & 8 & CB/B/G/Y/R/RE/NIR1/NIR2 & 256 & 1024 \\
		WorldView-3\footnotemark[2]&2014 & USA & 1.24 m & 0.31 m & 400-1040 nm & 450-800 nm & 8 & CB/B/G/Y/R/RE/NIR1/NIR2 & 256 & 1024   \\
		Landsat8\footnotemark[3]& 2013 & USA  &30 m & 15 m & 430-12510 nm & 500-680 nm & 10   &  CA/B/G/R/NIR/SWIR1/SWIR2/Cirrus/TIR1/TIR2 & 512 & 1024 \\
		Landsat9\footnotemark[3]& 2021& USA  & 30 m & 15 m & 430-12510 nm & 500-680 nm & 10  &  CA/B/G/R/NIR/SWIR1/SWIR2/Cirrus/TIR1/TIR2  & 512 & 1024  \\
		\bottomrule
	\end{tabular}} 
	\smallskip
	
	\scriptsize{Band abbreviation: \textcolor{blue}{B}: Blue, \textcolor{green}{G}: Green, \textcolor{red}{R}: Red, \textcolor{cyan}{NIR}: Near-infrared, \textcolor{magenta}{CB}: Coastal Blue, \textcolor{yellow}{Y}: Yellow, \textcolor{orange}{RE}: Red Edge, \textcolor{purple}{SWIR}: Short-wave Infrared, \textcolor{orange}{TIR}: Thermal, \textcolor{teal}{MIR}: Mid-Infrared, \textcolor{brown}{CA}: Coastal Aerosol.}
\end{table*}

\subsection{Deep learning Pansharpening} 
With the powerful non-linear capacity of deep learning, pansharpening has witnessed substantial advancements~\cite{pereira2026multi,zhang2023stp}. PNN~\cite{masi2016pansharpening} devises the first pansharpening network with three simple convolution layers, achieving the superior performance of that time. PanNet~\cite{yang2017pannet} incorporates the residual connection to inject spatial information into the resampled spectral information, thereby preserving the spectral distribution as well as spatial details. Afterwards, lots of sophisticated structures~\cite{zhou2022panformer,he2024pan,cao2024diffusion} are investigated to fully explore the nonlinear capacity of the network for superior pansharpening quality~\cite{wang2025deep,zhou2024general,ma2025deep,chen2025cslp,li2025enhanced}. Inspired by the conventional techniques, Hu \emph{et al.}~\cite{hu2020pan} built a multi-scale network to dynamically and locally extract the information from source images. Yang \emph{et al.}~\cite{yang2022memory} embedded a denoising-based prior and a non-local auto-regression prior to improve the long-range coherence, and combined the merits of model optimization with deep learning. Particularly, Ma \emph{et al.}~\cite{ma2020pan} took the first attempt to utilize the generative adversarial network for an unsupervised pansharpening paradigm. Multiple discriminators are introduced to assist the generator in jointly learning the reasonable spectral and spatial distributions~\cite{xu2023upangan,zhong2024ssdiff}. ZeroSharpen~\cite{wang2024zero} and ZSPan~\cite{cao2024zero} break the satellite-specific learning paradigm and explore the image-pair-specific fusion strategies to improve the practicality of deep models when processing the insufficient and unknown data. Recent efforts have increasingly focused on improving the fusion generalization across heterogeneous satellites and diverse scenes. Some works~\cite{zhang2025rethinking} employ a band-truncation strategy to enforce a unified input format by selecting a fixed subset of spectral bands and discarding others, thereby enabling models to be optimized on samples from different satellites. UniPAN~\cite{cui2025enpowering} further projects these standardized data into a shared distribution, reducing the learning difficulty of the network and improving the cross-task adaptability. Additionally, Cui \emph{et al.}~\cite{cui2025leveraging} simulate datasets for pansharpening pre-training, boosting the generalization capacity.

In summary, existing methods are incapable of handling MS inputs with varying spectral configurations from different satellites, rendering them sensor-specific and non-transferable. And the lack of large-scale, real-world datasets further constrains generalization across diverse scenes. To the best of our knowledge, UniPS is the first model designed for truly universal pansharpening. Besides, PSBench is the first comprehensive multi-satellite benchmark constructed at a global scale to facilitate rigorous evaluation and future development in this field.

\section{Methodology}\label{sec:3}

\footnotetext[1]{\url{https://data.cresda.cn}} 
\footnotetext[2]{\url{https://maxar-opendata.s3.amazonaws.com/}}
\footnotetext[3]{\url{https://www.usgs.gov/landsat-missions/}}

\begin{figure*}[t]
	\centering
	\includegraphics[width=1.0\linewidth]{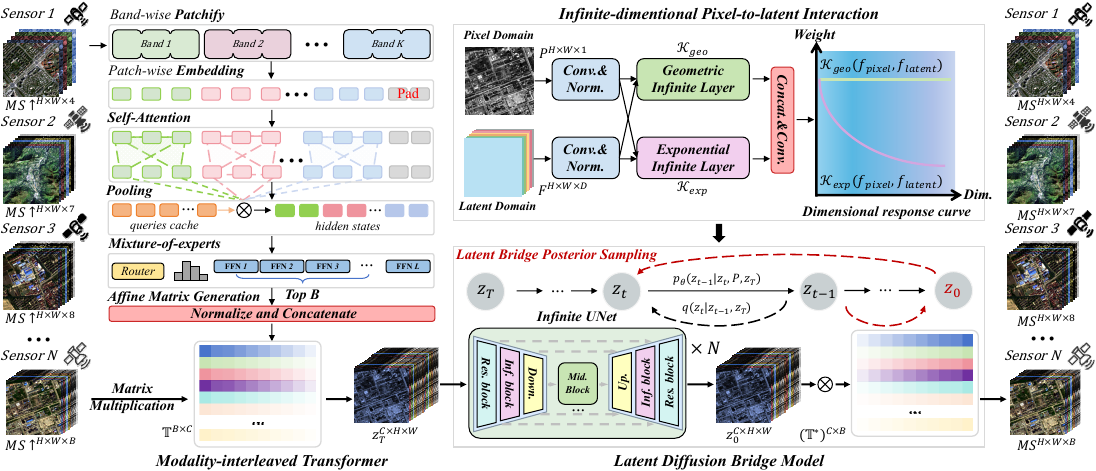} 
	\caption{The overall framework of our method. Multiple satellite observations with arbitrary spectral bands are first unified into a latent space through the modality-interleaved transformer. A latent diffusion bridge model is established for progressive latent refinement through a sequence of iterative updates, whose backbone is a tailored Infinite U-Net empowered by an infinite-dimensional interaction mechanism. This design seamlessly injects complementary spatial cues extracted from PAN images into the latent representations. These enhanced latent representations are projected back to the image domain to reconstruct the high-resolution MS images with flexible spectral dimensions.}
	\label{method:1-1}
\end{figure*}
 
\subsection{Data Source} 
In the light of significant criteria including task diversity, satellite sensors, global coverage, and data accessibility, we select several series of representative satellites to serve as the data sources, as reported in Tab.~\ref{tab:psbench}. Specifically, to ensure task diversity, we consider pansharpening scenarios involving four-, seven-, eight-, and ten-band spectral configurations, with each exhibiting distinct band orders and spectral ranges. These scenarios present widely varying spatial resolutions, ranging from submeter to tens of meters, and cover spectral wavelengths from visible to thermal infrared, thereby constructing a highly heterogeneous imaging environment. Temporally, the data span from 1999 to 2025, covering multiple generations of satellite sensors. Geographically, samples are globally distributed across seventeen classical land-cover categories, ensuring broad environmental and scene diversity, as illustrated in Fig.~\ref{intro:1-2}. Moreover, the benchmark incorporates both freely accessible and commercial satellite data, balancing openness with data quality to support reproducible yet practical research. These criteria are designed to challenge and enable a model to handle pansharpening tasks under varying spectral configurations and scene complexities. All data utilized in this work have been properly licensed.

\subsection{Data Acquisition and Processing}
We collect the source image pairs from both open-access platforms and commercial satellite providers to ensure data quality and diversity. To maintain atmospheric clarity, only samples with cloud and haze coverage below 5\% are retained, while all others are discarded, thereby guaranteeing that the remaining images exhibit high radiometric fidelity and minimal atmospheric distortion. Additionally, we crop all PAN images to $1024\times 1024$ patches, then crop the corresponding MS images according to the precise spatial ratio. To enable supervised training, we simulate the reduced-scale MS/PAN image pairs following the Wald protocol~\cite{wald1997fusion}. This pipeline yields PSBench, a large-scale, globally-covered pansharpening dataset comprising 449,936 aligned MS/PAN patch pairs with a total of 1.68 TB (Fig.~\ref{intro:1-2}). As for evaluation, we randomly select 1091, 463, 561, and 1029 image pairs from the entire PSBench dataset as test sets for the 4-, 7-, 8-, and 10-band pansharpening tasks, respectively.

\subsection{Method Overview} 
UniPS conceptualizes pansharpening into three fundamental stages: unified representation, fusion, and reconstruction, as illustrated in Fig.~\ref{method:1-1}. First, the MS image is upsampled to align its spatial resolution with the PAN image, and then fed into a modality-interleaved transformer (MiT). It undergoes band-wise patchification and patch-wise embedding, followed by self-attention and pooling to obtain hidden states, which are then input into the mixture-of-expert system. A router generates probability weights over experts, and a subset of experts corresponding to the number of MS bands is dynamically selected. These selected expert outputs, representing spectral affine bases, are normalized and concatenated to form a reversible mapping tensor. Through tensor multiplication, MS images with arbitrary bands are deterministically projected into a unified latent space, thereby achieving band-agnostic representation with consistent dimensionality. Second, the core cross-modal fusion is performed by an infinite-dimensional feature interaction mechanism, which serves as the pivotal component of the Infinite-UNet. This mechanism comprehensively integrates PAN features with the evolving latent MS representations via a Hadamard product, modulated by complementary geometric and exponential kernels. The exponential kernel assigns decaying weights across dimensions, emphasizing low-order interactions, while the geometric kernel assigns uniform weights, enabling full-order interactions. Third, a latent diffusion bridge iteratively evolves the coarse latent representations toward higher quality. During diffusion inference, we seamlessly integrate a bridge posterior sampling strategy, which provides manifold-constrained gradients to guide the fusion process. This not only enhances representation quality but also allows training-free fusion adaptation across different satellites and scenes. Finally, the refined latent MS representation is projected back to the image space via the inverse of the mapping tensor, thereby reconstructing the desired MS image. 

\subsection{Modality-interleaved Transformer} 
For the first vital stage, we introduce a MiT to map MS images with arbitrary spectral bands into a band-agnostic unified latent space. Let $x \in \mathbb{R}^{h\times w\times B}$ denote the $B$-band MS image with a height of $h$ and a width of $w$. The upsampled MS image is defined as $x\uparrow \in \mathbb{R}^{H\times W\times B}$, which matches the spatial resolution of PAN image $P \in \mathbb{R}^{H\times W\times 1}$ with size of $H$ and $W$. MiT takes $x\uparrow$ as input and performs band-wise patchification by independently dividing each band into non-overlapping patches. All patches are then embedded to obtain patch embedding $\mathit{E}\in \mathbb{R}^{N\times p^2}$:
\begin{equation}
	\mathit{E} = \text{PatchEmbed}(x\uparrow),
\end{equation}
where $N\!=\!B\!\times\! \frac{H}{p} \!\times\! \frac{W}{p}$ is the number of patches with size $p$. These patch embeddings are augmented with positional encodings and then processed by a Transformer-based self-attention~\cite{han2022survey} mechanism to extract hidden states $H_s\in \mathbb{R}^{N\times p^2}$:
\begin{align}
	&Q = EW_Q,\, K = EW_K, V = EW_V\\
	&H_s = \text{softmax}(\frac{QK^T}{\sqrt{d_k}}) V,
\end{align}
Where $W_{(\cdot)}$ is a linear layer. To aggregate band-wise features and produce a compact representation for expert routing, we employ learnable-query attention pooling. A learnable query vector $q\in \mathbb{R}^{p^2}$ is replicated to form a query matrix $Q_s \in \mathbb{R}^{N\times p^2}$, which interacts with hidden states via cross-attention:
\begin{equation}
	H_{p} = Q_s + \text{softmax}(\frac{Q_sH_s^T}{\sqrt{d_k}})H_s.
\end{equation}
Then $H_p\in \mathbb{R}^{N\times p^2}$ is condensed into a compact vector $h_p\in \mathbb{R}^{p^2}$ via a summation over the first dimension. Then $h_p$ is fed into a router network $\mathcal{G}$ to produce scores over total $L$ experts and determine a top-$B$ expert cache $\Omega_{B}$:
\begin{equation}
	\Omega_{B} = \text{TopB}({\text{softmax}(\mathcal{G}(h_p))}).
\end{equation}
Afterwards, the selected experts are activated. Each expert $\mathcal{E}_i$ outputs a spectral affine basis $T_i\in \mathbb{R}^{1\times C}$. These bases are concatenated to form a mapping tensor $\mathbb{T}\in \mathbb{R}^{B\times C} (C > B)$:
\begin{equation}
	\mathbb{T} = [T_1||T_2||\cdots||T_B], T_i = \mathcal{E}_i(H_s), i \in \Omega_{B}.
\end{equation}
To maintain stability and avoid rank deficiency, we parameterize the mapping tensor $\mathbb{T}$ with the Cayley formulation, thereby ensuring well-posed projection while preserving the learned spectral transformations. Specifically, $\mathbb{T}$ is divided into two column-wise blocks:
\begin{equation}
	\langle \Delta_1,\Delta_2  \rangle=\mathbb{T}
\end{equation}
where $\Delta_1\in \mathbb{R}^{B\times B}$ serves as the principal spectral transformation basis and $\Delta_2\in \mathbb{R}^{B\times (C-B)}$ provides additional latent expansion dimensions. Subsequently, a Cayley transformation is applied to obtain an orthogonal basis:
\begin{equation}
	\widetilde{\Delta}_1 = (I_{d} + (\Delta_1-\Delta_1^T))^{-1}(I_{d} - (\Delta_1-\Delta_1^T)),
\end{equation}
where $I_{d} \in \mathbb{R}^{B\times B}$ is the identity matrix, $(\cdot)^T$ is the matrix transpose, and $\widetilde{\Delta}_1 \in\mathbb{R}^{B\times B}$ satisfies $\widetilde{\Delta}_1^{H}\widetilde{\Delta}_1 =I$ (See Appendix~\ref{sec:suppl_A}), thereby providing a stable full-rank spectral basis. The final mapping tensor is constructed by concatenating the orthogonal basis with the learned latent expansion basis, which can be expressed as:
\begin{equation}
	\mathbb{T}=[\widetilde{\Delta}_1||\Delta_2]\in\mathbb{R}^{B\times C}.
\end{equation}
The MS image can be deterministically projected into a latent representation $\mathbf{Z}\in \mathbb{R}^{H\times W \times C}$ with fixed dimensionality $C$ via tensor multiplication:
\begin{equation}\label{MiT:eq6}
	\mathbf{Z} = (x\uparrow) \mathbb{T}.
\end{equation} 
Notably, $\mathbb{T}$ maintains full row rank, guaranteeing $\mathbb{T} \mathbb{T}^{\ast} = I_{d}$, where $\mathbb{T}^{\ast}$ is the pseudo inverse of $\mathbb{T}$. By dynamically activating expert subsets according to input spectral characteristics, MiT can adaptively model diverse spectral configurations with specialized bases, preserving modality-specific information and balancing the extended modality-agnostic dimension across heterogeneous sensors. This design establishes a unified latent space for subsequent fusion stages. 
 
\subsection{Latent Diffusion Bridge Model}
Benefiting from the MiT, the learned latent space exhibits clear modality-aware separation, where different spectral bands are mapped into distinct yet structured latent distributions (See Fig.~\ref{notion:1-3}). Such a discriminative latent space preserves the intrinsic spectral identity of each modality and balances the heterogeneous spectral representations. Consequently, the latent representations of low-quality and high-quality observations maintain the band-dependent correspondence. Therefore, we employ a latent diffusion bridge model (LDBM) to construct a point-to-point probabilistic path between the low-quality latent distribution $p_{LQ}(\mathbf{Z})$ and the high-quality latent distribution $p_{HQ}(\mathbf{Z})$. Let $\mathbf{z}_t$ be a finite random variable governed by the generalized diffusion bridge process in~\cite{wang2025residual,kieu2025bidirectional}, with terminal condition $\mathbf{z}_T$. The evolution of its marginal distribution $p(\mathbf{z}_t \!\mid\! \mathbf{z}_T)$ satisfies the following stochastic differential equation under a fixed stationary coefficient ratio $\lambda$ with noise schedule $\theta_t$ (See Appendix.~\ref{sec:suppl_B}):
\begin{equation}\label{ldbm_eq1}
	d\mathbf{z}_t = \theta_t\coth(\overline{\theta}_{t:T})(\mathbf{z}_T-\mathbf{z}_t)dt + \sqrt{2\lambda \theta_t} d\omega_t,
\end{equation}
where $\overline{\theta}_{s:t} = \int_s^t \theta_\tau d\tau$ and $\omega_t$ is the standard Wiener process. Given an initial state $\mathbf{z}_0$, the analytical solution of $\mathbf{z}_t$ at time $0<t<T$ of that SDE in Eq.~\eqref{ldbm_eq1} satisfies a Gaussian distribution with expectation $\mathbb{E}[\mathbf{z}_t]$ and variance $Var[\mathbf{z}_t]$:
\begin{equation}
	\!\mathbb{E}[\mathbf{z}_t] = \mathbf{z}_T \!+\! (\mathbf{z}_0\!-\!\mathbf{z}_T)\Theta_t, \, \Theta_t \! \eqcolon \! \frac{\sinh(\overline{\theta}_{t:T})}{\sinh(\overline{\theta}_{0:T})},  \label{sec:method_eq4}	
\end{equation}
\begin{equation}
	\!Var[\mathbf{z}_t]  = 2\lambda\frac{\sinh(\overline{\theta}_{0:t})\sinh(\overline{\theta}_{t:T})}{\sinh(\overline{\theta}_{0:T})} \coloneqq \Sigma_t^2. \label{sec:method_eq5}\\
\end{equation}
Apparently, the initial state at $t = 0$ is a high-quality representation $\mathbf{z}_0$, while the terminal state at $t=T$ is a low-quality representation $\mathbf{z}_T$. From Eq.~\eqref{sec:method_eq4}-\eqref{sec:method_eq5}, the transition probability distributions from $\mathbf{z}_0$ to intermediate states $\mathbf{z}_t$ and its adjacent states $\mathbf{z}_{t-1}$ are:
\begin{align}
	\!\!q(\mathbf{z}_t\vert \mathbf{z}_0,\mathbf{z}_T) &\!=\! \mathcal{N}(\mathbf{z}_T\! + \!(\mathbf{z}_0 \!-\! \mathbf{z}_T)\Theta_{t},\Sigma^2_{t}\boldsymbol{I}),\label{sec:method_eq9} \\
	\!\!\!q(\mathbf{z}_{t\!-\!1}\vert \mathbf{z}_0,\mathbf{z}_T) &\!=\! \mathcal{N}(\mathbf{z}_T \! + \!(\mathbf{z}_0 \!-\! \mathbf{z}_T)\Theta_{t\!-\!1},\Sigma^2_{t\!-\!1}\boldsymbol{I}), \label{sec:method_eq10} 
\end{align}
where $\boldsymbol{I}$ is standard Gaussian noise. Assuming that? sampling from $\mathbf{z}_t$ to $\mathbf{z}_{t-1}$ follows the Gaussian distribution, the deterministic sampling of the reverse process is:
\begin{align}
	\!\! \mathbf{z}_{t\!-\!1} &\!=\! \mathbf{z}_T \!+\! {\frac{\Sigma_{t\!-\!1}}{\Sigma_{t}}}(\mathbf{z}_t \!-\!\mathbf{z}_T) \!+\! (\Theta_{t\!-\!1} \!-\! \Theta_{t} {\frac{\Sigma_{t\!-\!1}}{\Sigma_{t}}})(\widehat{\mathbf{z}}_0^{t} \!-\!\mathbf{z}_T), \nonumber \\
	& \!=\! \mathbf{z}_T \!+\! \frac{\Theta_{t\!-\!1}}{\Theta_t}(\mathbf{z}_t \!-\! \mathbf{z}_T) \!-\! ({\frac{\Theta_{t\!-\!1}}{\Theta_{t}}\Sigma_{t} \!-\! \Sigma_{t\!-\!1}}) \widehat{\epsilon}_t,\label{sec:method_eq12}
\end{align}  
where $\widehat{\mathbf{z}}_0^{t}$ and $\widehat{\epsilon}_t$ are the high-quality representation and noise at a certain state predicted by network $\mathcal{U}(z_t,z_T,P,t)$, respectively. Essentially, these two prediction modes are equivalent.

\subsection{Bridge Posterior Sampling}
Tweedie's formulas~\cite{efron2011tweedie} unveil that $\epsilon_t$-prediction is equivalent to estimating the score function, i.e., the gradient of the log-density of the data distribution. For the case of LDBM sampling, $\widehat{\epsilon}_t$ in Eq.~(\ref{sec:method_eq12}) can be converted into:
\begin{equation} 
	\widehat{\epsilon}_t = - \Sigma_t \nabla_{\mathbf{z}_t} \log p(\mathbf{z}_t\vert \mathbf{z}_T).
\end{equation}
Under such circumstances, we can further incorporate the \emph{Bayer's theorem} to achieve the diffusion guidance, which is expressed as:
\begin{equation}\label{LBPS:eq1}
	\!\!\!\!\nabla_{\mathbf{z}_t} \! \log p(\mathbf{z}_t\vert \mathbf{z}_T) \! = \! \nabla_{\mathbf{z}_t} \! \log p(\mathbf{z}_t) \!+\! \nabla_{\mathbf{z}_t} \! \log p(\mathbf{z}_T\vert \mathbf{z}_t).
\end{equation}
Eq.~(\ref{LBPS:eq1}) offers the profound insights that $\nabla_{\mathbf{z}_t}\log p(\mathbf{z}_T\vert \mathbf{z}_t)$ can be adapted to prevent samples from deviating from the generative manifolds. However, the likelihood term is in fact analytically intractable due to its dependence on time $t$. To circumvent the intractability, we use the surrogate function to maximize the likelihood, yielding approximate posterior sampling~\cite{chung2023diffusion}. Specifically, we make the approximation (see Appendix.~\ref{sec:suppl_C}).
\begin{equation}\label{LBPS:eq2}
	\nabla_{\mathbf{z}_t} \log p(\mathbf{z}_T\vert \mathbf{z}_t) \approx \nabla_{\mathbf{z}_t} \log p(\mathbf{z}_T\vert \widehat{\mathbf{z}}_0^{t}),
\end{equation}
where $\widehat{\mathbf{z}}_0^{t} = \mathbb{E}[\mathbf{z}_0\vert \mathbf{z}_T]$ is the estimated high-quality latent representation at time $t$. Unfortunately, $\mathbf{z}_0$ and $\mathbf{z}_T$, as the two endpoints of the latent diffusion bridge, exhibit an implicit conditional relationship in the latent space. However, they can be projected back to the pixel space through the pseudo inverse of tensor matrix $\mathbb{T}$, which allows us to reformulate the guidance objective. Mathematically, pansharpening can be viewed as a special case of linear inverse problems for multi-spectral imagery, where a degraded MS measurement $x$ is derived from the desired MS image $y \in \mathbb{R}^{H\times W \times B}$ by below formulation:
\begin{equation} 
	x = y\downarrow + \mathbf{n},
\end{equation}
where $\downarrow$ is the forward operator for spatial degradation and $\mathbf{n}$ is the noise. By combining with Eq.~(\ref{MiT:eq6}), we can obtain:
\begin{equation}
	\mathbf{z}_T = (\mathbf{z}_0 \mathbb{T}^{\ast} \downarrow)\mathbb{T}  + \mathbf{n}_{z},
\end{equation}
Suppose $\mathbf{n}_{z} \sim \mathcal{N}(0,\sigma^2)$ is latent noise and follows Gaussian distribution, then Eq.~(\ref{LBPS:eq2}) can be:
\begin{equation}
	\!\!\nabla_{\mathbf{z}_t}\! \log p(\mathbf{z}_T\vert \widehat{\mathbf{z}}_0^{t}) \!= \!-\frac{1}{\sigma^2} \nabla_{\mathbf{z}_t} \|\mathbf{z}_T \!-\! (\widehat{\mathbf{z}}_0^t\mathbb{T}^{\ast}\downarrow)\mathbb{T}\|_2.
\end{equation}

\begin{algorithm}[t]
	\caption{UniPS: Training}
	\label{Alg:train}
	\KwIn{MS image $x$; PAN image $P$; Reference $y$; MiT $\mathcal{M}(\cdot)$, Infinite-UNet $\mathcal{U}(\cdot)$}
	
	\Repeat{converged}{
		
		$\mathbb{T} = \mathcal{M}(x\uparrow)$, $\mathbf{z}_T =  x\mathbb{T} \uparrow$, $\mathbf{z}_0 = y\mathbb{T}$
		
		$t\sim Uniform(1,\cdots,T)$;
		
		$\epsilon\sim\mathcal{N}(0,\boldsymbol{I})$;
		
		$\mathbf{z}_t = \mathbf{z}_T+ (\mathbf{z}_0-\mathbf{z}_T)\Theta_t + \Sigma_t \epsilon$;
		
		$\widehat{\mathbf{z}}_0^t = \mathcal{U}(z_t,z_T,P,t)$
		
		Take the gradient descent step on
		
		$ \nabla \big( \| \hat{\mathbf{z}}_0^t  -    y \mathbb{T} \|_1 + \|  \hat{\mathbf{z}}_0^t\mathbb{T}^{\ast} - y \|\big) $
		
	}  
	
\end{algorithm}
 
\begin{algorithm}[t]
	\caption{UniPS: Sampling}
	\label{Alg:sample}
	\KwIn{MS image $x$; PAN image $P$; Network $\mathcal{M}(\cdot), \mathcal{U}(\cdot)$; Guidance weight $\eta$.}
	
	$\mathbb{T} = \mathcal{M}(x\uparrow)$, $\mathbf{z}_T = x\mathbb{T}\uparrow $
	
	\For{\textnormal{$t=T-1$ to $1$}}{
		
		$\widehat{\mathbf{z}}_0^t = \mathcal{U}(z_t,z_T,P,t)$
		
		$\!\mathbf{z}_{t-1} \!=\! \mathbf{z}_T \!+\! {\frac{\Sigma_{t\!-\!1}}{\Sigma_{t}}}(\mathbf{z}_t \!-\!\mathbf{z}_T) \!+\! (\Theta_{t-1} \!-\! \Theta_{t} {\frac{\Sigma_{t\!-\!1}}{\Sigma_{t}}})(\widehat{\mathbf{z}}_0^t \!-\!\mathbf{z}_T)$

		$\widehat{\mathbf{z}}_{t-1} = \mathbf{z}_{t-1} + \eta \nabla_{\mathbf{z}_t} \|\mathbf{z}_T -  (\widehat{\mathbf{z}}_0^t\mathbb{T}^{\ast}\downarrow)\mathbb{T}\|_2$
		
	}
	
	$\widehat{y} =  \widehat{z}_{0} \mathbb{T}^{\ast}$
	
	\KwOut{$\widehat{y}$}
\end{algorithm}
\noindent Furthermore, we can derive the closed-form upper bound of the Jensen gap for the approximation in Eq.~(\ref{LBPS:eq2}) as follows:
\begin{equation}
	\!\!\!\mathcal{J}(\sigma,M) \! \le \! \frac{d}{\sqrt{2\pi\sigma^2}} e^{-\frac{1}{2\sigma^2}}||\nabla_{\mathbf{z}_t}  (\widehat{\mathbf{z}}_0^t\mathbb{T}^{\ast}\downarrow)\mathbb{T}||_2 M,
\end{equation}
where $M\!:=\!\int\! \|\mathbb{T}^{\ast}\mathbf{z}_0\!\downarrow \! - \! \mathbb{T}^{\ast}\widehat{\mathbf{z}}_0^t\!\downarrow\!\| d p(\mathbf{z}_0\vert \mathbf{z}_t)$. Note that $M$ is finite for most of the distribution in practice, which can be considered as the generalized absolute distance between the observed reference and estimated data. $\mathcal{J}(\sigma,M)$ can approach 0 as $\sigma\!\rightarrow\!0$ or $\infty$, suggesting that the approximation errors reduce when the measurement noise is extremely small or large. Specifically, if the predictions of $\widehat{\mathbf{z}}_0^t$ are accurate, the upper bound $\mathcal{J}(\sigma,M)\vert_{\sigma \!\rightarrow\! 0,M \!\rightarrow\! 0}$ shrinks due to the low variance and distortion. Oppositely, if the predictions lack accuracy, the upper bound of the $\mathcal{J}(\sigma,M)\vert_{\sigma \!\rightarrow\! \infty, M \!\rightarrow\! M_{max}}$ shrinks as well for large variance and limited distortion. For the inverse inference in Eq.~(\ref{sec:method_eq12}), we can employ bridge posterior sampling to guide the current state by the following modifications:
\begin{align}
	\widehat{\mathbf{z}}_0^t &:= \widehat{\mathbf{z}}_0^t + \eta_{z} \nabla_{\mathbf{z}_t} \|\mathbf{z}_T -  (\widehat{\mathbf{z}}_0^t\mathbb{T}^{\ast}\downarrow)\mathbb{T}\|_2, \label{eq22} \\
	\widehat{\epsilon}_t &:= \widehat{\epsilon}_t + \eta_{\epsilon} \nabla_{\mathbf{z}_t} \|\mathbf{z}_T -  (\widehat{\mathbf{z}}_0^t\mathbb{T}^{\ast}\downarrow)\mathbb{T}\|_2.
\end{align} 
Building on these theoretical foundations, BPS functions as a training-free guidance mechanism that enables flexible control over the fusion effect via a tunable weight $\eta_{(\cdot)}$. Larger values enforce stricter spectral consistency with the MS observation, whereas smaller values preserve more spatial priors, leading to improved structural details. Such a mechanism allows the model to adapt across varying satellite sensors and diverse scenarios, achieving scene-adaptive fusion without retraining. Unlike DPS~\cite{chung2023diffusion}, which performs posterior sampling solely in the image domain, our bridge posterior sampling explicitly leverages the reversible spectral projection learned by MiT. Consequently, the degradation model can be equivalently transferred into the latent domain, allowing posterior guidance to be performed directly on structured latent representations without sacrificing spectral fidelity. Besides, according to the manifold assumption~\cite{li2025back}, image data should lie on a low-dimensional manifold, whereas noised quantities do not. Therefore, we adopt the $\widehat{\mathbf{z}}_0^t$-prediction mode as training and sampling criteria, and the corresponding algorithms are detailed in Alg.~\ref{Alg:train} and Alg.~\ref{Alg:sample}, respectively.

\begin{figure*}[t]
	\centering
	\includegraphics[width=0.99\linewidth]{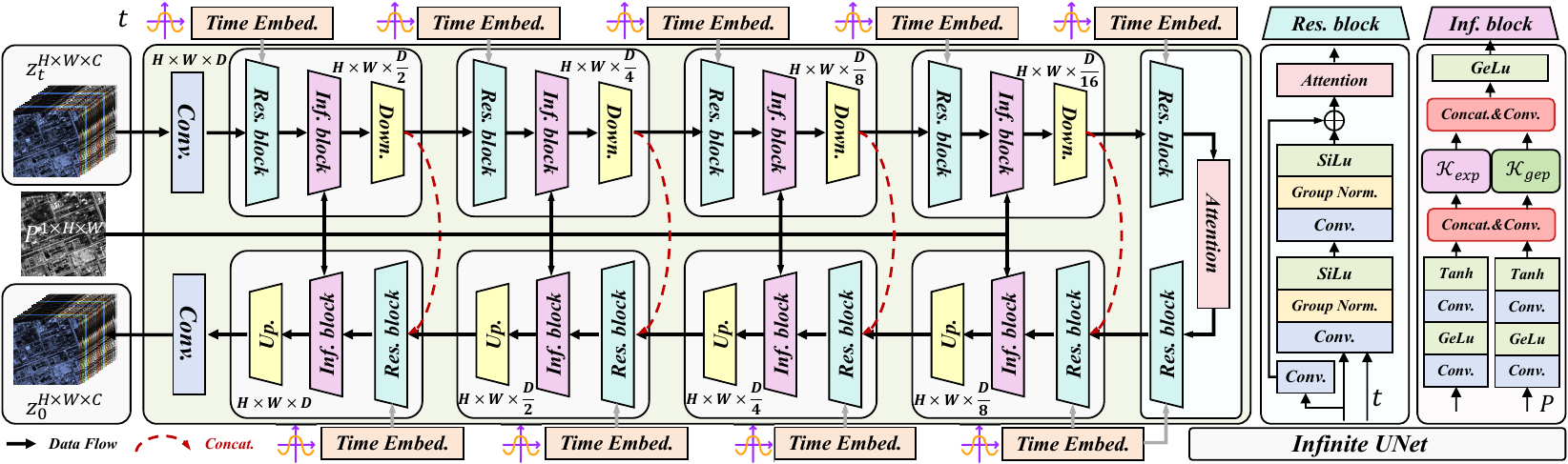}
	\vspace{-0.2em}
	\caption{The network architecture of Infinite-UNet in latent diffusion bridge model.}
	\label{method:1-2}
\end{figure*}

\subsection{Infinite Dimension Interaction}
To progressively refine the latent representation toward high-quality reconstruction, it is crucial to effectively exploit the pixel-domain PAN image as a spatial prior for latent enhancement. However, the domain discrepancy between pixel-space PAN features and latent-space MS representations makes such cross-domain interaction inherently challenging. Existing fusion modules, including concatenation, attention, and multiplicative interactions, primarily capture finite-order feature couplings. Although stacking multiple layers implicitly enriches interactions, the resulting interaction order remains fundamentally bounded by network depth. Therefore, instead of explicitly constructing increasingly higher-order interactions, we seek an implicit formulation capable of aggregating interactions of all orders.

Suppose the feature basis of pixel space to be $\mathcal{A} \! =\! \{\alpha_i\}_{i=1}^m$ and that of the latent space be $\mathcal{B} \!=\! \{\beta_j\}_{j=1}^n$ with degrees of freedom $m$ and $n$, respectively. Most methods simply concatenate them to construct the union space $S^{(1)}=\mathcal{A}\oplus \mathcal{B}  =\text{span}\{\alpha_1,\cdots,\alpha_m,\beta_1,\cdots,\beta_n\}$ with the degree of freedom of $d =  m+n$. This formulation shows that feature concatenation only constructs the first-order space without explicitly feature interactions. Star operation~\cite{ma2024rewrite} utilizes the Hadamard product $\odot$ to extend the feature interaction into an implicit interactive quadratic space $S^{(2)} = W^T\mathcal{V}\odot W^T\mathcal{V} = \text{span}\{\alpha_i\alpha_j,\alpha_i\beta_j,\cdots,\beta_i\beta_j\}$ with a degree of freedom $d(d+1)/2$ for $\mathcal{V} \in S^{(1)}$:
\begin{align}
	 &W^T\mathcal{V}\odot W^T\mathcal{V}\\
	=&\bigg(\sum_{i=1}^d \omega_i v^i\bigg)\odot \bigg(\sum_{j=1}^d \omega_j v^j\bigg)\\
	=&\sum_{i=1}^d \sum_{j=1}^d  \omega_i  \omega_j  v^i  v^j \\
	=&\underbrace{\omega_{1,1} v^1 \! v^1 \!+\! \cdots \!+\! \omega_{i,j}\! v^i v^j + \cdots\! +\! \omega_{d,d} v^d \! v^d}_{d(d+1)/2 \text{items}}.
\end{align}
This suggests that multiplicative operators implicitly lift feature interactions into higher-order polynomial spaces. Self-attention mechanism in Transformers can be interpreted as applying two successive Hadamard products among query, key, and value features~\cite{xu2024infinite}, which implicitly constructs a cubic interaction space $S^{(3)}$ with a degree of freedom $d(d+1)(d+2)/6$. Extending this observation, for a $k$-order Hadamard product interaction on features of total dimension $d$, the dimension of the feature interaction space can be:
\begin{equation}\label{Inf:eq0}
	dim(S^{(k)}) = \binom{d+k-1}{k} = \frac{(d+k-1)!}{(d-1)!k!}.
\end{equation}
The above analysis reveals that existing interaction mechanisms can all be interpreted as finite-order approximations of a unified interaction space. Motivated by that, we define an infinite-dimensional interaction space that simultaneously incorporates interactions of all orders to fully exploit feature interactions:
\begin{equation}\label{Inf:eq1}
	\!\!S^{\text{inf}} = \text{span}\{\bigcup_{k=1}^{\infty}\! S^{(k)}\},
\end{equation}
where the union collects interaction subspaces of different orders. Intuitively, a naive approach to realizing $S^{\text{inf}}$ would require explicitly computing and aggregating infinitely many Hadamard products, which is clearly computationally intractable. Instead, we derive two closed-form surrogate functions that implicitly aggregate all interaction orders, namely a geometric surrogate function $\mathcal{K}_{\text{geo}}(\cdot)$ and an exponential surrogate function $\mathcal{K}_{\text{exp}}(\cdot)$. Let $\mathcal{V} \in S^{(1)}$ be a feature representation, then these surrogate functions can implicitly capture and aggregate infinite-order interactions:
\begin{align}
	\!\!\mathcal{K}_{\text{geo}}(\mathcal{W^TV}) &\!= \!\frac{1}{1 - (\mathcal{W^TV})}\!= \!\sum_{k=0}^{\infty} \mathcal{(W^TV)}^{\odot k}, \label{Inf:eq2}\\
	\!\!\mathcal{K}_{\text{exp}}(\mathcal{W^TV}) &\!= \!\exp{(\mathcal{W^TV})} \!=\! \sum_{k=0}^{\infty} \frac{(\mathcal{W^TV})^{\odot k}}{k!}\label{Inf:eq3},
\end{align}
where $(\cdot)^{\odot k}$ denotes the $k$-fold Hadamard power and $W^T$ is a weight matrix that can be densely or sparsely implemented via a linear layer or a convolution layer, respectively. These expansions generate the complete set of all symmetric monomials spanning $S^{\infty}$. Notably, the two surrogate functions exhibit complementary characteristics. $\mathcal{K}_{\text{geo}}$ assigns equal weight to all interaction orders, while $\mathcal{K}_{\text{exp}}$ imposes factorial decay on higher-order terms, effectively emphasizing lower-order interactions. Both of them enable the full-order feature interactions, without explicit order selection or truncation. For numerical stability, the exponential kernel is kept strictly within a bounded range and thus immune to overflow by constraining the extracted feature with $tanh(\cdot)$ activation (Fig.~\ref{method:1-2}). Meanwhile,  the geometric kernel is stabilized by using $tanh(\cdot)$ activation to avoid singularities.
 
\subsection{Network Architecture}
Our network architecture comprises the modality-interleaved transformer and the latent diffusion bridge model. Specifically, MiT comprises a router network and a set of feed-forward expert layers. The former integrates a patch-embedding module, a multi-layer Transformer stack, and a query-cached attention pooling mechanism, followed by a multi-layer perception (MLP) head that produces expert selection scores. The latter consists of multiple parallel sub-networks, each implemented as a two-layer MLP with gated activation functions, enabling adaptive band specialization. LDBM adopts a U-shaped denoising architecture with symmetric encoder-decoder layers and skip connections, termed Infinite-UNet (Fig.~\ref{method:1-2}). Initially, it utilizes a convolutional layer to project input into a $D$-dimension feature space. Then it employs time-conditioned residual blocks modulated by sinusoidal embeddings to enable step-wise feature adaptation along the diffusion trajectory. To capture long-range dependencies, standard self-attention modules are inserted across scales. Moreover, infinite dimension interaction blocks are inserted at each stage, leveraging geometric and exponential kernels to model full-order interactions between PAN observation and latent MS features. Additionally, multi-scale representation is learned via strided convolutions in encoder and up-sampling followed by convolution in decoder, maintaining spatial alignment.

\begin{figure*}[t]
	\centering
	\includegraphics[width=0.99\linewidth]{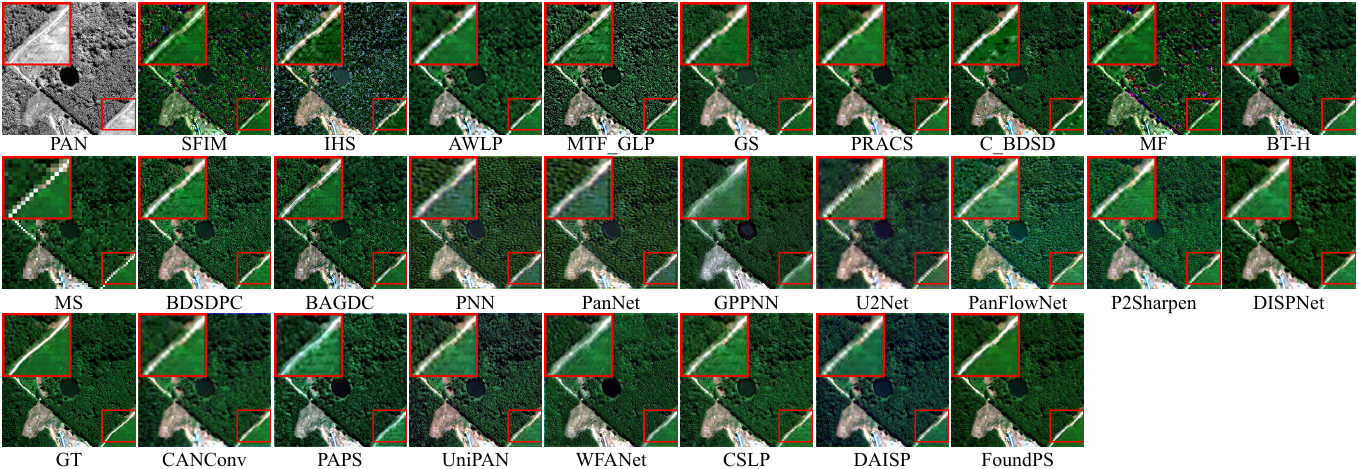}
	\caption{Visual comparisons of reduced scale on 4-band PSBench.}
	\label{exp:reduced_4}
\end{figure*}

\begin{figure*}[t]
	\centering
	\includegraphics[width=0.99\linewidth]{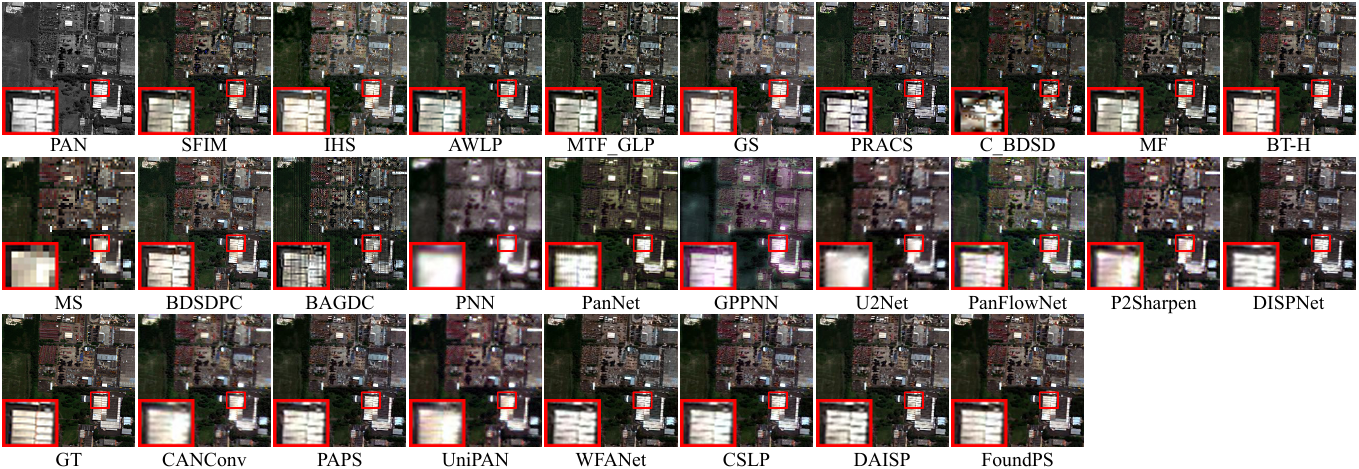}
	\caption{Visual comparisons of reduced scale on 8-band PSBench.}
	\label{exp:reduced_8}
\end{figure*}

\subsection{Loss Functions}
To stabilize the fusion performance, we employ a coupled loss that enforces cross-domain consistency by jointly constraining the latent representation and the reconstructed image. Specifically, given the predicted high-quality latent representations $\widehat{\mathbf{z}}_0^t$ at time $t$, the loss is formulated as:
\begin{equation}
	\mathcal{L}_{\text{ref}} 
	= \| \widehat{\mathbf{z}}_0^t -  y \mathbb{T} \|_1 
	+ \|  \widehat{\mathbf{z}}_0^t  \mathbb{T}^{\ast} - y \|_1 ,
\end{equation}
where $\|\cdot\|_1$ represents $\ell_1$ loss. These two items share a similar formulation but impose complementary constraints. The first term constrains the unified latent space by enforcing information consistency across all latent dimensions, particularly the extended dimensions, preventing them from drifting without semantic constraints and stabilizing latent evolution. The second term enforces reconstruction consistency in the image domain through the inverse projection, preserving spectral fidelity and observation consistency. 

  
\begin{table*}[t]
	\centering
	\renewcommand\arraystretch{1.0}
	\caption{Quantitative comparisons on reduced scale over the PSBench. The best and second best results are shown in \textbf{bold} and \underline{underline} respectively. The up or down arrows indicate higher or lower values correspond to better results.}
	\label{tab:reduced}
	\resizebox{2.10\columnwidth}{!}{
		\begin{tabular}{c||c||c@{\,~}c@{\,~}c@{\,~}c@{\,~}|c@{\,~}c@{\,~}c@{\,~}c@{\,~}|c@{\,~}c@{\,~}c@{\,~}c@{\,~}|c@{\,~}c@{\,~}c@{\,~}c@{\,~}||c@{\,~}c@{\,~}c@{\,~}c@{\,~}}
			\thickhline
			&  & \multicolumn{4}{c|}{4 Bands} & \multicolumn{4}{c|}{7 Bands} & \multicolumn{4}{c|}{8 Bands} &  \multicolumn{4}{c|}{10 Bands} & \multicolumn{4}{c}{Average}  \\ 
			\multirow{-2}{*}{\textbf{Method}}  & \multirow{-2}{*}{\textbf{Year}} & PSNR$\uparrow$	  & SSIM$\uparrow$ & ERGAS$\downarrow$ & SAM$\downarrow$ & PSNR$\uparrow$	  & SSIM$\uparrow$ & ERGAS$\downarrow$ & SAM$\downarrow$ & PSNR$\uparrow$	  & SSIM$\uparrow$ & ERGAS$\downarrow$ & SAM$\downarrow$ & PSNR$\uparrow$	  & SSIM$\uparrow$ & ERGAS$\downarrow$ & SAM$\downarrow$ & PSNR$\uparrow$	  & SSIM$\uparrow$ & ERGAS$\downarrow$ & SAM$\downarrow$\\
			\thickhline
			\multicolumn{22}{l}{\textbf{Traditional Method}}\\
			\hline
			SFIM~\cite{liu2000smoothing}& 2000  &   26.901   &   0.796    &   19.712   &   5.697    &   29.643   &   0.899    &   8.327    &   2.227    &   26.770   &   0.793    &   35.380   &   6.176   &   33.330   &   0.958    &   4.220    &   2.756     &   29.475   &   0.867    &   14.917   & 4.186      \\
			IHS~\cite{tu2001new} & 2001   &   22.947   &   0.712    &   10.916   &   7.715    &   24.400   &   0.857    &   4.166    &   3.026    &   25.835   &   0.708    &   11.245   &   8.607    
			&   33.422   &   0.955    &   3.199    &   2.450    & 27.060 & 0.817 & 7.234 & 5.286      \\
			AWLP~\cite{otazu2005introduction}& 2005  
			&   28.919   &   0.799    &   6.916    &   6.998   	&   28.923   &   0.744    &   2.586    &   2.160    &   24.520   &   0.782    &   10.888   &   10.929   &   34.180   &   0.867    &   2.509    &   1.960   & 29.994 & 0.809 & 5.286 & 5.065      \\
			MTF\_GLP~\cite{aiazzi2006mtf} & 2006  
			&   24.682   &   0.771    &   16.405   &   6.401  &   20.445   &   0.866    &   4.878    &   1.842  &   26.609   &   0.772    &   10.462   &   6.104   &   32.696   &   0.955    &   3.091    &   1.237    & 26.833 & 0.848 & 9.115 & 3.854      \\
			GS~\cite{aiazzi2007improving}& 2007  &   27.064   &   0.804    &   11.141   &   5.314 	&   23.299   &   0.880    &   5.202    &   2.363   &   26.123   &   0.810    &   13.529   &   6.333    &   36.638   &   0.967    &   2.037    &   1.503   & 29.387 & 0.872 & 7.453 & 3.690   \\
			PRACS~\cite{choi2010new} & 2010  &   28.298   &   0.762    &   8.300    &   7.238    
			&   27.835   &   0.708    &   3.505    &   3.684    &   26.815   &   0.760    &   10.846   &   8.887  	&   31.732   &   0.794    &   4.545    &   3.801    & 29.121 & 0.763 & 6.590 & 5.722   \\
			C\_BDSD~\cite{garzelli2014pansharpening} & 2014 &   30.169   &   0.803    &   10.550   &   5.900    &   27.683   &   0.741    &   3.778    &   4.106   &   25.052   &   0.753    &   20.561   &   9.290    &   33.132   &   0.863    &   3.553    &   3.768  & 29.942 & 0.804 & 8.526 & 5.381  	 \\
			MF~\cite{restaino2016fusion} & 2016 &   24.726   &   0.773    &   25.941   &   5.498    
			&   25.574   &   0.899    &   10.479   &   3.032   &   27.600   &   0.818    &   30.155   &   5.544    &   33.300   &   0.959    &   4.229    &   2.694  & 28.107 & 0.863 & 16.696 & 4.147   	 \\ 
			BT-H~\cite{lolli2017haze} & 2017 &   26.504   &   0.757    &   17.324   &   6.414   &   24.610   &   0.908    &   3.866    &   3.271    &   27.431   &   0.834    &   10.455   &   5.360    
			&   29.320   &   0.936    &   3.456    &   2.350    
			& 27.224 & 0.854 & 9.372 & 4.368   	\\
			BDSDPC~\cite{vivone2019robust}& 2019 &   30.107   &   0.751    &   9.054    &   6.005    
			&   29.891   &   0.781    &   2.857    &   2.955   &   24.062   &   0.637    &   14.902   &   9.459  	&   34.333   &   0.874    &   3.192    &   2.863    
			& 30.561 & 0.780 & 6.891 & 4.941   \\
			BAGDC~\cite{lu2021unified} & 2021  &   28.507   &   0.737    &   9.180    &   4.965   &   29.581   &   0.707    &   2.229    &   1.331   &   22.292   &   0.614    &   17.573   &   9.017   &   33.396   &   0.803    &   2.492    &   1.721   & 29.384 & 0.735 & 6.987 & 3.852   	\\
			\hline
			\multicolumn{22}{l}{\textbf{Deep Learning Method}}\\
			\hline
			PNN~\cite{masi2016pansharpening}& 2016 	&   30.537  &   0.847   &   6.178   &   6.935 	&   30.392   &   0.872    &   2.827    &   3.642	&   27.949   &   0.741    &   10.781   &   9.254 	&   37.402   &   0.945    &   2.543    &   3.205
			& 32.377 & 0.868 & 5.068 & 5.468\\
			PanNet~\cite{yang2017pannet}& 2017 	&   30.649  &   0.857   &   7.330   &   8.023	&   31.484   &   0.903    &   2.491    &   3.328	&   28.796   &   0.799    &   13.157   &   10.502	&   38.524   &   0.959    &   1.749    &   2.847	& 33.103 & 0.890 & 5.498 & 5.856\\
			GPPNN~\cite{xu2021deep}& 2021  	&  31.553   &   0.856   &   6.481   &   6.561 	&  32.063   &   0.921    &   2.296    &   2.987	&  29.696   &   0.815    &   10.490   &   8.893	&  39.871   &   0.966    &   1.681    &   2.464 & 34.093 & 0.898 & 4.754 & 4.926\\
			U2Net~\cite{peng2023u2net} & 2023	&  32.078   &   0.826   &   6.128   &   4.946 	&  31.378   &   0.815   &   2.496   &   2.513	&  30.060   &   0.785   &   9.269   &   5.176 	&  39.136   &   0.935   &   1.698   &   2.511	& 33.966 & 0.854 & 4.493 & 3.749\\
			PanFlowNet~\cite{yang2023panflownet}& 2023	&  32.938   &   0.887    &   5.711    &   5.918 
			&  32.699   &   0.917    &   1.988    &   2.638	&  30.504   &   0.836    &   9.864    &   8.953
			&  40.361   &   0.979    &   1.757    &   1.505   & 34.966 & 0.915 & 4.364 & 4.335	\\
			P2Sharpen~\cite{zhang2023p2sharpen}& 2023	&  32.249   &   0.893   &  11.359   &   5.379 	&  34.876   &   0.936   &   1.533   &   2.068	&  30.631   &   0.848   &   9.461   &   7.236	&  40.181   &   0.958   &   1.558   &   2.263	& 35.076 & 0.915 & 6.118 & 4.042\\
			DISPNet~\cite{wang2024deep}& 2024 	&  32.732   &   0.892   &   8.139   &   4.837  	&  35.443   &   0.939   &   1.408   &   1.738	&  31.910   &   0.877   &   7.244   &   5.438	&  40.538   &   0.976   &   1.327   &   2.107 	& 35.650 & 0.926 & 4.577 & 3.479\\  
			CANConv~\cite{duan2024content} & 2024 	&  33.129   &   0.873   &   6.185   &   4.652	&  35.111   &   0.932   &   1.964   &   2.040	&  32.302   &   0.879   &   7.399   &   5.578 		&  40.139   &   {\underline{{0.982}}}    &   1.510    &   2.366 & 35.655 & 0.920 & 4.081 & 3.574\\
			PAPS~\cite{jia2024paps} & 2024 	&  33.053   &   0.883   &   5.619   &   5.429  &  35.242   &   0.944   &   1.489   &   1.849	&  32.958   &   {\underline{{0.903}}}   &   7.343   &   5.102 	&  41.078   &   0.969   &   1.333   &   1.968    & 36.056 & 0.925 & 3.733 & 3.609\\
			UniPAN~\cite{cui2025enpowering}& 2025 	&  33.836   &   0.892   &   6.007   &   4.970	&  35.587   &   0.945   &   1.535   &   1.947	&  30.323   &   0.827    &   9.449    &   8.065		&   41.879   &   0.979    &   1.299    &   1.923	& 36.263 & 0.920 & 4.175 & 3.889 \\
			WFANet~\cite{huang2025wavelet}& 2025&  32.693   &   0.868    &   5.722    &   5.842	&  {\underline{{37.098}}}   &   0.928    &   1.314    &   {\underline{{1.513}}}  	&  {\underline{{33.283}}}   &   0.895    &   {\underline{{6.063}}}   &   5.383  	&  {\underline{{43.365}}}   &   0.967    &   {\underline{{1.109}}}    &  {\underline{{1.492}}} 	& 37.059 & 0.915 & 3.476 & 3.578\\
			CSLP~\cite{chen2025cslp}& 2025	&  {\underline{{35.506}}}   &  {\underline{{0.906}}}    &   {\underline{{4.346}}}    &   {\underline{{3.306}}}	&  36.873   &   0.948    &   1.319    &   1.630	&  33.156   &   0.902    &   6.859    &   {\underline{{4.887}}}	&  41.380   &   0.979    &   1.375    &   1.824 & {\underline{{37.326}}} & {\underline{{0.937}}} & {\underline{{3.204}}} & {\underline{{2.755}}}\\
			DAISP~\cite{wang2025deep}& 2025 &  34.577   &   0.897   &   5.653   &   4.326	&  37.087   &   {\underline{{0.950}}}   &   {\underline{{1.251}}}   &   1.589	&  32.608   &   0.892   &   7.016   &   4.967	&  42.474   &   0.977   &   1.144   &   1.647	& 37.320 & 0.932 & 3.592 & 3.055\\
			\hline
			\multicolumn{22}{l}{\textbf{Universal Method}} \\
			\hline
			UniPS-T & -  &   35.609   &   0.907    &   4.499    &   3.691    &   35.960   &   0.950    &   1.414    &   1.863    &   31.771   &   0.872    &   7.782    &   5.876    &   40.992   &   0.980    &   1.410    &   2.107    &   36.867   &   0.934    &   3.417    &   3.166  \\
			UniPS-S & -   &   35.526   &   0.907    &   4.046    &   3.539    &   38.517   &   0.954    &   1.111    &   1.278    &   31.871   &   0.871    &   6.348    &   5.632    &   43.162   &   0.985    &   0.902    &   1.500    &   38.021   &   0.936    &   2.830    &   2.774    \\
			UniPS-B & - &   36.616   &   0.925    &   5.063    &   3.046    &   39.009   &   0.957    &   1.073    &   1.115    &   33.360   &   0.907    &   6.632    &   4.813    &   45.139   &   0.987    &   0.875    &   1.170    &   39.351   &   0.949    &   3.209    &   2.346      \\
			UniPS-L & -   &   \textbf{37.370}   &   \textbf{0.929}    &   \textbf{3.636}    &   \textbf{2.825}    &   \textbf{41.307}   &   \textbf{0.972}    &   \textbf{0.804}    &   \textbf{0.904}    &   \textbf{34.063}   &   \textbf{0.913}    &  \textbf{5.531}    &   \textbf{4.402}    &   \textbf{47.614}   &   \textbf{0.991}    &   \textbf{0.640}    &   \textbf{0.887}    &   \textbf{40.937}   &   \textbf{0.955}    &  \textbf{2.427}    &   \textbf{2.079}    \\	 
			\thickhline
		\end{tabular}
	}
\end{table*}

\begin{figure*}[t]
	\centering
	\includegraphics[width=0.99\linewidth]{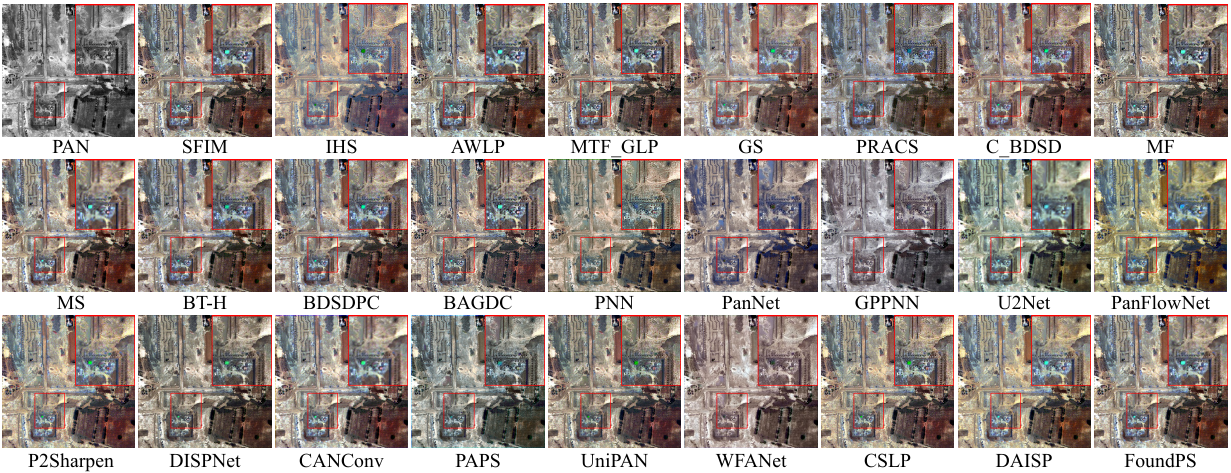}
	\caption{Visual comparisons of full scale on 4-band PSBench.}
	\label{exp:full_4}
\end{figure*}

\section{Experimental Results}\label{sec:4}

In this section, we conduct comprehensive experiments to evaluate the performance of UniPS. We first describe the experimental settings and implementation details. We then compare our approach with state-of-the-art methods on PSBench. Subsequently, ablation studies and efficiency analyses are performed to assess the effectiveness of individual design components. Afterwards, we evaluate the generalization capability and segmentation performance on other public datasets and assess the classical remote sensing applications on PSBench for all methods.

\begin{figure*}[t]
	\centering
	\includegraphics[width=0.99\linewidth]{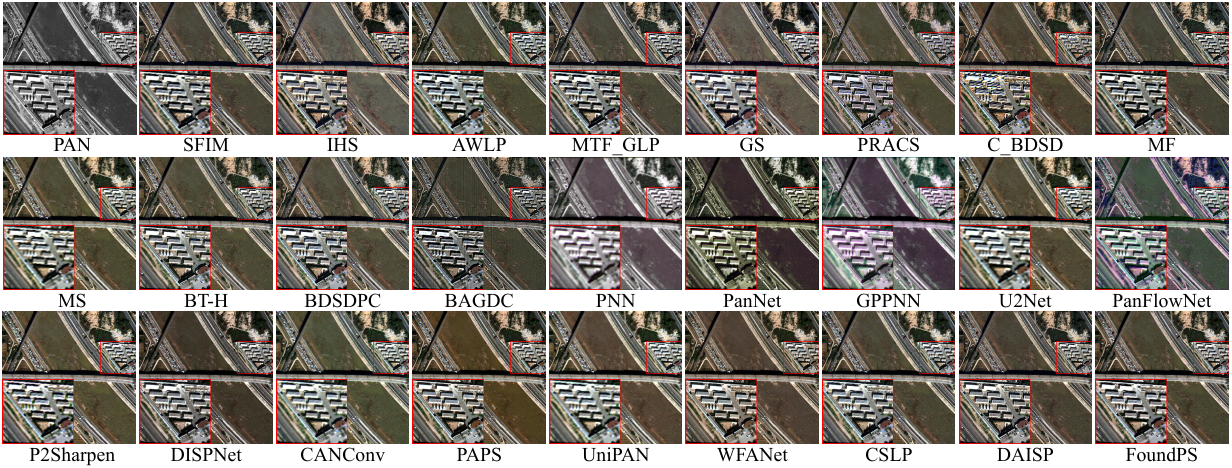}
	\caption{Visual comparisons of full scale on 8-band PSBench.}
	\label{exp:full_8}
\end{figure*}

\subsection{Experiment Configurations}
Extensive experiments are conducted to thoroughly evaluate our methods on both reduced and full scales. As for validation on the reduced scale, the degraded MS and PAN images are regarded as the inputs, while the original MS images are used as the groundtruth by following the Wald protocol~\cite{wald1997fusion}. We use four reference metrics, namely PSNR~\cite{huynh2008scope}, SSIM~\cite{wang2004image}, ERGAS~\cite{wald2002data} and SAM~\cite{alparone2007comparison} to evaluate the fusion performance. It is worth mentioning that the larger values of PSNR and SSIM indicate the higher quality, while the smaller values of ERGAS and SAM mean the better performance. As for the validation on the full scale, three mainstream metrics are used to assess the fusion performance, including $D_\lambda$, $D_s$, and QNR~\cite{alparone2008multispectral}. Notably, these non-reference metrics are evaluated on artificially manipulated parameters, which do not necessarily coincide with the visual effects. Several representative methods are selected for comprehensive comparisons at both reduced and full scales. Traditional methods include SFIM~\cite{liu2000smoothing}, IHS~\cite{tu2001new}, AWLP~\cite{otazu2005introduction}, MTF\_GLP~\cite{aiazzi2006mtf}, GS~\cite{aiazzi2007improving}, PRACS~\cite{choi2010new}, C\_BDSD~\cite{garzelli2014pansharpening}, MF~\cite{restaino2016fusion}, BT-H~\cite{lolli2017haze}, BDSDPC~\cite{vivone2019robust} and BAGDC~\cite{lu2021unified}. Deep learning methods are PNN~\cite{masi2016pansharpening}, PanNet~\cite{yang2017pannet}, GPPNN~\cite{xu2021deep}, U2Net~\cite{peng2023u2net}, PanFlowNet~\cite{yang2023panflownet}, P2Sharpen~\cite{zhang2023p2sharpen}, DISPNet~\cite{wang2024deep}, CANConv~\cite{duan2024content}, PAPS~\cite{jia2024paps}, UniPAN~\cite{cui2025enpowering}, WFANet~\cite{huang2025wavelet}, CSLP~\cite{chen2025cslp} and DAISP~\cite{wang2025deep}. Since PSBench provides MS images with different spectral configurations, all deep learning methods are re-trained separately on each task-specific dataset with necessary network modifications.

Our UniPS is trained using 8 NVIDIA A100 80GB GPUs with the PyTorch framework for 144h. Adam optimizer and L1 loss are employed for 500k iterations with a learning rate of 1e-4. We set the total batch size as 32 and distribute it evenly to each task. We randomly crop $256\times256$ PAN patches and MS patches with corresponding scale factor as network input for training and use 10 timesteps for both reduced- and full-scale testing. The maximum number of experts is $16$. Besides, we propose several Infinite-UNet variants by changing the channel number of the hidden layers dimension $D$ to obtain different versions with varied parameter quantities: 
\begin{itemize}
	\item[$\bullet$] UniPS-T: $D$=32, channel multiplier = \{1,1,1,1\}
	\item[$\bullet$] UniPS-S: $D$=64, channel multiplier = \{1,1,1,1\}
	\item[$\bullet$] UniPS-B: $D$=32, channel multiplier = \{1,2,2,4\}	
	\item[$\bullet$] UniPS-L: $D$=64, channel multiplier = \{1,2,2,4\}
\end{itemize}
We adopt the default noise schedule and stationary variance $\lambda = 0.001$ in~\cite{wang2025residual} with total diffusion steps as 1000.  Notably, UniPS refers to UniPS-L unless otherwise stated. 

\begin{table*}[t]
	\centering
	\renewcommand\arraystretch{1.00}
	\caption{Quantitative comparisons on full scale over the PSBench. The best and second best results are shown in {\textbf{bold}} and {\underline{underline}} respectively. The up or down arrows indicate higher or lower values correspond to better results.}
	\label{tab:full}
	\resizebox{2.10\columnwidth}{!}{
		\begin{tabular}{c||c||ccc|ccc|ccc|ccc||ccc}
			\thickhline
			&  & \multicolumn{3}{c|}{4 Bands} & \multicolumn{3}{c|}{7 Bands} & \multicolumn{3}{c|}{8 Bands} &  \multicolumn{3}{c|}{10 Bands} & \multicolumn{3}{c}{Average}  \\  
			\multirow{-2}{*}{\textbf{Method}} & \multirow{-2}{*}{\textbf{Year}}	& QNR$\uparrow$ & $D_\lambda\downarrow$ & $D_s\downarrow$  & QNR$\uparrow$ & $D_\lambda\downarrow$ & $D_s\downarrow$ &  QNR$\uparrow$ & $D_\lambda\downarrow$ & $D_s\downarrow$ &  QNR$\uparrow$ & $D_\lambda\downarrow$ & $D_s\downarrow$  & QNR$\uparrow$ & $D_\lambda\downarrow$ & $D_s\downarrow$   \\
			\thickhline
			\multicolumn{17}{l}{\textbf{Traditional Method}}\\
			\hline
			SFIM~\cite{liu2000smoothing}& 2000 &   0.7849   &   0.0919   &   0.1372   &   0.8115   &   0.0768   &   0.1227   &   0.8437   &   0.0719   &   0.0918   &   0.8158   &   0.0580   &   0.1343  & 0.8084 & 0.0752 & 0.1270	\\
			IHS~\cite{tu2001new} & 2001 &   0.5925   &   0.1055   &   0.3393   &   0.6957   &   0.1360   &   0.2002   &   0.8112   &   0.0875   &   0.1128   &   0.7444   &   0.1408   &   0.1356   & 0.6928 & 0.1198 & 0.2145 \\
			AWLP~\cite{otazu2005introduction}& 2005 &   0.7733   &   0.0839   &   0.1619   &   0.8032   &   0.0984   &   0.1112   &   0.8533   &   0.0486   &   0.1042   &   0.7354   &   0.1550   &   0.1305   
			& 0.7780 & 0.1046 & 0.1341	\\
			MTF\_GLP~\cite{aiazzi2006mtf} & 2006 &   0.7054   &   0.1175   &   0.2051   &   0.8200   &   0.0880   &   0.1035   &   0.8064   &   0.0855   &   0.1196   &   0.7725   &   0.0520   &   0.1857    
			& 0.7627 & 0.0861 & 0.1680	 \\
			GS~\cite{aiazzi2007improving}& 2007 &   0.6515   &   0.0681   &   0.3064   &   0.7734   &   0.0787   &   0.1654   &   0.8625   &   0.0426   &   0.0998   &   0.8235   &   0.0507   &   0.1335   
			& 0.7606 & 0.0605 & 0.1942	   \\
			PRACS~\cite{choi2010new} & 2010 &   0.6979   &   0.0921   &   0.2436   &   0.7238   &   0.1479   &   0.1537 &  0.7620   &   0.1172   &   0.1382   &   0.7486   &   0.1130   &   0.1571   
			& 0.7286 & 0.1126 & 0.1837 	 \\
			C\_BDSD~\cite{garzelli2014pansharpening} & 2014 &   {\underline{0.9173}}   &   {\underline{0.0307}}   &   {\underline{0.0544}}   &  0.8075   &   0.0899   &   0.1137     &   0.8020   &   0.0595   &   0.1472   &   0.8252   &   0.0516   &   0.1306 
			& 0.8506 & 0.0523 & {\underline{0.1036}}		 		\\
			MF~\cite{restaino2016fusion} & 2016 &   0.7162   &   0.1135   &   0.1964   &   0.7570   &   0.0996   &   0.1610    &   0.8346   &   0.0768   &   0.0971   &   0.7911   &   0.0719   &   0.1480   
			& 0.7654 & 0.0920 & 0.1596   	 \\ 
			BT-H~\cite{lolli2017haze} & 2017 &   0.6991   &   0.1360   &   0.2056   &   0.8299   &   0.0792   &   {\underline{0.1002}}   &   0.8545   &   0.0546   &   0.0979   &   0.7533   &   0.0904   &   0.1724  
			& 0.7631 & 0.0990 & 0.1601   	  \\
			BDSDPC~\cite{vivone2019robust}& 2019&   0.9038   &   0.0347   &   0.0657   &  0.8190   &   0.0817   &   0.1093   &   0.8590   &   0.0498   &   0.0973 & 0.7517   &   0.0981   &   0.1677    
			& {\underline{0.8323}} & 0.0661 & 0.1115   	   \\
			BAGDC~\cite{lu2021unified} & 2021 &   0.8306   &   0.0826   &   0.1040   &   0.6927   &   0.2217   &   0.1117   &    0.7539   &   0.1090   &   0.1541   &   0.6311   &   0.2238   &   0.1874   
			& 0.7294 & 0.1575 & 0.1400 	   \\
			\hline
			\multicolumn{17}{l}{\textbf{Deep Learning Method}}\\
			\hline
			PNN~\cite{masi2016pansharpening}& 2016 	&   0.6454   &   0.1459   &   0.2484    &   0.7985   &   0.1125   &   0.1013    &   0.7992   &   0.0889   &   0.1230     &   0.7035   &   0.1307   &   0.1922    	& 0.7144 & 0.1266 & 0.1853  \\
			PanNet~\cite{yang2017pannet}& 2017  &   0.6151   &   0.1285   &   0.2971  	&   0.8157   &   0.0816   &   0.1124	&   0.8162   &   0.0818   &   0.1116 &   0.7322   &   0.1274   &   0.1619   & 0.7188 & 0.1129 & 0.1926    \\
			GPPNN~\cite{xu2021deep}& 2021	&   0.6313   &   0.1474   &   0.2638   	&   0.7673   &   0.1011   &   0.1471  	&   0.7969   &   0.1039   &   0.1118   &   0.7681   &   0.1139   &   0.1337 	& 0.7247 & 0.1218 & 0.1780  \\
			U2Net~\cite{peng2023u2net} & 2023	&   0.8469   &   0.0696   &   0.0897    &   0.7984   &   0.1107   &   0.1028
			&   {\underline{{0.8639}}}   &   {\underline{{0.0272}}}   &   0.1121    &   0.7167   &   0.1366   &   0.1706  & 0.7981 & 0.0926 & 0.1218\\
			PanFlowNet~\cite{yang2023panflownet}& 2023	&   0.7472   &   0.1273   &   0.1459  &   0.7731   &   0.1073   &   0.1347	&   0.8254   &   0.0792   &   0.1038	&   0.8160   &   0.0698   &   0.1230   	& 0.7859 & 0.0978 & 0.1302 \\
			P2Sharpen~\cite{zhang2023p2sharpen}& 2023	&   0.7819   &   0.0520   &   0.1758   &   0.8054   &   0.0728   &   0.1321  &   0.8554   &   0.0595   &   0.0908 &   0.7937   &   0.0886   &   0.1293 & 0.8008 & 0.0688 & 0.1403 \\
			DISPNet~\cite{wang2024deep}& 2024	&   0.7682   &   0.0903   &   0.1560 &   0.8413   &   0.0553   &   0.1097   &   0.8274   &   0.0918   &   0.0895 &   0.8133   &   0.0612   &   0.1341 & 0.8047 & 0.0748 & 0.1308\\  
			CANConv~\cite{duan2024content} & 2024	&   0.7768   &   0.0416   &   0.1915   &   0.8208   &   0.0642   &   0.1233	&   0.8631   &   0.0513   &   0.0907  &   0.8280   &   0.0532   &   0.1257   & 0.8141 & {\underline{{0.0509}}} & 0.1430   \\
			PAPS~\cite{jia2024paps} & 2024 	&   0.7439   &   0.1164   &   0.1597   &   0.8065   &   0.0782   &   0.1260   	&   0.8503   &   0.0593   &   0.0965   	&   0.7917   &   0.0964   &   0.1242  & 0.7864 & 0.0946 & 0.1328\\
			UniPAN~\cite{cui2025enpowering}& 2025	&   0.7780   &   0.0693   &   0.1640    &   0.8121   &   0.0696   &   0.1278   &   0.8601   &   0.0557   &   0.0893  	&   {\underline{{0.8286}}}   &   0.0609   &   {\underline{{0.1178}}}  & 0.8127 & 0.0646 & 0.1314 \\
			WFANet~\cite{huang2025wavelet}& 2025&   0.7109   &   0.1030   &   0.2062 	&   {\underline{{0.8512}}}   &   {\underline{{0.0453}}}   &   0.1086 	&   0.8560   &   0.0619   &  {\underline{{0.0878}}} 	&   0.7969   &   0.0789   &   0.1351   
			& 0.7854 & 0.0788 & 0.1481 \\
			CSLP~\cite{chen2025cslp}& 2025 	&   0.7782   &   0.0707   &   0.1641 &   0.8300   &   0.0556   &   0.1216 		&   0.8618   &   0.0530   &   0.0908  &   0.8226   &   0.0651   &   0.1202  & 0.8143 & 0.0636 & 0.1314 \\
			DAISP~\cite{wang2025deep}& 2025	 &   0.7805   &   0.0619   &   0.1688  &   0.8342   &   0.0481   &   0.1242  &   0.8152   &   0.0911   &   0.1041   &   0.8211   &   {\underline{{0.0515}}}   &   0.1346 & 0.8085 & 0.0603 & 0.1401  \\
			\hline
			\multicolumn{17}{l}{\textbf{Universal Method}} \\
			\hline
			UniPS-T & - &   0.8377   &   0.0661   &   0.1040   &   0.7969   &   0.0828   &   0.1322   &   0.8577   &   0.0590   &   0.0889   &   0.7416   &   0.1291   &   0.1496   &   0.8019   &   0.0886   &   0.1218  \\
			UniPS-S & -  &   0.8633   &   0.0440   &   0.0972   &   0.8414   &   0.0481   &   0.1165   &   0.8711   &   0.0478   &   0.0854   &   0.8049   &   0.0578   &   0.1459   &   0.8414   &   0.0498   &   0.1148  \\
			UniPS-B & -  &   0.9135   &   0.0286   &   0.0601   &   0.8677   &   0.0300   &   0.1057   &   0.8850   &   0.0403   &   0.0782   &   0.8402   &   0.0426   &   0.1226   &   0.8771   &   0.0352   &   0.0914   \\
			UniPS-L & -&   \textbf{0.9327}   &   \textbf{0.0170}   &   \textbf{0.0514}   &   \textbf{0.9009}   &   \textbf{0.0108}   &   \textbf{0.0893}   &   \textbf{0.9135}   &   \textbf{0.0186}   &   \textbf{0.0693}   &   \textbf{0.8679}   &   \textbf{0.0215}   &   \textbf{0.1131}   &   \textbf{0.9030}   &   \textbf{0.0176}   &   \textbf{0.0810}    \\	 
			\thickhline
		\end{tabular}
	}
\end{table*}

\begin{figure*}[htp] 
\centering
\subfigure[Expert distribution]{
	\includegraphics[width=0.18\linewidth]{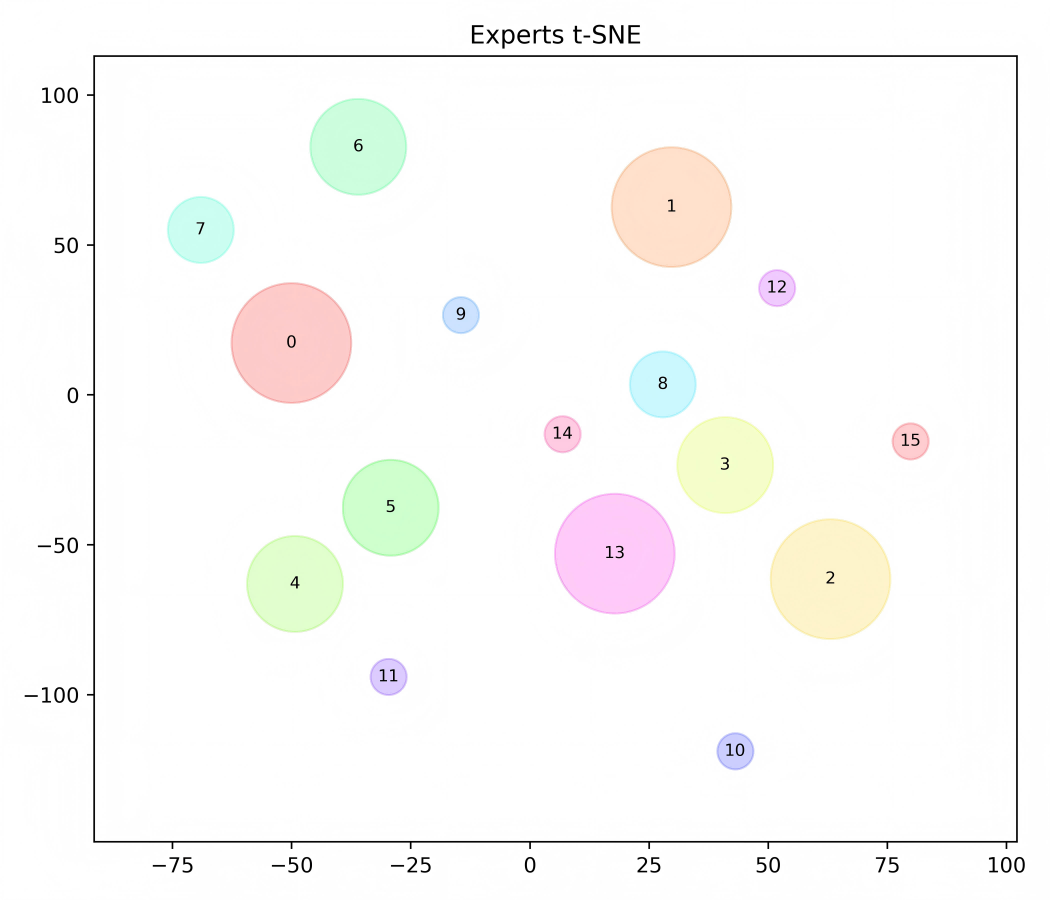}
	\label{2-1-x}
}
\subfigure[Reduced-scale MS]{
	\includegraphics[width=0.18\linewidth]{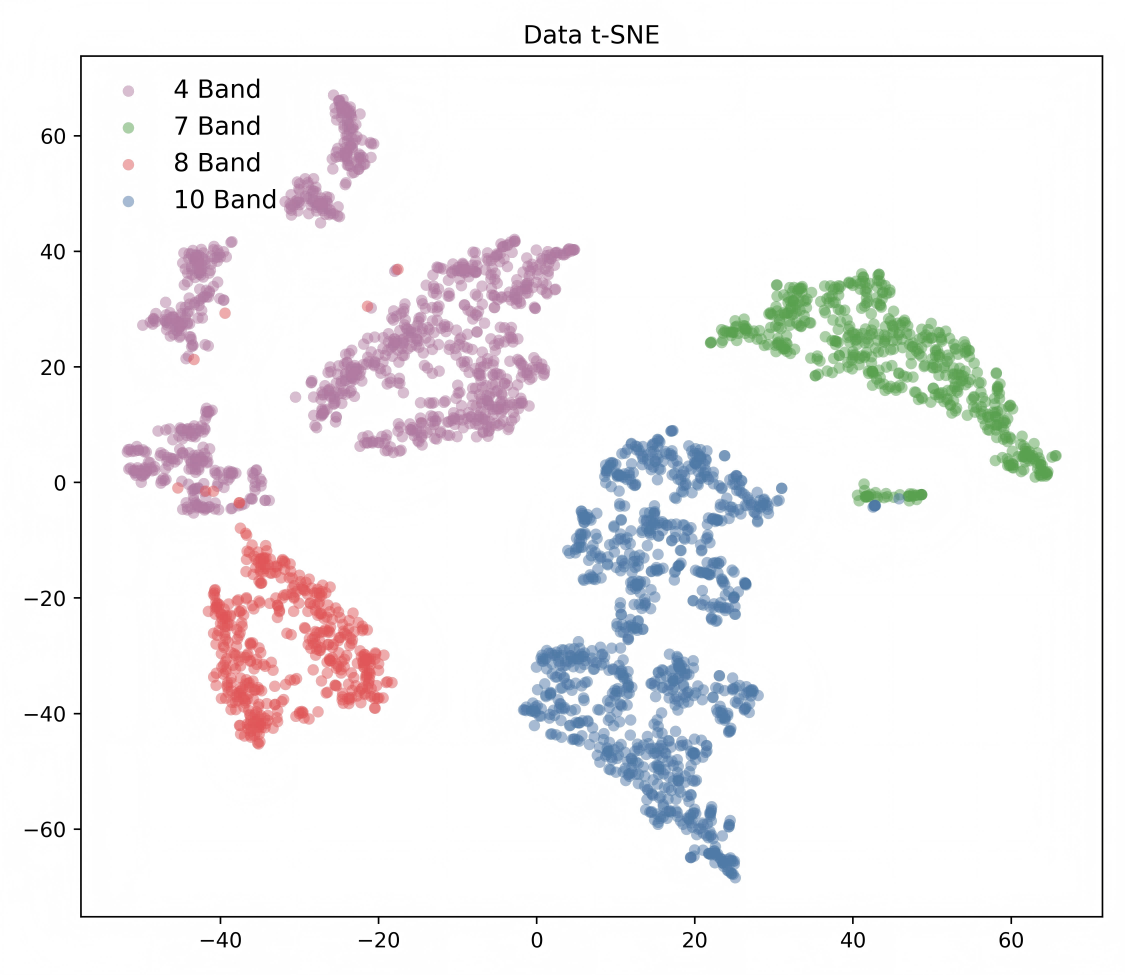}
	\label{2-1-a}
}
\subfigure[Reduced-scale Latent]{
	\includegraphics[width=0.18\linewidth]{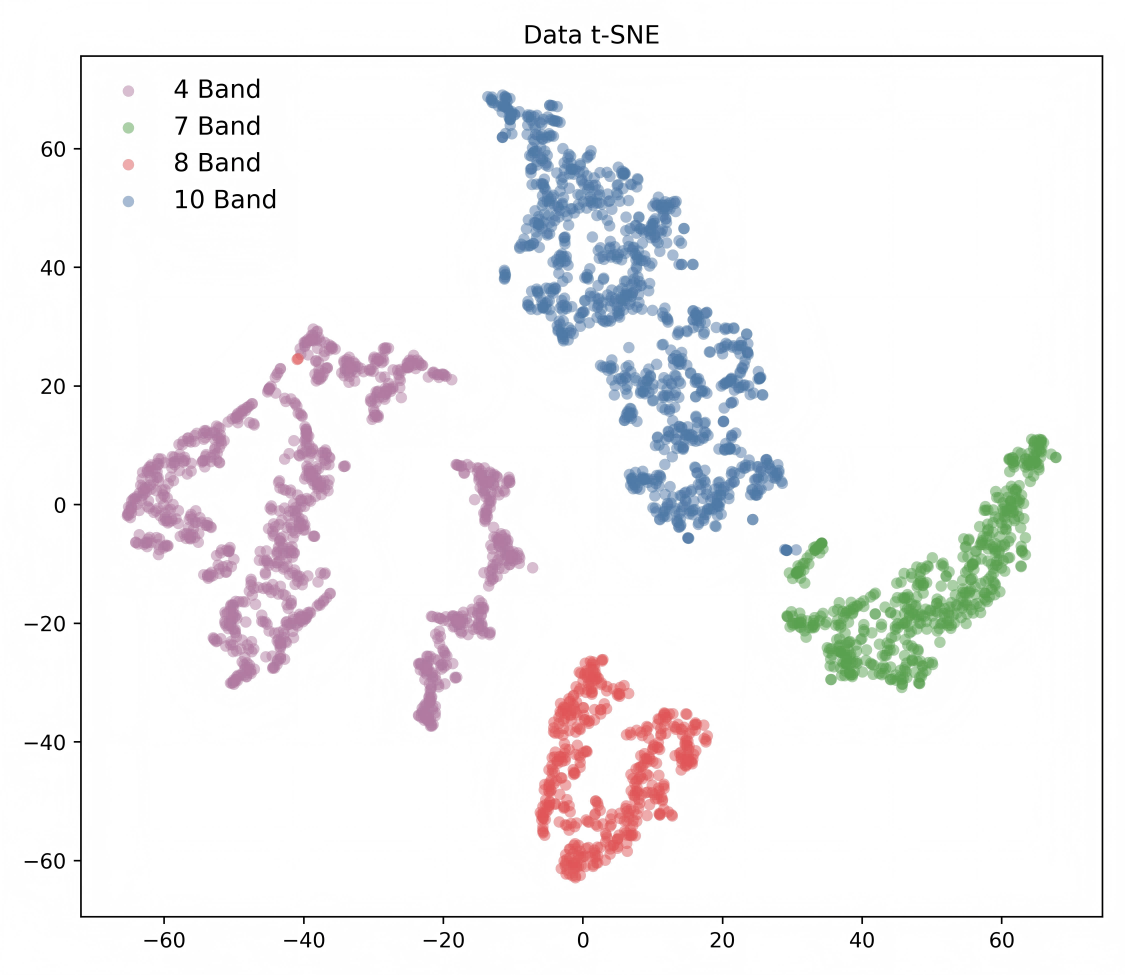}
	\label{2-1-b}
}
\subfigure[Full-scale MS]{
	\includegraphics[width=0.18\linewidth]{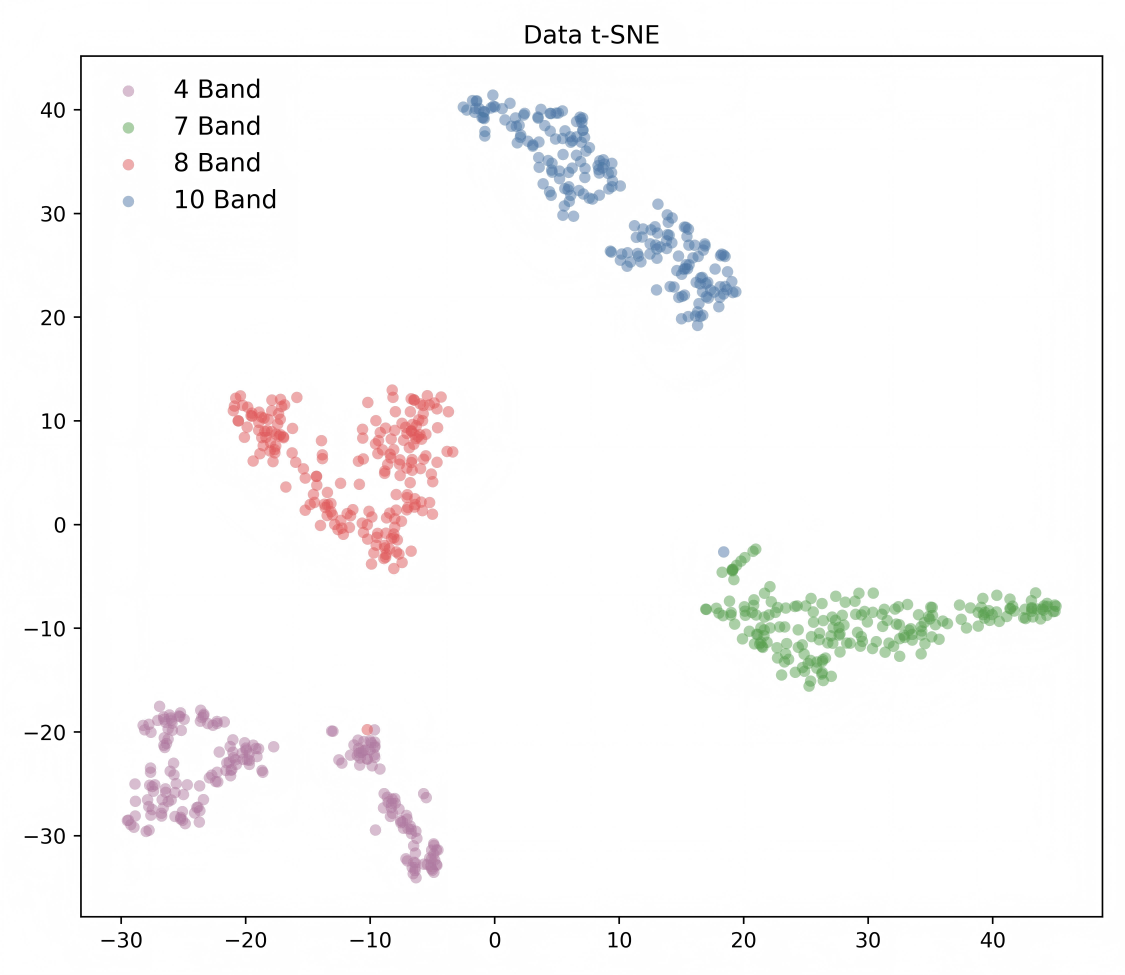}
	\label{2-1-c}
}
\subfigure[Full-scale Latent]{
	\includegraphics[width=0.18\linewidth]{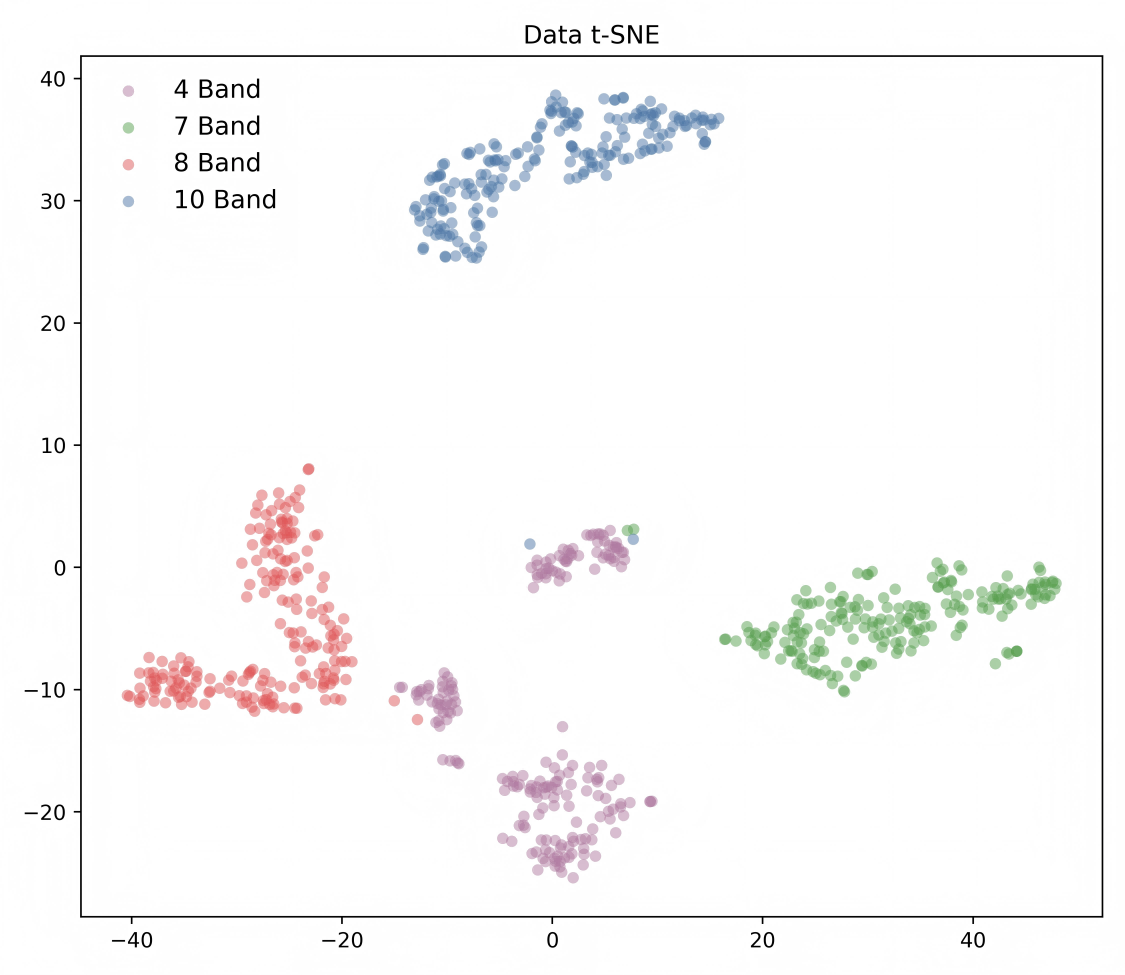}
	\label{2-1-d}
}
\caption{Distribution of data and parameters by t-SNE visualization. (a) Expert parameters and their activation frequency. Larger circles represent higher activation frequencies. (b) MS images on the reduced scale. (c) MS latents generated by MiT on the reduced scale. (d) MS images on the full scale. (e) MS latents generated by MiT on the full scale.}
\label{notion:1-3}
\end{figure*}

\begin{table*}[htp] 
	\caption{Pansharpening performance of reduced scale with different sampling steps.}
	\label{tab:nfe}
	\centering
	\renewcommand\arraystretch{1.05}
	\resizebox{2.10\columnwidth}{!}{
		\begin{tabular}{c|cccc|cccc|cccc|cccc|cccc}
			\toprule
			\multirow{2}{*}{\textbf{NFEs}} & \multicolumn{4}{c|}{4 Bands} & \multicolumn{4}{c|}{7 Bands} & \multicolumn{4}{c|}{8 Bands} &  \multicolumn{4}{c|}{10 Bands} & \multicolumn{4}{c}{Average}  \\ 
			& PSNR$\uparrow$ & ERGAS$\downarrow$ & SAM$\downarrow$ & QNR$\uparrow$ & PSNR$\uparrow$ & ERGAS$\downarrow$ & SAM$\downarrow$ & QNR$\uparrow$ &  PSNR$\uparrow$	 & ERGAS$\downarrow$ & SAM$\downarrow$ & QNR$\uparrow$ & PSNR$\uparrow$	 & ERGAS$\downarrow$ & SAM$\downarrow$ & QNR$\uparrow$ & PSNR$\uparrow$  & ERGAS$\downarrow$ & SAM$\downarrow$ & QNR$\uparrow$\\
			\midrule 
			1 	&   33.232   &   6.219    &   3.220  &   0.9240  &   37.949   &   1.144    &   1.195  &   0.9262  &   28.610   &   11.338   &   4.880  &   0.8906  &   45.127   &   0.849    &   0.973  &   0.9042  &   37.290   &   4.306    &   2.366  &   0.9130   \\
			3  	&   34.991   &   4.839    &   2.845  &   0.9434  &   39.698   &   0.939    &   1.004  &   0.9217  &   30.736   &   7.981    &   4.362  &   0.9197  &   46.759   &   0.700    &   0.852  &   0.8934  &   39.035   &   3.257    &   2.090  &   0.9197  \\
			5  	&   36.795   &   3.872    &   2.708  &   0.9392   &   40.973   &   0.825    &   0.912  &   0.9097   &   33.196   &   6.057    &   4.255  &   0.9167  &   47.869   &   0.616    &   0.822  &   0.8760  &   40.634   &   2.582    &   1.996  &   0.9100  \\
			10 	&   37.370   &   3.637    &   2.826  &   0.9327  &   41.307   &   0.804    &   0.904  &   0.9009  &   34.063   &   5.531    &   4.402  &   0.9135  &   47.614   &   0.640    &   0.887  &   0.8679  &   40.937   &   2.427    &   2.079  &   0.9030   \\
			20 	&   37.047   &   3.782    &   3.003  &   0.9294  &   41.126   &   0.825    &   0.927  &   0.8992   &   33.807   &   5.658    &   4.577  &   0.9134  &   47.216   &   0.671    &   0.932  &   0.8664  &   40.625   &   2.510    &   2.185  &   0.9010  \\
			\bottomrule
	\end{tabular}} 

	\vspace{\dblfloatsep}
 	
	\caption{Fusion performance of different $\eta$ via bridge posterior sampling.}\label{tab:sampling_bps}
	\centering
	\renewcommand\arraystretch{1.05}
	\resizebox{2.10\columnwidth}{!}{
		\begin{tabular}{c|cccc|cccc|cccc|cccc|cccc}
			\toprule
			\multirow{2}{*}{\textbf{$\eta$}} & \multicolumn{4}{c|}{4 Bands} & \multicolumn{4}{c|}{7 Bands} & \multicolumn{4}{c|}{8 Bands} &  \multicolumn{4}{c|}{10 Bands} & \multicolumn{4}{c}{Average}  \\ 
			& PSNR$\uparrow$ & ERGAS$\downarrow$ & SAM$\downarrow$ & QNR$\uparrow$ & PSNR$\uparrow$ & ERGAS$\downarrow$ & SAM$\downarrow$ & QNR$\uparrow$ &  PSNR$\uparrow$	 & ERGAS$\downarrow$ & SAM$\downarrow$ & QNR$\uparrow$ & PSNR$\uparrow$	 & ERGAS$\downarrow$ & SAM$\downarrow$ & QNR$\uparrow$ & PSNR$\uparrow$  & ERGAS$\downarrow$ & SAM$\downarrow$ & QNR$\uparrow$\\
			\midrule 
			0  &   37.370   &   3.637    &   2.826   &   0.9327 &   41.307   &   0.804    &   0.904   &   0.9009&   34.063   &   5.531    &   4.402   &   0.9135  &   47.614   &   0.640    &   0.887   &   0.8679  &   40.937   &   2.427    &   2.079   &   0.9030 \\ 
			0.1 &   37.372   &   3.638    &   2.826   &   0.9296  &   41.331   &   0.802    &   0.903   &   0.9136 	&   34.065   &   5.531    &   4.402   &   0.9074 &   47.625   &   0.639    &   0.887   &   0.8802 &   40.947   &   2.427    &   2.078   &   0.9073 \\
			0.5 &   37.370   &   3.641    &   2.826   &   0.9259 &   41.391   &   0.797    &   0.900   &   0.9108  &   34.063   &   5.531    &   4.402   &   0.9060  &   47.661   &   0.637    &   0.884   &   0.8748 &   40.968   &   2.426    &   2.077   &   0.9036 \\
			1.0 &   37.370   &   3.637    &   2.826   &   0.9246  &   41.400   &   0.796    &   0.899   &   0.9110 &   34.063   &   5.531    &   4.402   &   0.9055 	&   47.685   &   0.635    &   0.882   &   0.8731 &   40.978   &   2.424    &   2.076   &   0.9025\\
			5.0 &   37.370   &   3.637    &   2.826   &   0.9202 	&   41.307   &   0.804    &   0.904   &   0.9103  &   34.063   &   5.531    &   4.402   &   0.9031  &   47.625   &   0.640    &   0.886   &   0.8695  &   40.941   &   2.427    &   2.078   &   0.8993 \\ 
			\bottomrule
	\end{tabular}}
  
\end{table*}

\subsection{Comparative Experiments}

\subsubsection{Reduced-scale Assessment}
The qualitative results of reduced scale across different tasks are provided in Fig.~\ref{exp:reduced_4}, and Fig.~\ref{exp:reduced_8}, while additional visual comparisons are deferred to Appendix~\ref{sec:suppl_D}. Apparently, traditional methods show spectral distortions due to the unbalanced combination of PAN-derived spatial cues. Deep methods improve on the integration of spectral and spatial information, yet our proposed method stands out with a more accurate spectral distribution and enhanced spatial detail. To statistically validate the performance, we report the quantitative results in Tab.~\ref{tab:reduced}. For fairness, we deactivate BPS. Our method achieves the best average value in all metrics, including PSNR, SSIM, SAM and ERGAS, indicating that the fused results are the most similar to the reference. Therefore, our method outperforms all other state-of-the-art methods and is the most competitive.

\subsubsection{Full-scale Assessment}
We also measure the fusion performance on full scale. The visual comparisons are shown in Fig.~\ref{exp:full_4}, and Fig.~\ref{exp:full_8} (See Appendix~\ref{sec:suppl_D} for more results). Regarding the traditional methods, some of them appropriately retain similar spectral attributes with the MS image but inadequately integrate the spatial details, while others distort both spectral and spatial information. As for deep methods, some of them struggle to maintain the balance between the preservation of spectral distribution and the injection of spatial details. In contrast, our method reasonably injects the spatial details and enhances the fusion quality. The quantitative results are presented in Tab.~\ref{tab:full}. Our method still secures the top spot on non-reference metrics and thus maintains a leading edge.
 
\subsection{Ablation Study}

To fully exploit the effectiveness of each design, we conduct ablation studies that include four categories: (i) Distributions and effects of  MiT, (ii) diffusion sampling steps and strategies, (iii) performance of varied UniPS variants, and (iv) scalability of the UniPS.

\subsubsection{Distributions and Effects of MiT} 
To thoroughly evaluate the role of the mixture-of-experts, we first visualize the distributions of expert parameters and their activation frequencies on PSBench, as illustrated in Fig.~\ref{2-1-x}.  Larger circles indicate higher activation frequencies. The non-uniform activation patterns demonstrate that different experts capture different spectral characteristics and are adaptively selected according to the spectral attributes. Besides, we provide the t-SNE distributions of MS pixel features and latent representations generated by MiT in Fig.~\ref{notion:1-3}. Compared with the original MS features, the latent representations preserve the discriminative characteristics of different spectral configurations while mapping heterogeneous inputs into a unified latent space. Hence, MiT effectively balances modality-specific information preservation and cross-sensor representation alignment, enabling UniPS to handle diverse spectral configurations within a single framework.
 
\subsubsection{Sampling Steps and Strategies} 
Model efficiency and fusion quality depend on the sampling steps, quantified by neural function evaluations (NFEs). We provide the fusion performance of different sampling steps in Tab.~\ref{tab:nfe}. In reduced-scale evaluation, reference metrics exhibit a consistent improvement as the number of sampling steps increases, with the best PSNR achieved at 10 NFEs. However, further increasing the sampling steps to 20 NFEs results in a slight degradation, suggesting that excessive iterations may introduce unnecessary refinements and deviate from the optimal bridge trajectory. For full-scale evaluation, 3 NFEs achieve the highest QNR, while additional steps lead to a gradual decline. This phenomenon may be attributed to the domain discrepancy between reduced-scale training and full-scale inference. Considering the balance between computational cost and robustness across both evaluation protocols, we adopt 10 sampling steps as the operational default.
 
We also assess the bridge posterior sampling strategies with different weights $\eta$ in Eq.~(\ref{eq22}). Since different spectral configurations contain varying numbers of bands and exhibit different attributes, the optimal balance between latent refinement and posterior guidance may vary across different inputs. The quantitative results are reported in Tab.~\ref{tab:sampling_bps}. The results show that moderate posterior guidance maintains stable fusion performance across different spectral configurations. At reduced scale, increasing $\eta$ from 0 to 1.0 slightly improves the average PSNR, while excessive guidance ($\eta=5.0$) causes a minor degradation. This indicates that BPS can effectively enhance consistency between generated latent representations and observed MS measurements. Meanwhile, full-scale metrics remain stable for $\eta\le 1.0$, demonstrating the robustness and controllability of BPS. Notably, $\eta = 0.1$ already achieves competitive results. Based on these observations, BPS is capable of serving as a training-free fusion adaptation mechanism that allows controllable adjustment between spectral fidelity and latent refinement according to different pansharpening scenarios.

\begin{table*}[htp]   
	\caption{Fusion performance of UniPS variants at reduced scale.}\label{tab:variantslr}
	\centering
	\renewcommand\arraystretch{1.0}
	\resizebox{2.10\columnwidth}{!}{
		\begin{tabular}{cccc|cccc|cccc|cccc|cccc|cccc}
			\toprule
			\multicolumn{4}{c|}{\textbf{Configurations}}  & \multicolumn{4}{c|}{4 Bands} & \multicolumn{4}{c|}{7 Bands} & \multicolumn{4}{c|}{8 Bands} &  \multicolumn{4}{c|}{10 Bands} & \multicolumn{4}{c}{Average}  \\ 
			& MiT & $\mathcal{K}_{exp}$ & $\mathcal{K}_{geo}$ 			& PSNR$\uparrow$	  & SSIM$\uparrow$ & ERGAS$\downarrow$ & SAM$\downarrow$ & PSNR$\uparrow$	  & SSIM$\uparrow$ & ERGAS$\downarrow$ & SAM$\downarrow$ & PSNR$\uparrow$	  & SSIM$\uparrow$ & ERGAS$\downarrow$ & SAM$\downarrow$ & PSNR$\uparrow$	  & SSIM$\uparrow$ & ERGAS$\downarrow$ & SAM$\downarrow$ & PSNR$\uparrow$	  & SSIM$\uparrow$ & ERGAS$\downarrow$ & SAM$\downarrow$\\
			\midrule  
			I&\ding{51} & \ding{51} & \ding{51}  &  37.370   &  0.929    &  3.636    &  2.825    &  41.307   &  0.972    &  0.804    &  0.904    &  34.063   &  0.913    & 5.531    &  4.402    &  47.614   &  0.991    &  0.640    &  0.887    &  40.937   &  0.955    & 2.427    &  2.079    \\  
			II&\ding{53} & \ding{51} & \ding{51}  &   33.497   &   0.857    &   5.482    &   3.937    &   37.909   &   0.953    &   1.141    &   1.329    &   29.067   &   0.778    &   8.615    &   5.171    &   42.349   &   0.982    &   1.198    &   1.494    &   36.529   &   0.903    &   3.767    &   2.854    \\  
			III&\ding{51} & \ding{51} & \ding{53}  &   35.919   &   0.917    &   4.205    &   3.253    &   38.119   &   0.957    &   1.125    &   1.246    &   32.382   &   0.891    &   7.924    &   5.680    &   43.731   &   0.984    &   1.021    &   1.452    &   38.346   &   0.942    &   3.158    &   2.660      \\  
			IV&\ding{51} & \ding{53} & \ding{51} &   36.255   &   0.920    &   4.541    &   3.066    &   39.199   &   0.959    &   1.090    &   1.108    &   32.785   &   0.898    &   8.125    &   5.099    &   44.793   &   0.987    &   0.881    &   1.139    &   39.063   &   0.946    &   3.252    &   2.383    \\   
			V&\ding{51} & \multicolumn{2}{c|}{\textbf{Add.}} &   33.776   &   0.865    &   5.645    &   3.382    &   37.698   &   0.953    &   1.220    &   1.251    &   29.360   &   0.803    &   9.375    &   5.099    &   44.324   &   0.983    &   1.125    &   1.332    &   37.278   &   0.910    &   3.925    &   2.584   \\  
			VI&\ding{51} & \multicolumn{2}{c|}{\textbf{Concat.}} &   33.574   &   0.865    &   6.040    &   3.688    &   38.149   &   0.954    &   1.120    &   1.275    &   29.439   &   0.804    &   10.972   &   4.996    &   44.047   &   0.983    &   1.004    &   1.414    &   37.210   &   0.911    &   4.235    &   2.704  \\  
			\bottomrule
	\end{tabular}}  
	
	\vspace{\dblfloatsep}
	
	\caption{Fusion performance of UniPS variants at full scale.}\label{tab:variantshr}
	\centering
	\renewcommand\arraystretch{1.0}
	\resizebox{2.00\columnwidth}{!}{
		\begin{tabular}{cccc|ccc|ccc|ccc|ccc|ccc}
			\toprule
			\multicolumn{4}{c|}{\textbf{Configurations}}& \multicolumn{3}{c|}{4 Bands} & \multicolumn{3}{c|}{7 Bands} & \multicolumn{3}{c|}{8 Bands} &  \multicolumn{3}{c|}{10 Bands} & \multicolumn{3}{c}{Average}  \\ 
			& MiT & $\mathcal{K}_{exp}$ & $\mathcal{K}_{geo}$ & QNR$\uparrow$ & $D_\lambda\downarrow$ & $D_s\downarrow$  & QNR$\uparrow$ & $D_\lambda\downarrow$ & $D_s\downarrow$ &  QNR$\uparrow$ & $D_\lambda\downarrow$ & $D_s\downarrow$ &  QNR$\uparrow$ & $D_\lambda\downarrow$ & $D_s\downarrow$  & QNR$\uparrow$ & $D_\lambda\downarrow$ & $D_s\downarrow$\\
			\midrule  
			I & \ding{51} & \ding{51} & \ding{51} & 0.9327   &   0.0170   &   0.0514   &   0.9009   &   0.0108   &   0.0893   &   0.9135   &   0.0186   &   0.0693   &   0.8679   &   0.0215   &   0.1131   &   0.9030   &   0.0176   &   0.0810   \\  
			II & \ding{53} & \ding{51} & \ding{51}  &   0.8902   &   0.0392   &   0.0742   &   0.8464   &   0.0450   &   0.1141   &   0.8773   &   0.0389   &   0.0875   &   0.8407   &   0.0493   &   0.1161   &   0.8642   &   0.0435   &   0.0970    \\  
			III & \ding{51} & \ding{51} & \ding{53}  &   0.9165   &   0.0313   &   0.0543   &   0.8774   &   0.0257   &   0.0995   &   0.8925   &   0.0353   &   0.0750   &   0.8534   &   0.0351   &   0.1157   &   0.8854   &   0.0321   &   0.0855  \\  
			IV & \ding{51} & \ding{53} & \ding{51} &   0.9211   &   0.0273   &   0.0535   &   0.8751   &   0.0271   &   0.1006   &   0.8996   &   0.0309   &   0.0718   &   0.8501   &   0.0352   &   0.1191   &   0.8865   &   0.0304   &   0.0861 \\   
			V & \ding{51} & \multicolumn{2}{c|}{\textbf{Add.}} &   0.9172   &   0.0289   &   0.0560   &   0.8701   &   0.0250   &   0.1076   &   0.8973   &   0.0291   &   0.0759   &   0.8449   &   0.0418   &   0.1185   &   0.8822   &   0.0324   &   0.0886  \\  
			VI & \ding{51} & \multicolumn{2}{c|}{\textbf{Concat.}} &   0.9082   &   0.0357   &   0.0589   &   0.8667   &   0.0264   &   0.1098   &   0.8904   &   0.0342   &   0.0781   &   0.8329   &   0.0509   &   0.1226   &   0.8735   &   0.0388   &   0.0917   \\  
			\bottomrule
	\end{tabular}}    
	
	\vspace{\dblfloatsep}
	
	\caption{Model efficiency and details of deep pansharpening methods.} 
	\label{tab:time}
	\resizebox{2.10\columnwidth}{!}{
		\begin{tabular}{c@{\,~}|c@{\,~}c@{\,~}c@{\,~}c@{\,~}c@{\,~}c@{\,~}c@{\,~}c@{\,~}c@{\,~}c@{\,~}c@{\,~}c@{\,~}c@{\,~}c@{\,~}c@{\,~}c@{\,~}c@{\,~}}
			\toprule
			Method   	&  PNN~\cite{masi2016pansharpening}  & PanNet~\cite{yang2017pannet}  & GPPNN~\cite{xu2021deep}  & U2Net~\cite{peng2023u2net}   &   PanFlowNet~\cite{yang2023panflownet}  &  P2Sharpen~\cite{zhang2023p2sharpen}  & DISPNet~\cite{wang2024deep}  &  CANConv~\cite{duan2024content} & PAPS~\cite{jia2024paps} & UniPAN~\cite{cui2025enpowering} & WFANet~\cite{huang2025wavelet} & CSLP~\cite{chen2025cslp} & DAISP~\cite{wang2025deep} & UniPS-T & UniPS-S & UniPS-B & UniPS-L \\
			\midrule
			Parameters (M) & 0.081 & 0.095  & 0.119 & 0.527 & 0.087 & 0.715 & 1.566 & 0.199 & 1.141 & 0.076 & 0.506 & 1.209 & 2.324 & 5.249 & 8.965 & 8.315 & 20.661
			\\
			FLOPs (G)      & 5.263 & 6.221  & 5.586 & 8.301 & 22.850 & 46.547 & 110.667 & 1.614 & 80.971 & 4.982 & 12.323 & 27.621 & 135.640 & 36.798& 105.383 & 47.057& 144.429
			\\
			GPU Memory (G) & 0.593 & 0.546  & 0.667 & 0.702 & 0.606 & 1.093 & 0.589  & 0.979 & 0.757 & 0.548 & 0.831 & 0.655 & 0.899 & 1.032 & 1.177 & 1.055 & 1.255
			\\
			Time (s) 	   & 0.061 & 0.082  & 0.093 & 0.164 & 0.130 & 0.173 & 0.179 & 0.217 & 0.167 & 0.076 & 0.166 & 0.130 & 0.183 & 0.946 & 1.313 & 2.048 & 2.735
			\\
			\bottomrule
	\end{tabular}}
\end{table*}
 
\subsection{Performance of UniPS variants}

\begin{figure}[t]  
	\centering
	\includegraphics[width=0.99\linewidth]{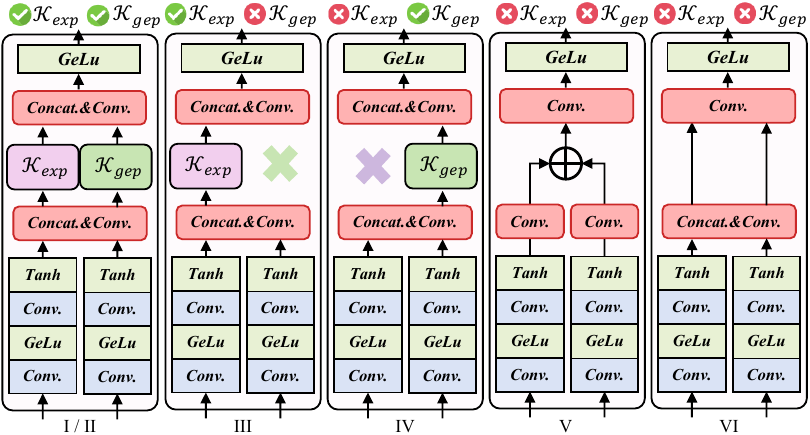}
	\caption{Feature interaction blocks of UniPS variants.}
	\label{exp:infblocks} 
\end{figure}

To investigate the effectiveness of the MiT and the feature interaction mechanism, we construct several UniPS variants with different designs. Concretely, we remove MiT and replace missing bands with zero padding to enable unified processing of heterogeneous inputs. In addition, we decouple the two infinite-dimensional kernels to evaluate their individual contributions. We further compare our parameter-free interaction strategy with conventional feature fusion schemes, including addition and concatenation. The detailed configurations of these variants are illustrated in Fig.~\ref{exp:infblocks}. The quantitative results are provided in Tab.~\ref{tab:variantslr} and Tab.~\ref{tab:variantshr} for reduced and full scales, respectively. At reduced scale, the complete UniPS (I) consistently achieves the best overall performance across all spectral settings. Removing MiT (II) results in a notable deterioration in all reference metrics. This demonstrates that simply padding missing bands with zeros is insufficient for heterogeneous spectral inputs, as different band configurations introduce imbalanced extended spectral representations and make unified feature optimization more challenging. In contrast, MiT learns modality-aware representations that preserve modality-specific characteristics while projecting heterogeneous spectral inputs into a balanced latent space. The comparison among I, III, and IV reveals the complementary roles of the exponential and geometric kernels. Using either kernel alone leads to a considerable performance decrease. These two kernels capture complementary interaction patterns between PAN and latent MS representations and are most effective when jointly employed. We further compare our interaction mechanism with conventional addition and concatenation operations. This suggests that simple feature aggregation is insufficient to characterize the complex correspondence between PAN and MS representations, whereas the parameter-free kernel interaction implicitly captures higher-order dependencies. The full-scale results exhibit a consistent trend. Variant I achieves the highest average QNR, outperforming the MiT-free variant and the addition- and concatenation-based variants. Moreover, jointly using the two kernels yields better average spectral and spatial consistency than employing either kernel individually. In conclusion, MiT is essential for accommodating heterogeneous spectral configurations. The exponential and geometric kernels provide complementary interaction cues and their joint parameter-free formulation is substantially more effective than conventional addition or concatenation. Together, these components enable UniPS to achieve robust and consistent fusion across varying bands.

\begin{figure*}[htp]  
	\centering
	\includegraphics[width=0.99\linewidth]{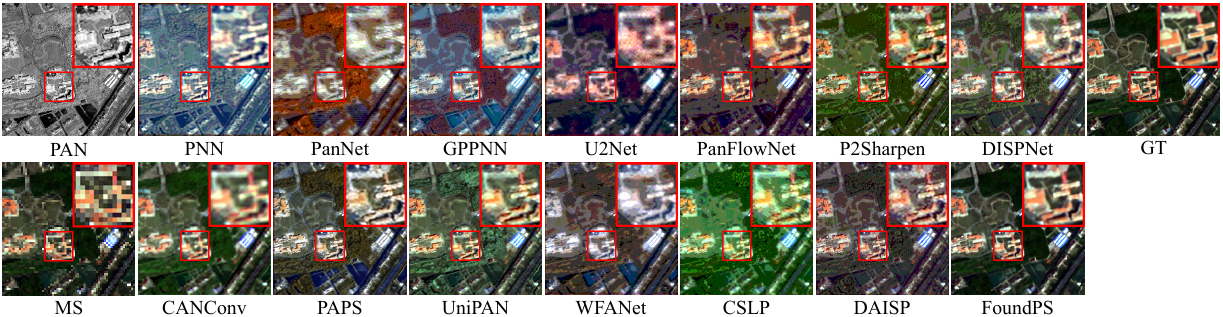}
	\caption{Visual comparisons of reduced scale on SegGF.}
	\label{exp:genGF_reduced} 
	 
	 \centering
	 \includegraphics[width=0.99\linewidth]{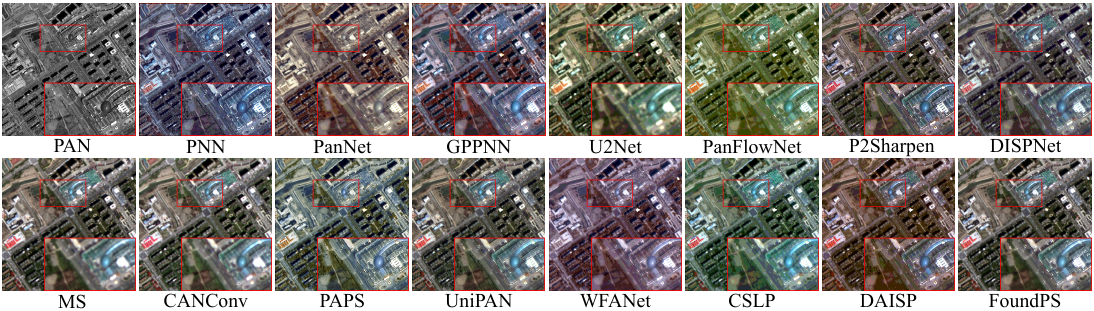}
	 \caption{Visual comparisons of full scale on SegGF.}
	 \label{exp:genGF_full} 
\end{figure*} 

\subsubsection{Scalability of UniPS} 
To demonstrate our scalability, we construct multiple UniPS variants by adjusting the feature channels in Infinite-UNet, thereby varying model capacity. Quantitative results on both reduced and full scales are summarized in Tab.~\ref{tab:reduced} and Tab.~\ref{tab:full}, respectively. The results clearly indicate that performance scales positively with parameter count, confirming the architecture's ability to leverage increased capacity. Notably, even the small variant UniPS-S remains competitive against other methods, underscoring our effectiveness. 
   
\begin{table*}[h] 
 	\centering
 	\caption{Zero-shot pansharpening performance on unseen scenes and satellite datasets.} 
 	\label{tab:zero}  
 	\renewcommand\arraystretch{1.1}
 	\resizebox{2.10\columnwidth}{!}{
	\begin{tabular}{c|ccccc|ccccc|ccccc|ccccc}
		\toprule
		\multirow{2}{*}{\textbf{Method}} & \multicolumn{5}{c|}{QuickBird~\cite{ma2020pan}} & \multicolumn{5}{c|}{SegGF~\cite{wang2025deep}} & \multicolumn{5}{c|}{Jiling~\cite{cao2026cross}} &  \multicolumn{5}{c}{skysat~\cite{cao2026cross}} \\  
		& PSNR$\uparrow$& SSIM$\uparrow$ & ERGAS$\downarrow$ & SAM$\downarrow$ & QNR$\uparrow$  & PSNR$\uparrow$& SSIM$\uparrow$ & ERGAS$\downarrow$ & SAM$\downarrow$ & QNR$\uparrow$ & PSNR$\uparrow$& SSIM$\uparrow$ & ERGAS$\downarrow$ & SAM$\downarrow$ & QNR$\uparrow$ & PSNR$\uparrow$& SSIM$\uparrow$ & ERGAS$\downarrow$ & SAM$\downarrow$ & QNR$\uparrow$\\
		\midrule
		PNN~\protect\cite{masi2016pansharpening} 
		&  24.9903   &   0.8396   &   7.5687   &  11.6871 &   0.4390 &  24.1732   &   0.6991   &   3.9805   &   5.0416   &   0.3723	&  24.8081   &   0.8318   &   6.2130   &   5.2326   &   0.8231   &  26.0699   &   0.8495   &   6.4022   &   5.3754   &   0.8319 \\
		PanNet~\protect\cite{yang2017pannet}
		&  21.6332   &   0.7963   &  12.6515   &  14.6404   &   0.4285	&  20.8404   &   0.5865   &   6.1647   &  10.2467   &   0.5113 	&  22.8129   &   0.7539   &   7.9233   &  10.1700   &   0.7261   &  24.0236   &   0.7826   &   7.8340   &  10.1602   &   0.7052 	 \\
		GPPNN~\protect\cite{xu2021deep}
		&  26.3027   &   0.8616   &   6.9445   &  11.7491   &   0.7272	&  24.0514   &   0.6990   &   4.2215   &   5.6822   &   0.6025 &  25.2377   &   0.7580   &   4.9689   &   4.7102   &   0.7855   &  26.9535   &   0.7947   &   4.7313   &   4.8363   &   0.7915 	\\
		U2Net~\protect\cite{peng2023u2net}
		&  27.4097   &   0.7240   &   5.5303   &   8.7204   &   0.7961	&  22.4246   &   0.7077   &   5.7532   &   6.6718   &   \underline{0.8453}	&  26.6666   &   \underline{0.8405}   &   4.8783   &   5.2118   &   0.7506   &  27.6754   &   \underline{0.8642}   &   4.7137   &   5.7169   &   0.7776 	\\
		PanFlowNet~\protect\cite{yang2023panflownet}
		&  21.4585   &   0.7554   &  17.6751   &  13.2153   &   0.6191	&  12.7717   &   0.4703   &  30.0502   &  14.0260   &   0.3733 &  26.3085   &   0.7826   &   4.6496   &   3.5070   &   0.7918   &  28.7214   &   0.8209   &   4.2169   &   3.2943   &   0.8109  \\
		P2Sharpen~\protect\cite{zhang2023p2sharpen}
		&  33.3422   &   0.8840   &   \underline{2.8564}   &   3.3286   &   0.6301	&  22.3037   &   0.7225   &   4.1746   &   3.5671   &   0.5131	&  26.8236   &   0.7714   &   4.5884   &   3.4714   &   0.7981   &  28.4640   &   0.8104   &   4.7306   &   3.3967   &   0.8215  \\
		DISPNet~\protect\cite{wang2024deep}
		&  31.0746   &   0.8446   &   3.4811   &   4.6411   &   0.8618	&  24.2749   &   0.7382   &   4.7943   &   3.9017   &   0.8366	 &  27.5673   &   0.7919   &   4.2665   &   2.8796   &   \underline{0.8320}   &  30.2138   &   0.8311   &   3.7730   &   2.7048   &   \underline{0.8506} 	\\  
		CANConv~\protect\cite{duan2024content}
		&  \underline{33.5858}   &  \underline{0.8930}   &   2.9639   &   \underline{3.1327}   &   \underline{0.8863}  
		&  27.0788   &   \underline{0.8139}   &   3.1013   &   4.0545   &   0.8417
		&  27.7639   &   0.7911   &   \underline{3.9329}   &   \underline{2.5156}   &   0.7788   &  30.1067   &   0.8261   &   \underline{3.6146}   &   2.4056   &   0.7983 \\
		PAPS~\protect\cite{jia2024paps} 
		&  29.9086   &   0.8828   &   4.0986   &   7.8094   &   0.7588
		&  22.2894   &   0.7305   &   4.5931   &   8.8448   &   0.5202 &  25.4502   &   0.7577   &   5.4988   &   4.1382   &   0.6663   &  27.6202   &   0.7993   &   4.9407   &   3.8745   &   0.6986
		\\
		UniPAN~\protect\cite{cui2025enpowering}
		&  29.5113   &   0.8885   &   4.8035   &   3.2952   &   0.6353	&  27.8065   &   0.7565   &   2.7156   &   \underline{2.4996}   &   0.5909 &  26.6694   &   0.7869   &   4.7208   &   2.9249   &   0.7041   &  28.0552   &   0.8083   &   5.2080   &   2.9964   &   0.7414  \\
		WFANet~\protect\cite{huang2025wavelet}
		&  26.2470   &   0.8453   &   6.7174   &  10.7235   &   0.6647 &  18.9564   &   0.6613   &   5.7413   &   7.1638   &   0.6246 &  \underline{28.0767}   &   0.8232   &   4.1088   &   2.6809   &   0.8089   &  30.7180   &   0.8571   &   3.6163   &   2.4651   &   0.8389 \\
		CSLP~\protect\cite{chen2025cslp}
		&  31.2616   &   0.8801   &   4.3995   &   3.7732   &   0.7919	&  \underline{28.9008}   &   0.7695   &   \underline{2.2789}   &   2.5067   &   0.7327	&  26.8386   &   0.7763   &   4.9468   &   2.8636   &   0.7094   &  29.2750   &   0.8157   &   4.4504   &   2.6734   &   0.7438\\
		DAISP~\protect\cite{wang2025deep}
		&  25.7630   &   0.8054   &   9.0269   &   5.0775   &   0.6924	&  19.7114   &   0.6866   &   9.0789   &   4.9151   &   0.6437 &  27.4855   &   0.7923   &   4.4433   &   2.5966   &   0.7948   &  \underline{30.6517}   &   0.8323   &   3.8121   &   \underline{2.3298}   &   0.8223 \\
		UniPS
		&  \textbf{38.9823}   &   \textbf{0.9545}   &   \textbf{1.4540}   &   \textbf{1.7429}   &   \textbf{0.9224}  &  \textbf{35.2662}   &   \textbf{0.9369}   &   \textbf{1.1393}   &   \textbf{1.3139}   &   \textbf{0.9002} &  \textbf{30.2701}   &   \textbf{0.8861}   &   \textbf{2.9349}   &   \textbf{2.4837}   &   \textbf{0.9130}   &  \textbf{32.6353}   &   \textbf{0.9055}   &   \textbf{2.6951}   &   \textbf{2.2622}   &   \textbf{0.9226}\\
		\bottomrule
 	\end{tabular}}  
\end{table*}

\begin{table*}
	\centering
	\renewcommand\arraystretch{1.2}
	\caption{Quantitative comparisons on reduced scale over the truncated 4-band PSBench. All of deep methods are retrained on the mixed truncated datasets, denoted by $\ast$.}
	\label{tab:truncatedlr}
	\resizebox{2.10\columnwidth}{!}{
		\begin{tabular}{c|c@{\,~}c@{\,~}c@{\,~}c@{\,~}|c@{\,~}c@{\,~}c@{\,~}c@{\,~}|c@{\,~}c@{\,~}c@{\,~}c@{\,~}|c@{\,~}c@{\,~}c@{\,~}c@{\,~}|c@{\,~}c@{\,~}c@{\,~}c@{\,~}}
			\thickhline
			&  \multicolumn{4}{c|}{4 Bands} & \multicolumn{4}{c|}{Truncated 7 Bands} & \multicolumn{4}{c|}{Truncated 8 Bands} &  \multicolumn{4}{c|}{Truncated 10 Bands} & \multicolumn{4}{c}{Average}  \\ 
			\multirow{-2}{*}{\textbf{Method}}   & PSNR$\uparrow$	  & SSIM$\uparrow$ & ERGAS$\downarrow$ & SAM$\downarrow$ & PSNR$\uparrow$	  & SSIM$\uparrow$ & ERGAS$\downarrow$ & SAM$\downarrow$ & PSNR$\uparrow$	  & SSIM$\uparrow$ & ERGAS$\downarrow$ & SAM$\downarrow$ & PSNR$\uparrow$	  & SSIM$\uparrow$ & ERGAS$\downarrow$ & SAM$\downarrow$ & PSNR$\uparrow$	  & SSIM$\uparrow$ & ERGAS$\downarrow$ & SAM$\downarrow$\\
			\thickhline
			$\ast$PNN~\protect\cite{masi2016pansharpening}&   31.547   &   0.861    &   10.794   &   5.986    &   26.731   &   0.893    &   4.745    &   4.767    &   30.510   &   0.837    &   12.725   &   6.909    &   34.104   &   0.955    &   3.790    &   3.736    &   31.364   &   0.894    &   7.703    &   5.168  \\
			$\ast$PanNet~\protect\cite{yang2017pannet} &   32.325   &   0.875    &   6.368    &   7.615    &   30.029   &   0.933    &   2.837    &   5.576    &   31.058   &   0.869    &   7.385    &   8.240    &   37.797   &   0.972    &   2.022    &   4.157    &   33.513   &   0.916    &   4.463    &   6.211    \\
			$\ast$GPPNN~\protect\cite{xu2021deep} &   32.943   &   0.878    &   5.101    &   5.212    &   31.818   &   0.937    &   2.014    &   3.872    &   31.857   &   0.883    &   6.404    &   5.872    &   38.459   &   0.970    &   1.527    &   3.077    &   34.388   &   0.920    &   3.570    &   4.370 \\
			$\ast$U2Net~\protect\cite{peng2023u2net} &   33.081   &   0.849    &   6.156    &   4.608    &   33.197   &   0.923    &   1.881    &   3.338    &   29.632   &   0.785    &   8.527    &   5.632    &   41.810   &   0.970    &   1.253    &   1.940    &   35.451   &   0.892    &   4.138    &   3.659   \\
			$\ast$PanFlowNet~\protect\cite{yang2023panflownet} &   34.306   &   0.886    &   5.526    &   4.696    &   36.174   &   0.945    &   1.310    &   1.698    &   31.952   &   0.867    &   6.914    &   6.316    &   41.120   &   0.971    &   1.396    &   2.169    &   36.523   &   0.922    &   3.626    &   3.573    \\
			$\ast$P2Sharpen~\protect\cite{zhang2023p2sharpen} &   34.959   &   0.887    &   6.522    &   3.918    &   36.537   &   0.963    &   1.300    &   1.505    &   32.113   &   0.877    &   7.039    &   5.382    &   42.698   &   0.981    &   1.088    &   1.571    &   37.355   &   0.930    &   3.888    &   2.935 \\
			$\ast$DISPNet~\protect\cite{wang2024deep}&   34.686   &   0.893    &   4.481    &   3.810    &   39.144   &   0.957    &   0.963    &   1.317    &   32.391   &   0.892    &   5.681    &   4.818    &   43.355   &   0.978    &   0.970    &   1.670    &   37.981   &   0.932    &   2.878    &   2.811    \\  
			$\ast$CANConv~\protect\cite{duan2024content} &   33.769   &   0.876    &   4.779    &   3.726    &   39.649   &   0.956    &   0.933    &   1.012    &   32.576   &   0.888    &   5.680    &   4.788    &   41.277   &   0.973    &   1.232    &   1.134    &   37.103   &   0.924    &   3.061    &   2.548    \\
			$\ast$PAPS~\protect\cite{jia2024paps} &   35.222   &   0.902    &   4.199    &   3.949    &   37.640   &   0.950    &   1.140    &   1.601    &   33.174   &   0.900    &   5.286    &   4.861    &   42.621   &   0.978    &   1.035    &   1.681    &   37.773   &   0.935    &   2.778    &   2.922   \\
			$\ast$UniPAN~\protect\cite{cui2025enpowering} &   35.273   &   0.903    &   4.893    &   3.313    &   37.456   &   0.955    &   1.195    &   1.078    &   33.193   &   0.905    &   5.456    &   4.198    &   43.671   &   0.971    &   0.998    &   1.089    &   38.105   &   0.936    &   3.041    &   2.317  \\
			$\ast$WFANet~\protect\cite{huang2025wavelet}&   36.030   &   0.907    &   5.140    &   3.834    &   39.987   &   \underline{0.964}    &   0.892    &   1.262    &   33.351   &   0.891    &   6.992    &   5.239    &   44.440   &   \underline{0.983}    &   0.958    &   1.679    &   39.094   &   0.940    &   3.286    &   2.876  \\
			$\ast$CSLP~\protect\cite{chen2025cslp} &  \underline{36.137}    &   0.912    &   5.310    &   \underline{3.115}    &   39.857   &   0.959    &   0.937    &   \underline{0.984}    &   \underline{33.685}   &   0.908    &   5.553    &   4.169    &   \underline{45.170}   &   0.981    &   0.872    &   \underline{1.035}    &   \underline{39.399}   &   0.943    &   3.110    &   \underline{2.208}    \\
			$\ast$DAISP~\protect\cite{wang2025deep}&   36.068   &   \underline{0.913}    &   \underline{4.009}    &   3.116    &   \underline{40.096}   &   0.959    &   \underline{0.881}    &   1.109    &   33.594   &   \underline{0.909}    &   \underline{5.184}    &   \underline{4.076}    &   44.915   &   0.982    &   \underline{0.810}    &   1.195    &   39.318   &   \underline{0.944}    &   \underline{2.574}    &   2.269   \\
			UniPS &   \textbf{37.370}   &   \textbf{0.929}    &   \textbf{3.636}    &   \textbf{2.825}    &   \textbf{42.462}   &   \textbf{0.973}    &   \textbf{0.675}    &   \textbf{0.866}   &   \textbf{34.798}   &   \textbf{0.925}    &  \textbf{ 4.456}    &   \textbf{3.708}      &   \textbf{47.367}   &   \textbf{0.988}    &   \textbf{0.642}    &   \textbf{1.003}    &   \textbf{41.171}   &   \textbf{0.955}    &   \textbf{2.248}    &   \textbf{2.009}  \\
			\thickhline
		\end{tabular}
	}
	\vspace{\dblfloatsep}
	\centering
	\renewcommand\arraystretch{1.2}
	\caption{Quantitative comparisons on full scale over the truncated 4-band PSBench. All of deep methods are retrained on the mixed truncated datasets, denoted by $\ast$.}
	\label{tab:truncatedhr} 
	\resizebox{2.10\columnwidth}{!}{
		\begin{tabular}{c|c@{\,~}c@{\,~}c@{\,~}|c@{\,~}c@{\,~}c@{\,~}|c@{\,~}c@{\,~}c@{\,~}|c@{\,~}c@{\,~}c@{\,~}|c@{\,~}c@{\,~}c@{\,~}}
			\thickhline
			&  \multicolumn{3}{c|}{4 Bands} & \multicolumn{3}{c|}{7 Bands} & \multicolumn{3}{c|}{8 Bands} &  \multicolumn{3}{c|}{10 Bands} & \multicolumn{3}{c}{Average}  \\ 
			\multirow{-2}{*}{\textbf{Method}}   & QNR$\uparrow$ & $D_\lambda\downarrow$ & $D_s\downarrow$  & QNR$\uparrow$ & $D_\lambda\downarrow$ & $D_s\downarrow$ &  QNR$\uparrow$ & $D_\lambda\downarrow$ & $D_s\downarrow$ &  QNR$\uparrow$ & $D_\lambda\downarrow$ & $D_s\downarrow$  & QNR$\uparrow$ & $D_\lambda\downarrow$ & $D_s\downarrow$\\
			\thickhline
			$\ast$PNN~\protect\cite{masi2016pansharpening}&   0.8425   &   0.0793   &   0.0851   &   0.8021   &   0.0926   &   0.1164   &   0.8361   &   0.0730   &   0.0984   &   0.8275   &   0.0721   &   0.1086   &   0.8294   &   0.0784   &   0.1004    \\
			$\ast$PanNet~\protect\cite{yang2017pannet} &   0.7696   &   0.1282   &   0.1186   &   0.6858   &   0.1596   &   0.1870   &   0.8188   &   0.1024   &   0.0884   &   0.8210   &   0.0925   &   0.0960   &   0.7786   &   0.1184   &   0.1191  \\
			$\ast$GPPNN~\protect\cite{xu2021deep}&   0.8117   &   0.0872   &   0.1120   &   0.7783   &   0.0894   &   0.1468   &   0.8534   &   0.0738   &   0.0796   &   0.8402   &   0.0606   &   \underline{0.1059}   &   0.8212   &   0.0769   &   0.1115 \\
			$\ast$U2Net~\protect\cite{peng2023u2net} &   0.8682   &   0.0752   &   \underline{0.0615}   &   0.8115   &   0.1133   &   \textbf{0.0853}   &   0.8327   &   0.0628   &   0.1117   &   0.8628   &   0.0613   &   0.0813   &   0.8511   &   0.0756   &   \underline{0.0796} \\
			$\ast$PanFlowNet~\protect\cite{yang2023panflownet} &   0.7527   &   0.1129   &   0.1530   &   0.8439   &   0.0627   &   0.0997   &   0.7603   &   0.0930   &   0.1655   &   0.7564   &   0.1087   &   0.1535   &   0.7713   &   0.0996   &   0.1455  \\
			$\ast$P2Sharpen~\protect\cite{zhang2023p2sharpen} &   0.8092   &   0.0544   &   0.1446   &   0.8550   &   0.0283   &   0.1203   &   0.8557   &   0.0526   &   0.0977   &   0.8676   &   0.0249   &   0.1102   &   0.8433   &   0.0398   &   0.1221 \\
			$\ast$DISPNet~\protect\cite{wang2024deep}&   0.8337   &   0.0612   &   0.1123   &   0.8416   &   0.0450   &   0.1190   &   0.8852   &   0.0520   &   \underline{0.0666}   &   0.8667   &   0.0215   &   0.1143   &   0.8535   &   0.0440   &   0.1075   \\  
			$\ast$CANConv~\protect\cite{duan2024content}  &   0.7304   &   0.0917   &   0.1976   &   0.7843   &   0.0666   &   0.1619   &   0.8634   &   0.0643   &   0.0778   &   0.8473   &   0.0447   &   0.1132   &   0.7978   &   0.0678   &   0.1460  \\
			$\ast$PAPS~\protect\cite{jia2024paps}&   0.8431   &   0.0761   &   0.0890   &   0.7677   &   0.0848   &   0.1623   &   0.8355   &   0.0836   &   0.0891   &   0.8446   &   0.0404   &   0.1201   &   0.8290   &   0.0671   &   0.1123     \\
			$\ast$UniPAN~\protect\cite{cui2025enpowering} &   0.8430   &   0.0512   &   0.1130   &   0.7997   &   0.0755   &   0.1364   &   0.8811   &   0.0488   &   0.0745   &   0.8148   &   0.0644   &   0.1296   &   0.8316   &   0.0595   &   0.1169    \\
			$\ast$WFANet~\protect\cite{huang2025wavelet}  &   0.8664   &   0.0560   &   0.0831   &  \underline{0.8686}   &   \underline{0.0214}   &   0.1125   &   0.8643   &   0.0537   &   0.0868   &   \underline{0.8741}   &   0.0218   &   0.1063   &   \underline{0.8690}   &   0.0383   &   0.0965   \\
			$\ast$CSLP~\protect\cite{chen2025cslp} &   0.8425   &   0.0576   &   0.1072   &   0.8050   &   0.0538   &   0.1497   &   0.8814   &   0.0491   &   0.0736   &   0.8668   &   0.0320   &   \textbf{0.1046}   &   0.8495   &   0.0473   &   0.1090  \\
			$\ast$DAISP~\protect\cite{wang2025deep} &   \underline{0.8792}   &   \underline{0.0433}   &   0.0816  &   0.8237   &   0.0488   &   0.1346   &   \underline{0.8881}   &   \underline{0.0397}   &   0.0754   &   0.8699   &   \underline{0.0166}   &   0.1154   &   0.8675   &   \underline{0.0350}   &   0.1013 \\
			UniPS &   \textbf{0.9327}   &   \textbf{0.0170}   &   \textbf{0.0514}   &   \textbf{0.8932}   &   \textbf{0.0138}   &   \underline{0.0942}   &   \textbf{0.9107}   &   \textbf{0.0304}   &   \textbf{0.0606}   &   \textbf{0.8748}  &   \textbf{0.0158}   &   0.1113   &   \textbf{0.9034}   &   \textbf{0.0180}   &   \textbf{0.0800}  \\
			\thickhline
		\end{tabular}
	}
\end{table*} 

\subsection{Efficiency Analysis} 
We conduct the efficiency analysis to record model parameters, floating point operations (FLOPs), GPU memory usage, and time consumption of the deep pansharpening method on 4-band data of PSBench with PAN size of $256 \times 256$ and MS size of $64 \times 64 \times 4$ (Tab.~\ref{tab:time}). Most deep learning methods exhibit high processing speeds for test image samples. However, practical deployment requires training and maintaining separate models for individual satellite sensors, severely limiting their scalability. In comparison, UniPS introduces additional computational overhead due to its band-agnostic representation learning and diffusion-based fusion, which sacrifices some per-sample efficiency. In return, it offers stronger universality and scalability, while consistently achieving competitive performance across diverse tasks.

\subsection{Zero-shot Generalization} 
To verify the generalization capability of our universal pansharpening paradigm, we evaluate all pretrained models on external public datasets without retraining. Specifically, SegGF~\cite{wang2025deep}, which is captured by GaoFen-2 and excluded from PSBench, is utilized to assess unseen-scene generalization. QuickBird~\cite{ma2020pan} and the Jilin-series and SkySat-series datasets from PanScale~\cite{cao2026cross} are further adopted to evaluate cross-satellite generalization. As reported in Tab.~\ref{tab:zero}, UniPS consistently and significantly outperforms all task-specific methods across these unseen scenarios. Visual comparisons are presented in Fig.~\ref{exp:genGF_reduced} and Fig.~\ref{exp:genGF_full} for reduced and full scales, respectively, with additional results provided in Appendix~\ref{sec:suppl_D}. These results demonstrate that our method retains substantial advantages over existing methods when generalized to unseen scenes and satellites.
  
\subsection{Band Truncation Generalization}  
To further verify the superiority of UniPS, we conduct an additional experiment under a unified 4-band setting after spectral truncation. Specifically, for the 7-, 8-, and 10-band configurations in PSBench, we manually select the commonly shared B/G/R/NIR bands and discard the remaining bands. In this way, all datasets are converted into a consistent 4-band format, which allows existing methods to be retrained under the same input formats. Therefore, this setting removes the input-format mismatch among different satellites and provides a more direct evaluation of cross-sensor and cross-scene generalization. For fairness, all competing deep-learning-based methods are retrained on the mixed 4-band truncated datasets, denoted by $\ast$ in Tab.~\ref{tab:truncatedlr} and Tab.~\ref{tab:truncatedhr}. In contrast, UniPS keeps the same pretrained model without any retraining or additional adaptation. The reduced-scale results are reported in Tab.~\ref{tab:truncatedlr}. UniPS achieves the best average performance across all four metrics on all truncated configurations, demonstrating its consistent advantage under the unified band setting. The full-scale results in Tab.~\ref{tab:truncatedhr} further confirm this conclusion. UniPS achieves the highest average QNR and the lowest average $D_\lambda$ and $D_s$. These results indicate that UniPS maintains stronger spectral consistency and spatial fidelity even when all methods are re-optimized under the truncated datasets. In conclusion, manually truncating spectral bands can partially alleviate the input incompatibility of existing methods, but it does not fundamentally solve their limited generalization ability. UniPS still consistently outperforms them without additional fine-tuning. This verifies that our universal pansharpening paradigm learns more transferable spatial-spectral fusion priors.  

\begin{table*}[t]
	\caption{Pansharpening performance at reduced scale under different mask ratio.}
	\label{tab:masklr}
	\centering
	\renewcommand\arraystretch{1.0}
	\resizebox{2.10\columnwidth}{!}{
		\begin{tabular}{c|cccc|cccc|cccc|cccc|cccc}
			\toprule
			Mask & \multicolumn{4}{c|}{4 Bands} & \multicolumn{4}{c|}{7 Bands} & \multicolumn{4}{c|}{8 Bands} &  \multicolumn{4}{c|}{10 Bands} & \multicolumn{4}{c}{Average}  \\ 
			Ratio & PSNR$\uparrow$	  & SSIM$\uparrow$ & ERGAS$\downarrow$ & SAM$\downarrow$ & PSNR$\uparrow$	  & SSIM$\uparrow$ & ERGAS$\downarrow$ & SAM$\downarrow$ & PSNR$\uparrow$	  & SSIM$\uparrow$ & ERGAS$\downarrow$ & SAM$\downarrow$ & PSNR$\uparrow$	  & SSIM$\uparrow$ & ERGAS$\downarrow$ & SAM$\downarrow$ & PSNR$\uparrow$	  & SSIM$\uparrow$ & ERGAS$\downarrow$ & SAM$\downarrow$\\
			\midrule 
			0\% &  37.370   &  0.929    &  3.636    &  2.825    &  41.307   &  0.972    &  0.804    &  0.904    &  34.063   &  0.913    & 5.531    &  4.402    &  47.614   &  0.991    &  0.640    &  0.887    &  40.937   &  0.955    & 2.427    &  2.079   \\ 
			10\% &   36.918   &   0.927    &   3.661    &   2.998    &   39.174   &   0.960    &   1.078    &   1.181    &   33.665   &   0.910    &   5.760    &   4.768    &   45.527   &   0.986    &   0.873    &   1.156    &   39.657   &   0.950    &   2.594    &   2.330   \\
			20\%  &   35.874   &   0.919    &   4.453    &   3.795    &   36.002   &   0.950    &   2.204    &   2.182    &   32.889   &   0.902    &   6.982    &   5.719    &   43.049   &   0.982    &   1.465    &   1.907    &   37.802   &   0.943    &   3.444    &   3.171   \\
			40\% &   33.879   &   0.895    &   8.601    &   5.797    &   29.589   &   0.911    &   27.523   &   6.199    &   30.716   &   0.862    &   18.928   &   8.953    &   36.674   &   0.965    &   4.370    &   4.961    &   33.543   &   0.915    &   12.166   &   6.066   \\
			60\% &   31.583   &   0.851    &   22.574   &   8.842    &   22.313   &   0.831    &  165.034   &   14.053   &   27.772   &   0.788    &   58.058   &   14.934   &   28.278   &   0.906    &   15.746   &   13.660   &   28.281   &   0.856    &   51.053   &   12.249 \\
			\bottomrule
	\end{tabular}} 
	\vspace{\dblfloatsep} 
	\caption{Pansharpening performance at full scale under different mask ratio.}
	\label{tab:maskhr}
	\centering
	\renewcommand\arraystretch{1.0}
	\resizebox{2.10\columnwidth}{!}{
		\begin{tabular}{c|ccc|ccc|ccc|ccc|ccc}
			\toprule
			Mask& \multicolumn{3}{c|}{4 Bands} & \multicolumn{3}{c|}{7 Bands} & \multicolumn{3}{c|}{8 Bands} &  \multicolumn{3}{c|}{10 Bands} & \multicolumn{3}{c}{Average}  \\ 
			Ratio & QNR$\uparrow$ & $D_\lambda\downarrow$ & $D_s\downarrow$  & QNR$\uparrow$ & $D_\lambda\downarrow$ & $D_s\downarrow$ &  QNR$\uparrow$ & $D_\lambda\downarrow$ & $D_s\downarrow$ &  QNR$\uparrow$ & $D_\lambda\downarrow$ & $D_s\downarrow$  & QNR$\uparrow$ & $D_\lambda\downarrow$ & $D_s\downarrow$  \\
			\midrule 
			0\%   & 0.9327   &   0.0170   &   0.0514   &   0.9009   &   0.0108   &   0.0893   &   0.9135   &   0.0186   &   0.0693   &   0.8679   &   0.0215   &   0.1131   &   0.9030   &   0.0176   &   0.0810    \\
			10\% &   0.7733   &   0.0607   &   0.1770   &   0.8618   &   0.0325   &   0.1095   &   0.8665   &   0.0487   &   0.0898   &   0.8563   &   0.0365   &   0.1114   &   0.8300   &   0.0460   &   0.1306 \\
			20\% &   0.7589   &   0.0725   &   0.1827   &   0.8285   &   0.0555   &   0.1235   &   0.8532   &   0.0597   &   0.0932   &   0.8285   &   0.0584   &   0.1207   &   0.8080   &   0.0630   &   0.1387   \\
			40\% &   0.7222   &   0.0978   &   0.2018   &   0.7640   &   0.1046   &   0.1498   &   0.8140   &   0.0844   &   0.1133   &   0.7555   &   0.1151   &   0.1491   &   0.7541   &   0.1026   &   0.1623  \\
			60\% &   0.6866   &   0.1256   &   0.2187   &   0.6647   &   0.1822   &   0.1939   &   0.7287   &   0.1427   &   0.1559   &   0.6567   &   0.1901   &   0.1945   &   0.6791   &   0.1594   &   0.1971 \\
			\bottomrule
	\end{tabular}}
\end{table*} 
\subsection{Hyperspectral Generalization} 
To further showcase the compatibility of UniPS with unseen spectral bands, we conduct hyperspectral generalization on the Real HSI-MSI-PAN dataset~\cite{li2024real}, acquired by the Ziyuan-1 02D satellite. The dataset provides paired PAN and hyperspectral images with a spatial resolution ratio of $12\times$, where each HSI contains 76 spectral bands. The number of spectral bands in this dataset substantially exceeds the spectral configurations and expert assignments. Following the UniPS setting, we sequentially partition the 76 HSI bands into consecutive groups of four and fuse each group with the corresponding PAN image, without retraining or modifying the network architecture. For visualization, the 35th, 16th, and 7th bands of the original HSI are used to form the RGB composite. Notably, this experiment presents a dual generalization challenge: the model must accommodate previously unseen spectral bands while handling a substantially larger spatial resolution gap than those encountered during training. The qualitative results are shown in Fig.~\ref{exp:HyperGen}. UniPS successfully reconstructs high-resolution hyperspectral representations with clear spatial structures and well-preserved spectral characteristics, demonstrating its capability to generalize beyond predefined spectral configurations.

\begin{figure}[t]
	\centering
	\includegraphics[width=0.97\linewidth]{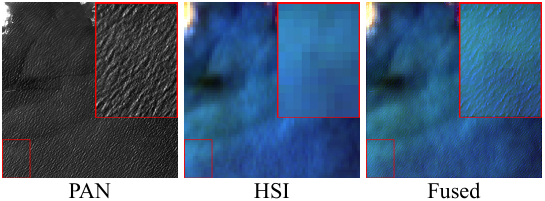}
	\vspace*{-0.4em}
	\caption{Hyperspectral generalization results.}
	\label{exp:HyperGen}
\end{figure}
\begin{figure*}[t]
	\centering
	\includegraphics[width=1\linewidth]{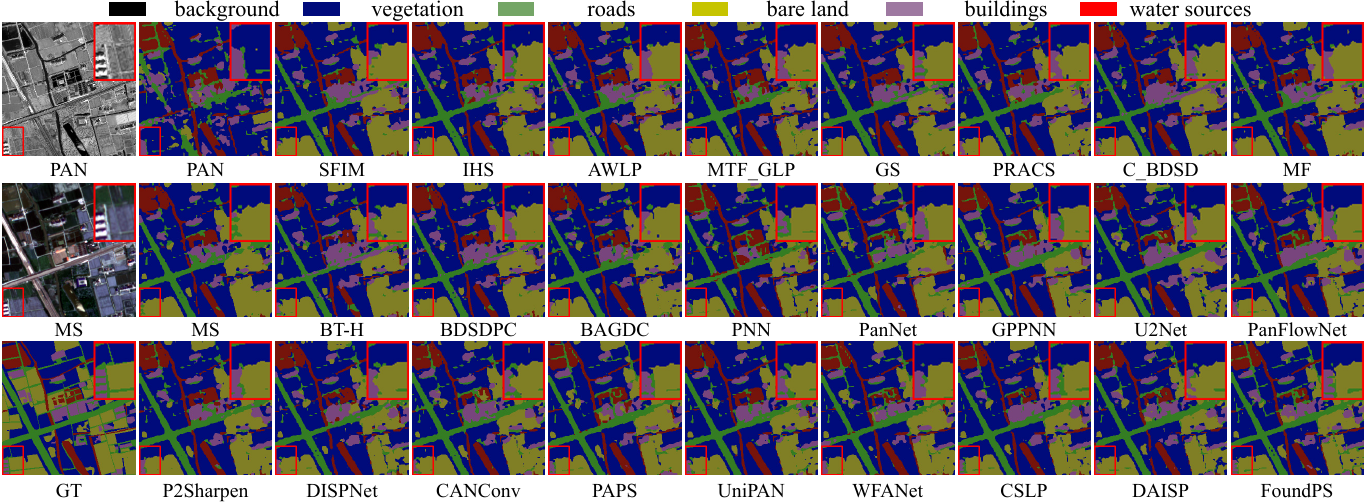}
	\vspace*{-0.4em}
	\caption{Reduced-scale segmentation comparisons on SegGF.}
	\label{exp:GFseg_reduced}
\end{figure*}

\begin{figure*}[ht]
	\centering
	\includegraphics[width=1\linewidth]{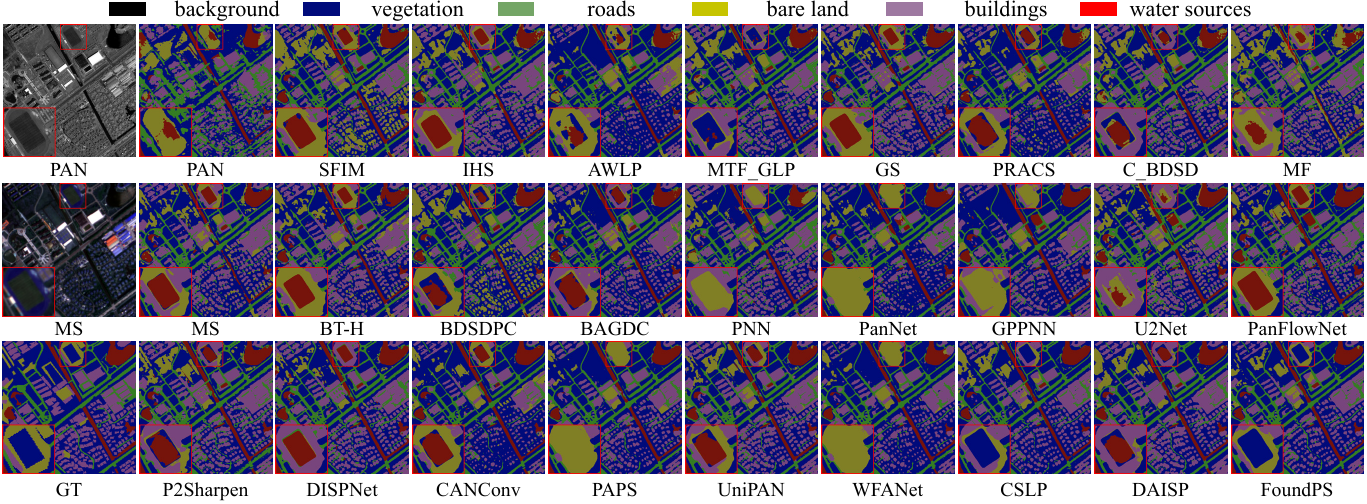}
	\vspace*{-0.4em}
	\caption{Full-scale segmentation comparisons on SegGF.}
	\label{exp:GFseg_full}
\end{figure*}  

\subsection{Band Corruption Completion}
To evaluate the robustness of UniPS, we randomly mask a portion of MS bands during inference. Specifically, each MS band is masked with probability determined by the mask ratio, and the selected bands are replaced by zero padding while their original spectral positions and the overall input dimensionality are retained. Therefore, the model is aware of which spectral positions are unavailable and is required to reconstruct the complete high-resolution MS image using the remaining valid bands with the PAN image. The reduced-scale results are reported in Tab.~\ref{tab:masklr}. As expected, the reconstruction quality gradually decreases as the mask ratio increases because progressively less spectral information is available. Nevertheless, the performance remains relatively stable under low mask ratios. For example, when 10\% of bands are masked, the model still achieves competitive performance. This result indicates that UniPS does not process each band independently. Instead, it captures intrinsic dependencies among spectral bands and exploits the PAN observation to compensate, to some extent, for locally missing spectral information. The masking operation also affects the expert-routing process in MiT. Under a low mask ratio, most of the cross-band dependency structure remains intact. The routing module can therefore infer a reliable latent representation from the observed bands, and the selected experts remain compatible with the underlying spectral configuration. However, when the mask ratio increases to 40\% or 60\%, a substantial portion of the spectral dependency structure is destroyed. The remaining bands no longer provide sufficient contextual evidence for reliable expert assignment and latent-space construction. Consequently, the routing process becomes uncertain, and the model performance drops significantly. The full-scale results in Tab.~\ref{tab:maskhr} show a consistent trend. UniPS obtains the best performance when all bands are available. With increasing mask ratios, QNR gradually decreases while both $D_\lambda$ and $D_s$ increase accordingly. Overall, UniPS is robust to mild spectral corruption and can exploit the intrinsic correlations among MS bands and the complementary spatial information provided by PAN. In conclusion once excessive masking severely disrupts the inter-band dependency structure, expert routing and missing-band inference become underdetermined, resulting in substantial performance degradation.

\begin{table*}[t]
	\centering 
	\caption{Segmentation performance on the SegGF dataset. \textbf{Bold} means the best average segmentation performance.}
	\vspace*{-0.4em}
	\label{tab:segment}
	\resizebox{2.10\columnwidth}{!}{
		\begin{tabular}{c@{\,~}|c@{\,~}c@{\,~}c@{\,~}c@{\,~}c@{\,~}c@{\,~}|c@{\,~}c@{\,~}c@{\,~}c@{\,~}c@{\,~}c@{\,~}|c@{\,~}c@{\,~}c@{\,~}c@{\,~}c@{\,~}c@{\,~}|c@{\,~}c@{\,~}c@{\,~}c@{\,~}c@{\,~}c@{\,~}}
			\toprule
			\multirow{3}{*}{\textbf{Method}}&\multicolumn{12}{c}{Reduced-scale}&\multicolumn{12}{c}{Full-scale}\\
			&\multicolumn{6}{c}{Accuracy}&\multicolumn{6}{c}{IoU}&\multicolumn{6}{c}{Accuracy}&\multicolumn{6}{c}{IoU}\\
			\cmidrule(lr){2-25}&Avg.&Land.&Build.&Road.&Vega.&Water.&Avg&Land.&build.&Road.&Vega.&water.&Avg.&Land.&Build.&Road.&Vega.&Water.&Avg.&Land.&build.&Road.&Vega.&water.\\
			\midrule
			PAN	&46.29&37.17&60.37&36.82&67.57&75.82&31.38&31.31&38.72&25.75&42.35&50.18&50.09&50.08&67.24&58.20&58.32&66.72&35.48&38.58&45.80&35.49&41.89&51.12\\			
			LRMS&54.84&71.18&59.27&47.72&83.13&67.76&42.52&57.76&47.60&33.16&62.05&42.35&57.15&73.46&71.90&48.77&82.28&66.46&44.83&59.05&52.94&37.38&63.97&55.64\\		
			SFIM~\cite{liu2000smoothing}&56.42&71.97&69.05&46.11&79.23&72.16&43.14&58.54&49.63&33.35&62.09&55.23&60.06&78.78&69.00&59.32&78.57&74.68&47.59&61.17&56.28&43.19&64.20&60.70\\		
			IHS~\cite{tu2001new}&56.05&72.16&67.40&42.84&79.40&74.51&42.49&57.64&49.87&31.62&61.31&54.49&57.40&76.09&74.79&54.30&80.62&58.63&45.53&59.45&55.92&40.68&64.61&52.52\\		
			AWLP~\cite{otazu2005introduction}&56.09&72.00&64.75&48.09&80.56&71.15&43.10&58.33&48.73&34.07&62.48&55.01&56.76&74.10&64.26&52.52&86.36&63.31	&45.47&56.51&54.10&40.32&65.88&55.99\\		
			MTF\_GLP~\cite{aiazzi2006mtf}&56.83&69.80&66.48&44.87&82.57&77.25&43.13&58.22&51.19&33.32&62.75&53.28&58.00&78.26&68.75&57.99&80.75&62.22&46.31&57.98&56.26&42.94&65.75&54.92\\		
			GS~\cite{aiazzi2007improving}&55.86&69.76&66.91&44.80&81.30&72.37&42.79&57.38&50.25&32.35&61.55&55.22&57.39&70.84&71.09&54.24&80.48&67.72&44.93&59.44&52.72&37.91&63.05&56.43\\		
			PRACS~\cite{choi2010new}&57.00&71.37&67.50&45.48&80.59&77.04&43.19&58.38&51.42&33.53&62.33&53.51&56.72&74.57&68.87&51.20&84.70&60.98&45.14&57.18&55.52&39.51&65.41&53.19\\		
			C\_BDSD~\cite{garzelli2014pansharpening}&55.44&70.27&67.59&46.15&81.52&67.11&42.81&57.89&49.32&32.83&62.43&54.38&57.97&79.09&69.17&55.16&81.47&62.96&46.43&57.70&56.83&42.14&65.97&55.92\\		
			MF~\cite{restaino2016fusion}&55.57&72.01&61.72&49.62&81.20&68.87&42.96&58.17&47.97&34.18&62.46&54.95&56.72&77.96&71.29&58.54&79.26&53.30&44.92&57.86&56.24&42.63&64.76&48.04\\		
			BT-H~\cite{lolli2017haze}&54.75&71.11&68.79&44.04&82.18&62.41&42.45&57.84&49.77&32.02&62.46&52.58&57.83&73.06&75.23&46.04&81.82&70.81&44.81&59.84&52.73&36.07&64.73&55.51\\		
			BDSDPC~\cite{vivone2019robust}&55.84&71.44&69.64&46.20&80.82&66.97&43.22&58.54&50.59&33.18&62.34&54.69&58.55&79.76&63.78&59.17&79.86&68.75&46.75&59.39&54.00&43.27&64.30&59.56\\		
			BAGDC~\cite{lu2021unified}&55.46&70.22&62.47&45.94&84.13&69.99&42.77&57.76&49.65&33.35&62.50&53.35&57.38&74.30&65.59&51.36&85.61&67.44&45.87&57.05&54.72&39.68&65.47&58.29\\		
			PNN~\cite{masi2016pansharpening}&55.58&68.04&55.10&48.60&83.11&78.63&41.22&57.33&45.58&35.07&61.85&47.52&57.93&72.29&77.44&61.52&79.92&56.39&45.77&59.16&55.89&44.18&65.07&50.34\\		
			PanNet~\cite{yang2017pannet}&55.14&69.24&64.98&47.85&78.18&70.61&41.84&54.98&46.19&34.58&60.39&54.87&58.07&74.16&64.56&55.16&82.25&72.30&45.52&58.41&51.11&41.08&64.80&57.69\\		
			GPPNN~\cite{xu2021deep}&56.51&69.63&71.06&46.31&79.82&72.24&43.42&57.56&50.67&34.64&60.78&56.86&58.98&71.20&74.11&50.07&83.73&74.80&45.97&59.73&54.21&40.54&64.77&56.55\\		
			U2Net~\cite{peng2023u2net}&55.32&72.79&65.21&47.68&80.86&65.39	&42.75&58.19&49.01&34.53&61.90&52.89&58.22&80.27&63.89&57.25&76.63&71.31&45.12&57.93&52.41&42.01&63.94&54.45\\		
			PanFlowNet~\cite{yang2023panflownet}&54.67&66.05&64.31&42.9&82.62&72.13&41.40&55.85&46.55&32.61&59.87&53.49&52.05&56.49&62.29&51.43&78.68&63.39	&37.69&51.86&39.95&32.69&61.16&40.47\\		
			P2Sharpen~\cite{zhang2023p2sharpen}&57.19&70.45&69.37&48.47&82.49&72.36	&44.49&59.12&52.80&35.49&62.61&56.91&61.32&71.86&78.90&58.68&80.77&77.69&48.47&61.98&58.31&43.81&65.11&61.63\\		
			DISPNet~\cite{wang2024deep}	&57.67&69.49&73.03&46.02&81.59&75.91&44.10&58.31&54.28&34.64&62.78&54.60&60.57&74.87&73.14&58.96&79.66&76.80&47.51&60.90&57.90&43.01&65.09&58.17\\		
			CANConv~\cite{duan2024content}&57.15&72.57&70.33&50.17&81.11&68.71&44.77&59.21&53.56&35.89&63.06&56.90&60.12&75.27&73.36&51.44&84.45&76.20&47.87&60.88&57.13&41.70&66.15&61.34\\		
			PAPS~\cite{jia2024paps}&56.50&71.37&67.10&48.04&80.58&71.90&43.61&57.14&50.23&35.92&61.57&56.80&60.16&73.32&73.25&59.22&83.01&72.15&48.11&61.19&56.57&44.71&65.69&60.47\\		
			UniPAN~\cite{cui2025enpowering}	&57.85&69.71&70.27&48.66&82.85&75.61&44.86&58.95&53.77&35.89&63.15&57.37&61.60&73.90&80.52&54.67&80.70&79.84&48.50&61.80&58.49&43.20&65.74&61.78\\		
			WFANet~\cite{huang2025wavelet}&56.15&71.13&66.41&46.76&79.84&72.78&42.85&56.44&47.96&35.01&62.00&55.71&59.95&72.40&74.27&57.42&82.90&72.72&47.59&60.38&55.71&44.01&65.81&59.64\\		
			CSLP~\cite{chen2025cslp}&58.34&73.33&74.37&51.16&80.79&70.36&45.98&60.71&56.07&37.60&63.00&58.48&61.58&75.01&77.83&59.09&82.87&74.70&49.68&62.87&60.92&45.35&66.32&62.60\\		
			DAISP~\cite{wang2025deep}&55.93&71.56&66.31&50.26&81.19&66.27&43.65&57.94&51.79&34.71&62.43&55.03&58.71&80.37&68.05&53.79&77.81&72.24&46.22&57.72&54.51&40.65&64.73&59.69\\		
			UniPS&\textbf{59.36}&74.27&72.23&53.21&81.33&75.10&\textbf{46.81}&60.43&56.45&39.64&64.18&60.15& 
			\textbf{62.05}&75.60&76.71&61.45&82.55&76.0&\textbf{50.11}&62.58&61.21&46.94&66.80&63.10\\		
			\bottomrule
		\end{tabular}
	} 
\end{table*}

\begin{figure*}[t]
	\centering
	\includegraphics[width=0.99\linewidth]{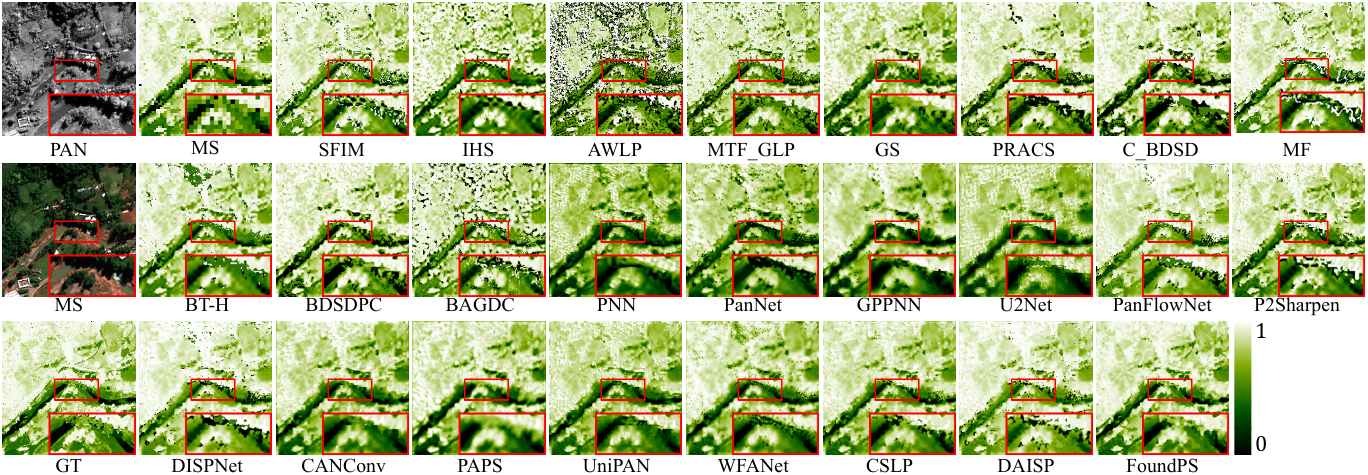}
	\caption{Remote sensing application of NDVI.}
	\label{exp:ndvi}
\end{figure*}

\begin{figure*}[t]
	\centering
	\includegraphics[width=0.99\linewidth]{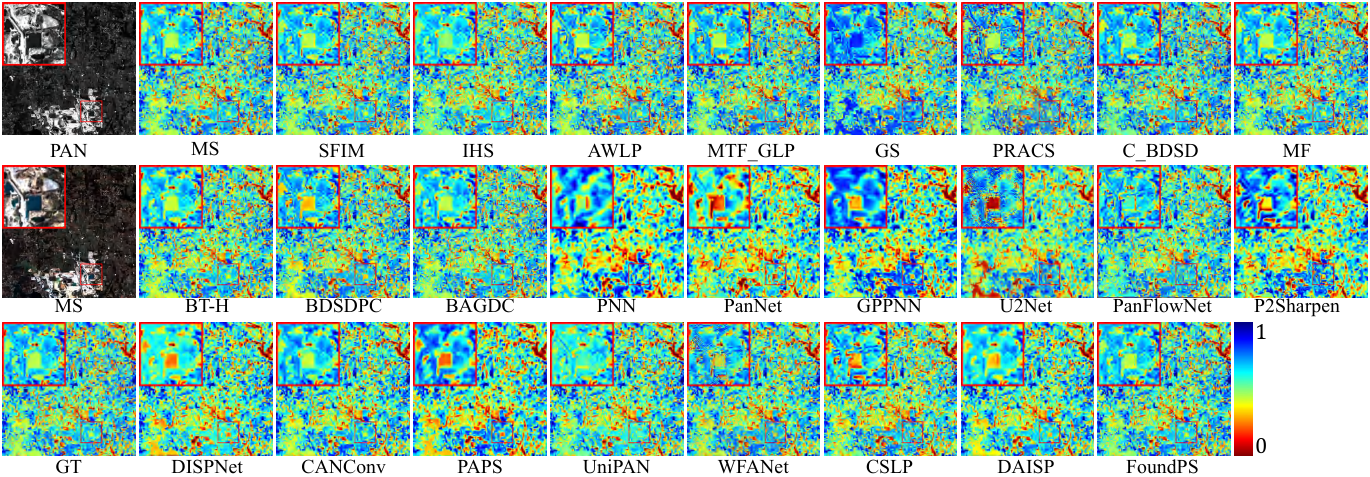}
	\caption{Remote sensing application of NDBI.}
	\label{exp:ndbi}
\end{figure*}

\begin{table*}[t]
	\centering
	\renewcommand\arraystretch{1.05}
	\caption{Quantitative evaluation of remote sensing index over four representative applications. The best results are shown in \textbf{bold}. The up or down arrows indicate higher or lower values correspond to better results.}
	\label{tab:application}
	\resizebox{2.10\columnwidth}{!}{
		\begin{tabular}{c|c@{\,~}c@{\,~}c@{\,~}|c@{\,~}c@{\,~}c@{\,~}|c@{\,~}c@{\,~}c@{\,~}|c@{\,~}c@{\,~}c@{\,~}|c@{\,~}c@{\,~}c@{\,~}}
			\thickhline
			& \multicolumn{3}{c|}{4 Bands (NDVI)} & \multicolumn{3}{c|}{7 Bands (NDWI)} & \multicolumn{3}{c|}{8 Bands (NDRE)} &  \multicolumn{3}{c|}{10 Bands (NDBI)} &  \multicolumn{3}{c}{Average} \\ 
			\multirow{-2}{*}{\textbf{Method}}  & RMSE$\downarrow$ & MAE$\downarrow$ & CC$\uparrow$  & RMSE$\downarrow$ & MAE$\downarrow$ & CC$\uparrow$ & RMSE$\downarrow$	 & MAE$\downarrow$ & CC$\uparrow$ & RMSE$\downarrow$  & MAE$\downarrow$ & CC$\uparrow$ & RMSE$\downarrow$  & MAE$\downarrow$ & CC$\uparrow$ \\
			\thickhline
			SFIM~\cite{liu2000smoothing} & 0.1202 & 0.0719 & 0.7965 & 0.0326 & 0.0204 & 0.9094  & 0.0966 & 0.0536 & 0.6701 & 0.0254 & 0.0175 & 0.9250 & 0.0701 & 0.0422 & 0.8401 \\
			IHS~\cite{tu2001new}& 0.1114 & 0.0748 & 0.8077  & 0.0373 & 0.0252 & 0.8802 	& 0.0846 & 0.0526 & 0.6965  & 0.0257 & 0.0178 & 0.9233 & 0.0662 & 0.0440 & 0.8421 \\
			AWLP~\cite{otazu2005introduction} & 0.2294 & 0.1227 & 0.6739 & 0.0298 & 0.0205 & 0.9136	& 0.2563 & 0.1369 & 0.3355 	& 0.0221 & 0.0150 & 0.9398 & 0.1299 & 0.0713 & 0.7539 \\
			MTF\_GLP~\cite{aiazzi2006mtf}& 0.1353 & 0.0808 & 0.7747 & 0.0317 & 0.0189 & 0.9223 	& 0.1398 & 0.0698 & 0.5704	& 0.0217 & 0.0140 & 0.9434 & 0.0803 & 0.0463 & 0.8262 \\
			GS~\cite{aiazzi2007improving} & 0.1088 & 0.0722 & 0.8077 & 0.0432 & 0.0286 & 0.8597 & 0.0835 & 0.0503 & 0.6994 	& 0.0276 & 0.0193 & 0.9209 & 0.0668 & 0.0439 & 0.8381 \\
			PRACS~\cite{choi2010new}& 0.2377 & 0.1402 & 0.6050 	& 0.0345 & 0.0222 & 0.8788	& 0.1710 & 0.0982 & 0.4921 	& 0.0291 & 0.0179 & 0.9067  & 0.1233 & 0.0729 & 0.7360   \\
			C\_BDSD~\cite{garzelli2014pansharpening} & 0.1536 & 0.0924 & 0.8002  & 0.0274 & 0.0173 & 0.9398	& 0.1849 & 0.0982 & 0.5088	& 0.0230 & 0.0159 & 0.9520 & 0.0929 & 0.0548 & 0.8319	 \\
			MF~\cite{restaino2016fusion} & 0.1097 & 0.0656 & 0.8204  & 0.0286 & 0.0182 & 0.9307  & 0.0850 & 0.0480 & 0.7177  & 0.0198 & 0.0141 & 0.9464 & 0.0622 & 0.0377 & 0.8662	 \\ 
			BT-H~\cite{lolli2017haze}& 0.1416 & 0.0784 & 0.7413 & 0.0264 & 0.0182 & 0.9195 & 0.0763 & 0.0438 & 0.7177 & 0.0201 & 0.0138 & 0.9469 & 0.0717 & 0.0414 & 0.8369 \\
			BDSDPC~\cite{vivone2019robust} 	& 0.1058 & 0.0666 & 0.8528  & 0.0221 & 0.0150 & 0.9480  & 0.1025 & 0.0569 & 0.6974  & 0.0191 & 0.0147 & 0.9511 & 0.0620 & 0.0390 & 0.8791 \\
			BAGDC~\cite{lu2021unified} 	& 0.1306 & 0.0709 & 0.7944 	& 0.0224 & 0.0155 & 0.9402 	& 0.1694 & 0.0944 & 0.4590 	& 0.0231 & 0.0168 & 0.9448 & 0.0818 & 0.0468 & 0.8202	\\ 
			PNN~\cite{masi2016pansharpening} & 0.1409 & 0.1097 & 0.7379 & 0.0635 & 0.0462 & 0.6573 	& 0.1323 & 0.0836 & 0.4035  & 0.0590 & 0.0473 & 0.7776 & 0.0990 & 0.0741 & 0.6873 \\
			PanNet~\cite{yang2017pannet} & 0.1849 & 0.1539 & 0.7634	& 0.0593 & 0.0479 & 0.7796 	& 0.1342 & 0.0893 & 0.4943 	& 0.0594 & 0.0489 & 0.7524  & 0.1139 & 0.0911 & 0.7231 \\
			GPPNN~\cite{xu2021deep} & 0.1441 & 0.1109 & 0.7591 	& 0.0512 & 0.0399 & 0.8160 	& 0.1128 & 0.0850 & 0.5390  & 0.0403 & 0.0306 & 0.8293  & 0.0889 & 0.0681 & 0.7598 \\
			U2Net~\cite{peng2023u2net} 	& 0.1201 & 0.0881 & 0.7842 	& 0.0510 & 0.0349 & 0.7483 	& 0.0801 & 0.0502 & 0.7118 	& 0.0487 & 0.0377 & 0.7363 & 0.0785 & 0.0565 & 0.7515 \\
			PanFlowNet~\cite{yang2023panflownet} & 0.1093 & 0.0742 & 0.8208	& 0.0582 & 0.0405 & 0.7369 	& 0.1347 & 0.0860 & 0.5070	& 0.0438 & 0.0252 & 0.7928 & 0.0825 & 0.0539 & 0.7505 \\
			P2Sharpen~\cite{zhang2023p2sharpen} & 0.1374 & 0.1045 & 0.8311 	& 0.0383 & 0.0286 & 0.8713 	& 0.1283 & 0.0768 & 0.5826 	& 0.0420 & 0.0323 & 0.8302 & 0.0872 & 0.0632 & 0.8014 \\
			DISPNet~\cite{wang2024deep} & 0.1334 & 0.0933 & 0.8263 & 0.0375 & 0.0269 & 0.8560 & 0.0893 & 0.0531 & 0.7276  & 0.0357 & 0.0286 & 0.9357 & 0.0778 & 0.0544 & 0.8529  \\  
			CANConv~\cite{duan2024content}	& 0.1020 & 0.0673 & 0.8213  & 0.0458 & 0.0342 & 0.8330  & 0.0990 & 0.0561 & 0.6925  & 0.0372 & 0.0328 & 0.9463 & 0.0703 & 0.0485 & 0.8453  \\
			PAPS~\cite{jia2024paps} & 0.1284 & 0.0971 & 0.8142 	& 0.0344 & 0.0251 & 0.8938 	& 0.0712 & 0.0460 & 0.7570  & 0.0374 & 0.0289 & 0.8568 & 0.0734 & 0.0544 & 0.8339 \\
			UniPAN~\cite{cui2025enpowering} & 0.0968 & 0.0681 & 0.8591  & 0.0462 & 0.0332 & 0.8077  & 0.1307 & 0.0763 & 0.6004	& 0.0347 & 0.0288 & 0.9324 & 0.0724 & 0.0502 & 0.8358 \\
			WFANet~\cite{huang2025wavelet} 	& 0.1215 & 0.0932 & 0.7952 	& 0.0563 & 0.0455 & 0.7425 	& 0.0809 & 0.0529 & 0.7111 	& 0.0285 & 0.0206 & 0.8978 & 0.0734 & 0.0550 & 0.8070 \\
			CSLP~\cite{chen2025cslp}& 0.0961 & 0.0660 & 0.8671 	& 0.0296 & 0.0204 & 0.9153 	& 0.0770 & 0.0445 & 0.7435 	& 0.0349 & 0.0276 & 0.8973 & 0.0614 & 0.0421 & 0.8674 \\
			DAISP~\cite{wang2025deep}& 0.1099 & 0.0729 & 0.8258  & 0.0342 & 0.0236 & 0.8688 & 0.0785 & 0.0464 & 0.7319 & 0.0302 & 0.0228 & 0.9272  & 0.0657 & 0.0438 & 0.8528 \\
			UniPS &   \textbf{0.0769}   &   \textbf{0.0502}   &   \textbf{0.8950}   &   \textbf{0.0196}   &   \textbf{0.0134}   &   \textbf{0.9510}   &   \textbf{0.0652}   &   \textbf{0.0396}   &   \textbf{0.7835}  &   \textbf{0.0184}   &   \textbf{0.0137}   &   \textbf{0.9542}    & \textbf{0.0458}  & \textbf{0.0302}  & \textbf{0.9079} \\	 
			\thickhline
		\end{tabular}}
\end{table*}

\subsection{Application on Segmentation} 
SegGF dataset~\cite{wang2025deep} is also specifically designed to overcome the limitations of non-reference evaluation metrics by providing meticulously crafted and highly reliable semantic labels. We employ a unified multi-modal segmentation model SegFormer~\cite{xie2021segformer} to rigorously assess the pansharpening performance of each method. For comprehensive comparisons, the mono-channel PAN images, along with resized low-quality MS images, are considered to provide segmentation performance baseline. The qualitative results on reduced scale and full scale are illustrated in Fig.~\ref{exp:GFseg_reduced} and Fig.~\ref{exp:GFseg_full}, respectively. Although PAN images contain more spatial details, their segmentation accuracy is significantly lower than that of lower-quality MS images, which highlights the critical influence of spectral information on segmentation. In comparison to other pansharpening methods, it is evident that the fusion of PAN and MS images can significantly enhance performance and accuracy in high-level tasks. Notably, our method achieves the best segmentation performance, particularly in highlighted regions. This demonstrates that our fusion method achieves the highest consistency between spectral and spatial information, leading to superior overall performance. Furthermore, we quantitatively evaluate the segmentation results using the intersection over Union (IoU) and mean accuracy (Acc.) metrics, as presented in Tab.~\ref{tab:segment}. Hence, our method has the highest mean IoU and Acc., indicating the best overall performance.

\subsection{Remote Sensing Application} 
To evaluate the spectral fidelity of each pansharpening method across diverse band configurations, we systematically evaluate the effectiveness of each method in remote sensing applications by computing four key normalized difference indices using specific band combinations. Specifically, we generate the normalized difference vegetation index (NDVI)~\cite{tucker1979red} in four-band datasets, water index (NDWI)~\cite{mcfeeters1996use} in seven-band datasets, red edge index (NDRE)~\cite{barnes2000coincident} in eight-band datasets and building index (NDBI)~\cite{zha2003use} in ten-band datasets. Their definitions are given below:
\begin{align} 
	\text{NDVI} &= \frac{NIR - R}{NIR + R},   \\
	\text{NDWI} &= \frac{G - NIR}{G + NIR},  \\
	\text{NDRE} &= \frac{NIR - RE}{NIR + RE}, \\
	\text{NDBI} &= \frac{SWIR - NIR}{SWIR + NIR}.
\end{align}
In accordance with these definitions, the normalized difference indices are confined to the range of $[-1, 1]$, which can serve as the spectral accuracy indicators of the pansharpening process to a certain extent. The quantitative results over these four indices are presented in Tab.~\ref{tab:application}. Clearly, our method achieves the best values across all these objective metrics, namely root mean squared error (RMSE), mean absolute error (MAE), and correlation coefficient (CC). Moreover, we present the visualizations in Fig.~\ref{exp:ndvi}, and Fig.~\ref{exp:ndbi}, with more visual comparisons deferred to Appendix~\ref{sec:suppl_D}. Particularly, all of our index maps exhibit the closest resemblance to the reference, while the other methods suffer from intensity distortions.

\section{Conclusion and Discussion}\label{sec:5} 
In this work, we present UniPS, a universal model for pansharpening that establishes a unified, band-agnostic, and scene-robust fusion paradigm. By projecting MS images with arbitrary band configurations into a shared latent space and performing diffusion-based cross-modal fusion, UniPS enables a model to generalize across heterogeneous sensors and diverse scenes. To support large-scale training and fair evaluation, we further construct PSBench, a comprehensive worldwide benchmark comprising multi-satellite MS and PAN image pairs, which fills a critical gap in the development of pansharpening  models. Extensive experiments demonstrate that existing pansharpening methods, particularly lightweight or small-capacity models, often fail to achieve robust performance under heterogeneous spectral settings. This limitation stems not only from insufficient model capacity, but more fundamentally from the lack of architectural designs that can internalize and disentangle diverse spectral modalities within a unified representation. As a result, such models struggle to fully capture cross-band and cross-sensor characteristics, leading to degraded generalization performance on large-scale and multi-modal benchmarks.

While UniPS provides an effective and scalable pipeline for universal pansharpening, several limitations remain. First, due to hardware constraints, the current implementation processes PAN images using patches up to $1024\times1024$ pixels, which limits direct modeling of ultra-large, gigapixel-scale satellite scenes. Second, the bridge posterior sampling strategy requires gradient preservation during inference, introducing additional computational overhead. Finally, the parameter scale of UniPS remains moderate compared with emerging large-scale foundation models. Addressing these limitations, especially scaling to larger contexts and improving inference efficiency, remains an important direction for future work.

\section*{Funding}
This work was supported in part by the Fundamental and Interdisciplinary Disciplines Breakthrough Plan of the Ministry of Education of China (JYB2025XDXM101), the National Natural Science Foundation of China (62276192, 62225113, 624B2109), the Zhongguancun Academy Project (20240308), and the New Generation Artificial Intelligence-National Science and Technology Major Project (2025ZD0123602).

\section*{Conflict of Interest}
The authors declare that they have no conflict of interest.

\section*{Ethics Approval and Consent to Participate}
Not applicable.

\section*{Consent for Publication}
Not applicable.

\section*{Data Availability}
The datasets used in this study are publicly available.

\section*{Materials Availability}
Not applicable.

\section*{Code Availability}
Available at \url{https://github.com/MiliLab/UniPS}.

\section*{Author Contribution}
Hebaixu Wang contributed to conceptualization, methodology, investigation, formal analysis, visualization, validation, funding acquisition, resources, and writing--original draft.
Jing Zhang contributed to methodology, formal analysis, funding acquisition, and writing--original draft.
Haonan Guo contributed to formal analysis, validation, and writing--review and editing.
Di Wang contributed to formal analysis, visualization, and validation.
Jiayi Ma contributed to conceptualization, methodology, funding acquisition, project administration, and writing--review and editing.
Bo Du contributed to methodology, funding acquisition, and writing--review and editing.
Liangpei Zhang contributed to formal analysis, validation, and writing--review and editing.
 
%
 
\bibliography{sn-bibliography}

@inproceedings{cui2025enpowering,
	title={Enpowering Your Pansharpening Models with Generalizability: Unified Distribution is All You Need},
	author={Cui, Yongchuan and Liu, Peng and Zhang, Hui},
	booktitle={Proc. IEEE Int. Conf. Comput. Vision},
	pages={11850--11860},
	year={2025}
}

@inproceedings{kieu2025bidirectional,
  title={Bidirectional diffusion bridge models},
  author={Kieu, Duc and Do, Kien and Nguyen, Toan and Nguyen, Dang and Nguyen, Thin},
  booktitle={Proc. ACM SIGKDD Int. Conf. Knowl. Discov. Data Min.},
  pages={1139--1148},
  year={2025}
}

@article{xu2023vitpose++,
  title={Vitpose++: Vision transformer for generic body pose estimation},
  author={Xu, Yufei and Zhang, Jing and Zhang, Qiming and Tao, Dacheng},
  journal={IEEE Trans. Pattern Anal. Mach. Intell.},
  volume={46},
  number={2},
  pages={1212--1230},
  year={2023},
  publisher={IEEE}
}

@article{cui2025leveraging,
	title={Leveraging Large-Scale Pretrained Spatial-Spectral Priors for General Zero-Shot Pansharpening},
	author={Cui, Yongchuan and Liu, Peng and Zeng, Yi},
	journal={arXiv preprint arXiv:2512.02643},
	year={2025}
}

@article{zhang2025rethinking,
	title={Rethinking Pan-sharpening: A New Training Process for Full-Resolution Generalization},
	author={Zhang, Ran and He, Xuanhua and Xueheng, Li and Cao, Ke and Liu, Liu and Xu, Wenbo and Jiabin, Fang and Qize, Yang and Zhang, Jie},
	journal={arXiv preprint arXiv:2507.15059},
	year={2025}
}

@article{nie2021evomoe,
	title={Evomoe: An evolutional mixture-of-experts training framework via dense-to-sparse gate},
	author={Nie, Xiaonan and Miao, Xupeng and Cao, Shijie and Ma, Lingxiao and Liu, Qibin and Xue, Jilong and Miao, Youshan and Liu, Yi and Yang, Zhi and Cui, Bin},
	journal={arXiv preprint arXiv:2112.14397},
	year={2021}
}

@article{chen2015sirf,
	title={SIRF: Simultaneous satellite image registration and fusion in a unified framework},
	author={Chen, Chen and Li, Yeqing and Liu, Wei and Huang, Junzhou},
	journal={IEEE Trans. Image Process},
	volume={24},
	number={11},
	pages={4213--4224},
	year={2015},
	publisher={IEEE}
}

@article{ballester2006variational,
	title={A variational model for P+ XS image fusion},
	author={Ballester, Coloma and Caselles, Vicent and Igual, Laura and Verdera, Joan and Roug{\'e}, Bernard},
	journal={Int. J. Comput. Vis.},
	volume={69},
	number={1},
	pages={43--58},
	year={2006},
	publisher={Springer}
}

@article{arienzo2022full,
	title={Full-resolution quality assessment of pansharpening: Theoretical and hands-on approaches},
	author={Arienzo, Alberto and Vivone, Gemine and Garzelli, Andrea and Alparone, Luciano and Chanussot, Jocelyn},
	journal={IEEE Geosci. Remote Sens. Mag.},
	volume={10},
	number={3},
	pages={168--201},
	year={2022},
	publisher={IEEE}
}

@article{xiong2020large,
	title={A large-scale remote sensing database for subjective and objective quality assessment of pansharpened images},
	author={Xiong, Yiming and Shao, Feng and Meng, Xiangchao and Jiang, Qiuping and Sun, Weiwei and Fu, Randi and Ho, Yo-Sung},
	journal={J. Vis. Commun. Image Represent.},
	volume={73},
	pages={102947},
	year={2020},
	publisher={Elsevier}
}

@inproceedings{cao2026cross,
	title={Cross-Scale Pansharpening via ScaleFormer and the PanScale Benchmark},
	author={Cao, Ke and He, Xuanhua and Li, Xueheng and Zhu, Lingting and Wang, Yingying and Ma, Ao and Zhang, Zhanjie and Zhou, Man and Xie, Chengjun and Zhang, Jie},
	booktitle={Proc. IEEE Conf. Comput. Vis. Pattern Recog.},
	pages={13211--13221},
	year={2026}
}

@article{meng2020large,
	title={A large-scale benchmark data set for evaluating pansharpening performance: Overview and implementation},
	author={Meng, Xiangchao and Xiong, Yiming and Shao, Feng and Shen, Huanfeng and Sun, Weiwei and Yang, Gang and Yuan, Qiangqiang and Fu, Randi and Zhang, Hongyan},
	journal={IEEE Geosci. Remote Sens. Mag.},
	volume={9},
	number={1},
	pages={18--52},
	year={2020},
	publisher={IEEE}
}

@inproceedings{pradeep2013implementation,
	title={Implementation of image fusion algorithm using MATLAB (LAPLACIAN PYRAMID)},
	author={Pradeep, M},
	booktitle={Proc. Int. Multi-Conf. Autom. Comput. Commun. Control Compress. Sens.},
	pages={165--168},
	year={2013},
	organization={IEEE}
}

@article{li2025uni,
	title={Uni-moe: Scaling unified multimodal llms with mixture of experts},
	author={Li, Yunxin and Jiang, Shenyuan and Hu, Baotian and Wang, Longyue and Zhong, Wanqi and Luo, Wenhan and Ma, Lin and Zhang, Min},
	journal={IEEE Trans. Pattern Anal. Mach. Intell.},
	year={2025},
	publisher={IEEE}
}

@inproceedings{wang2025hmoe,
	title={Hmoe: Heterogeneous mixture of experts for language modeling},
	author={Wang, An and Sun, Xingwu and Xie, Ruobing and Li, Shuaipeng and Zhu, Jiaqi and Yang, Zhen and Zhao, Pinxue and Han, Weidong and Kang, Zhanhui and Wang, Di and others},
	booktitle={Proc. Conf. Empir. Methods Nat. Lang. Process},
	pages={21954--21968},
	year={2025}
}

@article{ren2023pangu,
	title={Pangu-$\{$$\backslash$Sigma$\}$: Towards trillion parameter language model with sparse heterogeneous computing},
	author={Ren, Xiaozhe and Zhou, Pingyi and Meng, Xinfan and Huang, Xinjing and Wang, Yadao and Wang, Weichao and Li, Pengfei and Zhang, Xiaoda and Podolskiy, Alexander and Arshinov, Grigory and others},
	journal={arXiv preprint arXiv:2303.10845},
	year={2023}
}

@inproceedings{lenz2025jamba,
	title={Jamba: Hybrid transformer-mamba language models},
	author={Lenz, Barak and Lieber, Opher and Arazi, Alan and Bergman, Amir and Manevich, Avshalom and Peleg, Barak and Aviram, Ben and Almagor, Chen and Fridman, Clara and Padnos, Dan and others},
	booktitle={Int. Conf. Learn. Represent.},
	year={2025}
}

@inproceedings{lewis2021base,
	title={Base layers: Simplifying training of large, sparse models},
	author={Lewis, Mike and Bhosale, Shruti and Dettmers, Tim and Goyal, Naman and Zettlemoyer, Luke},
	booktitle={Proc. Int. Conf. Mach. Learn.},
	pages={6265--6274},
	year={2021},
	organization={PMLR}
}

@article{jiang2024mixtral,
	title={Mixtral of experts},
	author={Jiang, Albert Q and Sablayrolles, Alexandre and Roux, Antoine and Mensch, Arthur and Savary, Blanche and Bamford, Chris and Chaplot, Devendra Singh and Casas, Diego de las and Hanna, Emma Bou and Bressand, Florian and others},
	journal={arXiv preprint arXiv:2401.04088},
	year={2024}
}

@article{wei2024skywork,
	title={Skywork-moe: A deep dive into training techniques for mixture-of-experts language models},
	author={Wei, Tianwen and Zhu, Bo and Zhao, Liang and Cheng, Cheng and Li, Biye and L{\"u}, Weiwei and Cheng, Peng and Zhang, Jianhao and Zhang, Xiaoyu and Zeng, Liang and others},
	journal={arXiv preprint arXiv:2406.06563},
	year={2024}
}

@inproceedings{rajbhandari2022deepspeed,
	title={Deepspeed-moe: Advancing mixture-of-experts inference and training to power next-generation ai scale},
	author={Rajbhandari, Samyam and Li, Conglong and Yao, Zhewei and Zhang, Minjia and Aminabadi, Reza Yazdani and Awan, Ammar Ahmad and Rasley, Jeff and He, Yuxiong},
	booktitle={Proc. Int. Conf. Mach. Learn.},
	pages={18332--18346},
	year={2022},
	organization={PMLR}
}

@article{pan2024dense,
	title={Dense training, sparse inference: Rethinking training of mixture-of-experts language models},
	author={Pan, Bowen and Shen, Yikang and Liu, Haokun and Mishra, Mayank and Zhang, Gaoyuan and Oliva, Aude and Raffel, Colin and Panda, Rameswar},
	journal={arXiv preprint arXiv:2404.05567},
	year={2024}
}

@inproceedings{wumixture,
	title={Mixture of LoRA Experts},
	author={Wu, Xun and Huang, Shaohan and Wei, Furu},
	booktitle={Int. Conf. Learn. Represent.}
}

@article{cai2025survey,
	title={A survey on mixture of experts in large language models},
	author={Cai, Weilin and Jiang, Juyong and Wang, Fan and Tang, Jing and Kim, Sunghun and Huang, Jiayi},
	journal={IEEE Trans. Knowl. Data Eng.},
	year={2025},
	publisher={IEEE}
}

@inproceedings{peng2023u2net,
	title={U2net: A general framework with spatial-spectral-integrated double u-net for image fusion},
	author={Peng, Siran and Guo, Chenhao and Wu, Xiao and Deng, Liang-Jian},
	booktitle={Proc. ACM Int. Conf. Multimed.},
	pages={3219--3227},
	year={2023}
}

@article{roy2025classical,
	title={From Classical Image Fusion to Deep Representation Learning: A Survey of the Advances in Hyperspectral Image Pan-sharpening},
	author={Roy, Swalpa Kumar and Vivone, Gemine and Das, Swagatam and Raman, Balasubramanian and Singh, Pravendra and others},
	journal={Inf. Fusion},
	pages={103834},
	year={2025},
	publisher={Elsevier}
}

@article{xu2023upangan,
	title={UPanGAN: Unsupervised pansharpening based on the spectral and spatial loss constrained generative adversarial network},
	author={Xu, Qizhi and Li, Yuan and Nie, Jinyan and Liu, Qingjie and Guo, Mengyao},
	journal={Inf. Fusion},
	volume={91},
	pages={31--46},
	year={2023},
	publisher={Elsevier}
}

@article{shang2024mft,
	title={MFT-GAN: A multiscale feature-guided transformer network for unsupervised hyperspectral pansharpening},
	author={Shang, Yanli and Liu, Jianjun and Zhang, Jingyi and Wu, Zebin},
	journal={IEEE Trans. Geosci. Remote. Sens.},
	volume={62},
	pages={1--16},
	year={2024},
	publisher={IEEE}
}

@article{jia2024paps,
	title={Paps: Progressive attention-based pan-sharpening},
	author={Jia, Yanan and Hu, Qiming and Dian, Renwei and Ma, Jiayi and Guo, Xiaojie},
	journal={IEEE CAA J. Autom. Sin.},
	volume={11},
	number={2},
	pages={391--404},
	year={2024},
	publisher={IEEE}
}

@inproceedings{huang2025wavelet,
	title={Wavelet-Assisted Multi-Frequency Attention Network for Pansharpening},
	author={Huang, Jie and Huang, Rui and Xu, Jinghao and Peng, Siran and Duan, Yule and Deng, Liang-Jian},
	booktitle={Proc. AAAI Conf. Artif. Intell.},
	volume={39},
	number={4},
	pages={3662--3670},
	year={2025}
}

@article{wu2025semantic,
  title={A semantic-enhanced multi-modal remote sensing foundation model for Earth observation},
  author={Wu, Kang and Zhang, Yingying and Ru, Lixiang and Dang, Bo and Lao, Jiangwei and Yu, Lei and Luo, Junwei and Zhu, Zifan and Sun, Yue and Zhang, Jiahao and others},
  journal={Nat. Mach. Intell.},
  volume={7},
  number={8},
  pages={1235--1249},
  year={2025}, 
}

@article{gao2017bounds,
	title={Bounds on the jensen gap, and implications for mean-concentrated distributions},
	author={Gao, Xiang and Sitharam, Meera and Roitberg, Adrian E},
	journal={arXiv preprint arXiv:1712.05267},
	year={2017}
}

@article{ciotola2024hyperspectral,
  title={Hyperspectral pansharpening: Critical review, tools, and future perspectives},
  author={Ciotola, Matteo and Guarino, Giuseppe and Vivone, Gemine and Poggi, Giovanni and Chanussot, Jocelyn and Plaza, Antonio and Scarpa, Giuseppe},
  journal={IEEE Geosci. Remote Sens. Mag.},
  year={2024},
  publisher={IEEE}
}

@inproceedings{xu2025hipandas,
  title={Hipandas: Hyperspectral image joint denoising and super-resolution by image fusion with the panchromatic image},
  author={Xu, Shuang and Zhao, Zixiang and Bai, Haowen and Yu, Chang and Peng, Jiangjun and Cao, Xiangyong and Meng, Deyu},
  booktitle={Proc. IEEE Int. Conf. Comput. Vis.},
  pages={12002--12011},
  year={2025}
}

@article{zhou2025toward,
	title={Toward Resolution Mismatching: Modality-Aware Feature-Aligned Network for Pan-Sharpening},
	author={Zhou, Man and He, Xuanhua and Hong, Danfeng},
	journal={IEEE Trans. Pattern Anal. Mach. Intell.},
	year={2025},
	publisher={IEEE}
}

@inproceedings{barnes2000coincident,
	title={Coincident detection of crop water stress, nitrogen status and canopy density using ground based multispectral data},
	author={Barnes, EM and Clarke, TR and Richards, SE and Colaizzi, PD and Haberland, JULIO and Kostrzewski, M and Waller, P and Choi, C and Riley, E and Thompson, T and others},
	booktitle={Proc. Int. Conf. on Precision Agriculture},
	volume={1619},
	number={6},
	year={2000}
}

@article{zha2003use,
	title={Use of normalized difference built-up index in automatically mapping urban areas from TM imagery},
	author={Zha, Yong and Gao, Jay and Ni, Shaoxiang},
	journal={Int. J. Remote Sens.},
	volume={24},
	number={3},
	pages={583--594},
	year={2003},
	publisher={Taylor \& Francis}
}

@article{mcfeeters1996use,
	title={The use of the Normalized Difference Water Index (NDWI) in the delineation of open water features},
	author={McFeeters, Stuart K},
	journal={Int. J. Remote Sens.},
	volume={17},
	number={7},
	pages={1425--1432},
	year={1996},
	publisher={Taylor \& Francis}
}

@article{liu2000smoothing,
	title={Smoothing filter-based intensity modulation: A spectral preserve image fusion technique for improving spatial details},
	author={Liu, JG},
	journal={Int. J. Remote Sens.},
	volume={21},
	number={18},
	pages={3461--3472},
	year={2000},
	publisher={Taylor \& Francis}
}

@article{aiazzi2007improving,
	title={Improving component substitution pansharpening through multivariate regression of MS $+ $ Pan data},
	author={Aiazzi, Bruno and Baronti, Stefano and Selva, Massimo},
	journal={IEEE Trans. Geosci. Remote. Sens.},
	volume={45},
	number={10},
	pages={3230--3239},
	year={2007},
	publisher={IEEE}
}

@article{tu2001new,
	title={A new look at IHS-like image fusion methods},
	author={Tu, Te-Ming and Su, Shun-Chi and Shyu, Hsuen-Chyun and Huang, Ping S},
	journal={Inf. Fusion},
	volume={2},
	number={3},
	pages={177--186},
	year={2001},
	publisher={Elsevier}
}

@article{lu2021unified,
	title={A unified pansharpening model based on band-adaptive gradient and detail correction},
	author={Lu, Hangyuan and Yang, Yong and Huang, Shuying and Tu, Wei and Wan, Weiguo},
	journal={IEEE Trans. Image Process},
	volume={31},
	pages={918--933},
	year={2021},
	publisher={IEEE}
}

@article{chen2025cslp,
	title={CSLP: A novel pansharpening method based on compressed sensing and L-PNN},
	author={Chen, Yingxia and Wan, Zhenglong and Chen, Zeqiang and Wei, Mingming},
	journal={Inf. Fusion},
	volume={118},
	pages={103002},
	year={2025},
	publisher={Elsevier}
}

@inproceedings{yang2023panflownet,
	title={Panflownet: A flow-based deep network for pan-sharpening},
	author={Yang, Gang and Cao, Xiangyong and Xiao, Wenzhe and Zhou, Man and Liu, Aiping and Chen, Xun and Meng, Deyu},
	booktitle={Proc. IEEE Conf. Comput. Vis. Pattern Recogn.},
	pages={16857--16867},
	year={2023}
}

@inproceedings{yang2017pannet,
	title={PanNet: A deep network architecture for pan-sharpening},
	author={Yang, Junfeng and Fu, Xueyang and Hu, Yuwen and Huang, Yue and Ding, Xinghao and Paisley, John},
	booktitle={Proc. IEEE Int. Conf. Comput. Vision},
	pages={5449--5457},
	year={2017}
}

@inproceedings{xu2021deep,
	title={Deep gradient projection networks for pan-sharpening},
	author={Xu, Shuang and Zhang, Jiangshe and Zhao, Zixiang and Sun, Kai and Liu, Junmin and Zhang, Chunxia},
	booktitle={Proc. IEEE Conf. Comput. Vis. Pattern Recogn.},
	pages={1366--1375},
	year={2021}
}

@inproceedings{wang2025deep,
	title={Deep Adaptive Unfolded Network via Spatial Morphology Stripping and Spectral Filtration for Pan-sharpening},
	author={Wang, Hebaixu and Ma, Jiayi},
	booktitle={Proc. IEEE Int. Conf. Comput. Vision},
	pages={10730--10740},
	year={2025}
}

@inproceedings{li2025enhanced,
	title={Enhanced Pansharpening via Quaternion Spatial-Spectral Interactions},
	author={Li, Dong and Luo, Chunhui and Bao, Yuanfei and Yang, Gang and Xiao, Jie and Fu, Xueyang and Zha, Zheng-Jun},
	booktitle={Proc. IEEE Conf. Comput. Vis. Pattern Recogn.},
	pages={10908--10918},
	year={2025}
}

@article{guo2025better,
	title={Better Image Filter for Pansharpening},
	author={Guo, Anjing and Dian, Renwei and Wang, Nan and Li, Shutao},
	journal={IEEE Trans. Image Process},
	volume={34},
	pages={8171--8184},
	year={2025},
	publisher={IEEE}
}

@article{vivone2014critical,
	title={A critical comparison among pansharpening algorithms},
	author={Vivone, Gemine and Alparone, Luciano and Chanussot, Jocelyn and Dalla Mura, Mauro and Garzelli, Andrea and Licciardi, Giorgio A and Restaino, Rocco and Wald, Lucien},
	journal={IEEE Trans. Geosci. Remote. Sens.},
	volume={53},
	number={5},
	pages={2565--2586},
	year={2014},
	publisher={IEEE}
}

@article{fang2013variational,
	title={A variational approach for pan-sharpening},
	author={Fang, Faming and Li, Fang and Shen, Chaomin and Zhang, Guixu},
	journal={IEEE Trans. Image Process},
	volume={22},
	number={7},
	pages={2822--2834},
	year={2013},
	publisher={IEEE}
}

@article{ma2020pan,
	title={Pan-GAN: An unsupervised pan-sharpening method for Remote Sens. image fusion},
	author={Ma, Jiayi and Yu, Wei and Chen, Chen and Liang, Pengwei and Guo, Xiaojie and Jiang, Junjun},
	journal={Inf. Fusion},
	volume={62},
	pages={110--120},
	year={2020},
	publisher={Elsevier}
}

@article{zhong2024ssdiff,
	title={Ssdiff: Spatial-spectral integrated diffusion model for Remote Sens. pansharpening},
	author={Zhong, Yu and Wu, Xiao and Cao, Zihan and Dou, Hong-Xia and Deng, Liang-Jian},
	journal={Adv. Neural Inform. Process. Syst.},
	volume={37},
	pages={77962--77986},
	year={2024}
}

@article{wang2025forgotten,
	title={From Forgotten to Pan-sharpening},
	author={Wang, Jiaming and Lin, Yansong and Chen, Chuanxi and Huang, Xiao and Zhang, Ruiqian and Wang, Yu and Lu, Tao},
	journal={Pattern Recog.},
	pages={112653},
	year={2025},
	publisher={Elsevier}
}

@article{wang2024zero,
	title={Zero-Sharpen: A universal pansharpening method across satellites for reducing scale-variance gap via zero-shot variation},
	author={Wang, Hebaixu and Zhang, Hao and Tian, Xin and Ma, Jiayi},
	journal={Inf. Fusion},
	volume={101},
	pages={102003},
	year={2024},
	publisher={Elsevier}
}

@inproceedings{xiao2025hyperspectral,
	title={Hyperspectral Pansharpening via Diffusion Models with Iteratively Zero-Shot Guidance},
	author={Xiao, Jin-Liang and Huang, Ting-Zhu and Deng, Liang-Jian and Lin, Guang and Cao, Zihan and Li, Chao and Zhao, Qibin},
	booktitle={Proc. IEEE Conf. Comput. Vis. Pattern Recogn.},
	pages={12669--12678},
	year={2025}
}

@article{cao2024zero,
	title={Zero-shot semi-supervised learning for pansharpening},
	author={Cao, Qi and Deng, Liang-Jian and Wang, Wu and Hou, Junming and Vivone, Gemine},
	journal={Inf. Fusion},
	volume={101},
	pages={102001},
	year={2024},
	publisher={Elsevier}
}

@inproceedings{hou2024linearly,
	title={Linearly-evolved transformer for pan-sharpening},
	author={Hou, Junming and Cao, Zihan and Zheng, Naishan and Li, Xuan and Chen, Xiaoyu and Liu, Xinyang and Cong, Xiaofeng and Hong, Danfeng and Zhou, Man},
	booktitle={Proc. ACM Int. Conf. Multimed.},
	pages={1486--1494},
	year={2024}
}

@article{li2024real,
	title={Real HSI-MSI-PAN image dataset for the hyperspectral/multi-spectral/panchromatic image fusion and super-resolution fields},
	author={Li, Shuangliang},
	journal={arXiv preprint arXiv:2407.02387},
	year={2024}
}

@article{ma2025deep,
	title={Deep spatial--spectral fusion transformer for Remote Sens. pansharpening},
	author={Ma, Mengting and Jiang, Yizhen and Zhao, Mengjiao and Ma, Xiaowen and Zhang, Wei and Song, Siyang},
	journal={Inf. Fusion},
	volume={118},
	pages={102980},
	year={2025},
	publisher={Elsevier}
}

@article{he2025pan,
	title={Pan-mamba: Effective pan-sharpening with state space model},
	author={He, Xuanhua and Cao, Ke and Zhang, Jie and Yan, Keyu and Wang, Yingying and Li, Rui and Xie, Chengjun and Hong, Danfeng and Zhou, Man},
	journal={Inf. Fusion},
	volume={115},
	pages={102779},
	year={2025},
	publisher={Elsevier}
}

@article{hou2025mambamtl,
	title={MambaMTL: Progressive Mutual-Guided Mamba Multi-Task Learning for Hyperspectral Image Pansharpening and Classification},
	author={Hou, Shaoxiong and Xiao, Song and Qu, Jiahui and Dong, Wenqian},
	journal={IEEE Trans. Geosci. Remote. Sens.},
	year={2025},
	publisher={IEEE}
}

@article{zhang2023p2sharpen,
	title={P2Sharpen: A progressive pansharpening network with deep spectral transformation},
	author={Zhang, Hao and Wang, Hebaixu and Tian, Xin and Ma, Jiayi},
	journal={Inf. Fusion},
	volume={91},
	pages={103--122},
	year={2023},
	publisher={Elsevier}
}

@article{otazu2005introduction,
	title={Introduction of sensor spectral response into image fusion methods. Application to wavelet-based methods},
	author={Otazu, Xavier and Gonz{\'a}lez-Aud{\'\i}cana, Mar{\'\i}a and Fors, Octavi and N{\'u}{\~n}ez, Jorge},
	journal={IEEE Trans. Geosci. Remote. Sens.},
	volume={43},
	number={10},
	pages={2376--2385},
	year={2005},
	publisher={IEEE}
}

@article{lolli2017haze,
	title={Haze correction for contrast-based multispectral pansharpening},
	author={Lolli, Simone and Alparone, Luciano and Garzelli, Andrea and Vivone, Gemine},
	journal={IEEE Geosci. Remote Sens. Lett.},
	volume={14},
	number={12},
	pages={2255--2259},
	year={2017},
	publisher={IEEE}
}

@inproceedings{wang2024deep,
	title={Deep Unfolded Network with Intrinsic Supervision for Pan-Sharpening},
	author={Wang, Hebaixu and Gong, Meiqi and Mei, Xiaoguang and Zhang, Hao and Ma, Jiayi},
	booktitle={Proc. AAAI Conf. Artif. Intell.},
	volume={38},
	number={6},
	pages={5419--5426},
	year={2024}
}

@inproceedings{duan2024content,
	title={Content-adaptive non-local convolution for Remote Sens. pansharpening},
	author={Duan, Yule and Wu, Xiao and Deng, Haoyu and Deng, Liang-Jian},
	booktitle={Proc. IEEE Conf. Comput. Vis. Pattern Recogn.},
	pages={27738--27747},
	year={2024}
}

@article{aiazzi2006mtf,
	title={MTF-tailored multiscale fusion of high-resolution MS and Pan imagery},
	author={Aiazzi, Bruno and Alparone, Luciano and Baronti, Stefano and Garzelli, Andrea and Selva, Massimo},
	journal={Photogramm. Eng. Remote Sens.},
	volume={72},
	number={5},
	pages={591--596},
	year={2006},
	publisher={American Society for Photogrammetry and Remote Sens.}
}

@article{vivone2019robust,
	title={Robust band-dependent spatial-detail approaches for panchromatic sharpening},
	author={Vivone, Gemine},
	journal={IEEE Trans. Geosci. Remote. Sens.},
	volume={57},
	number={9},
	pages={6421--6433},
	year={2019},
	publisher={IEEE}
}

@article{garzelli2014pansharpening,
	title={Pansharpening of multispectral images based on nonlocal parameter optimization},
	author={Garzelli, Andrea},
	journal={IEEE Trans. Geosci. Remote Sens.},
	volume={53},
	number={4},
	pages={2096--2107},
	year={2014},
	publisher={IEEE}
}

@article{wang2023hyperspectral,
	title={Hyperspectral Image Super-Resolution Meets Deep Learning: A Survey and Perspective},
	author={Wang, Xinya and Hu, Qian and Cheng, Yingsong and Ma, Jiayi},
	journal={Proc. Int. Joint Conf. Artif. Intell.},
	volume={10},
	number={8},
	pages={1664--1687},
	year={2023}
}

@article{choi2010new,
	title={A new adaptive component-substitution-based satellite image fusion by using partial replacement},
	author={Choi, Jaewan and Yu, Kiyun and Kim, Yongil},
		journal={IEEE Trans. Geosci. Remote. Sens.},
	volume={49},
	number={1},
	pages={295--309},
	year={2010},
	publisher={IEEE}
}

@article{restaino2016fusion,
	title={Fusion of multispectral and panchromatic images based on morphological operators},
	author={Restaino, Rocco and Vivone, Gemine and Dalla Mura, Mauro and Chanussot, Jocelyn},
	journal={IEEE Trans. Image Process},
	volume={25},
	number={6},
	pages={2882--2895},
	year={2016},
	publisher={IEEE}
}

@article{zhang2022progress,
	title={Progress and Challenges in Intelligent Remote Sens. Satellite Systems},
	author={Zhang, Bing and Wu, Yuanfeng and Zhao, Boya and Chanussot, Jocelyn and Hong, Danfeng and Yao, Jing and Gao, Lianru},
	journal={IEEE J. Sel. Top. Appl. Earth Obs. Remote Sens.},
	volume={15},
	pages={1814--1822},
	year={2022},
	publisher={IEEE}
}

@article{alparone2016spatial,
	title={Spatial Methods for Multispectral Pansharpening: Multiresolution Analysis Demystified},
	author={Alparone, Luciano and Baronti, Stefano and Aiazzi, Bruno and Garzelli, Andrea},
	journal={IEEE Trans. Geosci. Remote. Sens.},
	volume={54},
	number={5},
	pages={2563--2576},
	year={2016},
	publisher={IEEE}
}

@inproceedings{fu2019variational,
	title={A variational pan-sharpening with local gradient constraints},
	author={Fu, Xueyang and Lin, Zihuang and Huang, Yue and Ding, Xinghao},
	booktitle={Proc. IEEE Conf. Comput. Vis. Pattern Recogn.},
	pages={10265--10274},
	year={2019}
}

@article{ruder2016overview,
	title={An Overview of Gradient Descent Optimization Algorithms},
	author={Ruder, Sebastian},
	journal={arXiv preprint arXiv:1609.04747},
	year={2016}
}

@article{masi2016pansharpening,
	title={Pansharpening by convolutional neural networks},
	author={Masi, Giuseppe and Cozzolino, Davide and Verdoliva, Luisa and Scarpa, Giuseppe},
	journal={Remote Sens.},
	volume={8},
	number={7},
	pages={594},
	year={2016},
	publisher={MDPI}
}

@article{huynh2008scope,
	title={Scope of validity of PSNR in image/video quality assessment},
	author={Huynh-Thu, Quan and Ghanbari, Mohammed},
	journal={Electron. Lett.},
	volume={44},
	number={13},
	pages={800--801},
	year={2008},
	publisher={IET}
}

@inproceedings{ma2024rewrite,
	title={Rewrite the stars},
	author={Ma, Xu and Dai, Xiyang and Bai, Yue and Wang, Yizhou and Fu, Yun},
	booktitle={Proc. IEEE Conf. Comput. Vis. Pattern Recogn.},
	pages={5694--5703},
	year={2024}
}

@article{efron2011tweedie,
	title={Tweedie’s formula and selection bias},
	author={Efron, Bradley},
	journal={J. Am. Stat. Assoc.},
	volume={106},
	number={496},
	pages={1602--1614},
	year={2011},
	publisher={Taylor \& Francis}
}

@inproceedings{chung2023diffusion,
	title={Diffusion posterior sampling for general noisy inverse problems},
	author={Chung, Hyungjin and Kim, Jeongsol and McCann, Michael T and Klasky, Marc L and Ye, Jong Chul},
	booktitle={Int. Conf. Learn. Represent.},
	year={2023}
}

@article{li2025back,
	title={Back to basics: Let denoising generative models denoise},
	author={Li, Tianhong and He, Kaiming},
	journal={arXiv preprint arXiv:2511.13720},
	year={2025}
}

@article{wald1997fusion,
	title={Fusion of Satellite Images of Different Spatial Resolutions: Assessing the Quality of Resulting Images},
	author={Wald, Lucien and Ranchin, Thierry and Mangolini, Marc},
	journal={Photogramm. Eng. Remote Sens.},
	volume={63},
	number={6},
	pages={691--699},
	year={1997}
}

@article{han2022survey,
	title={A survey on vision transformer},
	author={Han, Kai and Wang, Yunhe and Chen, Hanting and Chen, Xinghao and Guo, Jianyuan and Liu, Zhenhua and Tang, Yehui and Xiao, An and Xu, Chunjing and Xu, Yixing and others},
	journal={IEEE Trans. Pattern Anal. Mach. Intell.},
	volume={45},
	number={1},
	pages={87--110},
	year={2022},
	publisher={IEEE}
}

@inproceedings{wang2025residual,
	title={Residual diffusion bridge model for image restoration},
	author={Wang, Hebaixu and Zhang, Jing and Chen, Haoyang and Guo, Haonan and Wang, Di and Ma, Jiayi and Du, Bo},
	booktitle={Proc. IEEE Conf. Comput. Vis. Pattern Recog.},
	pages={8375--8386},
	year={2026}
}

@article{he2024pan,
	title={Pan-mamba: Effective pan-sharpening with state space model},
	author={He, Xuanhua and Cao, Ke and Zhang, Jie and Yan, Keyu and Wang, Yingying and Li, Rui and Xie, Chengjun and Hong, Danfeng and Zhou, Man},
	journal={Inf. Fusion},
	pages={102779},
	year={2024},
	publisher={Elsevier}
}

@article{hu2020pan,
	title={Pan-sharpening via multiscale dynamic convolutional neural network},
	author={Hu, Jianwen and Hu, Pei and Kang, Xudong and Zhang, Hui and Fan, Shaosheng},
	journal={IEEE Trans. Geosci. Remote. Sens.},
	volume={59},
	number={3},
	pages={2231--2244},
	year={2020},
	publisher={IEEE}
}

@inproceedings{zhou2022panformer,
	title={PanFormer: A transformer based model for pan-sharpening},
	author={Zhou, Huanyu and Liu, Qingjie and Wang, Yunhong},
	booktitle={Int. Conf. Multimedia and Expo.},
	pages={1--6},
	year={2022},
	organization={IEEE}
}

@article{cao2024diffusion,
	title={Diffusion model with disentangled modulations for sharpening multispectral and hyperspectral images},
	author={Cao, Zihan and Cao, Shiqi and Deng, Liang-Jian and Wu, Xiao and Hou, Junming and Vivone, Gemine},
	journal={Inf. Fusion},
	volume={104},
	pages={102158},
	year={2024},
	publisher={Elsevier}
}

@article{zhang2023panchromatic,
	title={Panchromatic and multispectral image fusion for Remote Sens. and earth observation: Concepts, taxonomy, literature review, evaluation methodologies and challenges ahead},
	author={Zhang, Kai and Zhang, Feng and Wan, Wenbo and Yu, Hui and Sun, Jiande and Del Ser, Javier and Elyan, Eyad and Hussain, Amir},
	journal={Inf. Fusion},
	volume={93},
	pages={227--242},
	year={2023},
	publisher={Elsevier}
}

@article{zhou2024general,
	title={A general spatial-frequency learning framework for multimodal image fusion},
	author={Zhou, Man and Huang, Jie and Yan, Keyu and Hong, Danfeng and Jia, Xiuping and Chanussot, Jocelyn and Li, Chongyi},
	journal={IEEE Trans. Pattern Anal. Mach. Intell.},
	year={2024},
	publisher={IEEE}
}

@book{wald2002data,
	title={Data fusion: definitions and architectures: fusion of images of different spatial resolutions},
	author={Wald, Lucien},
	year={2002},
	publisher={Presses des MINES}
}

@article{alparone2007comparison,
	title={Comparison of pansharpening algorithms: Outcome of the 2006 GRS-S data-fusion contest},
	author={Alparone, Luciano and Wald, Lucien and Chanussot, Jocelyn and Thomas, Claire and Gamba, Paolo and Bruce, Lori Mann},
	journal={IEEE Trans. Geosci. Remote. Sens.},
	volume={45},
	number={10},
	pages={3012--3021},
	year={2007},
	publisher={IEEE}
}

@article{wang2004image,
	title={Image quality assessment: from error visibility to structural similarity},
	author={Wang, Zhou and Bovik, Alan C and Sheikh, Hamid R and Simoncelli, Eero P},
	journal={IEEE Trans. Image Process},
	volume={13},
	number={4},
	pages={600--612},
	year={2004},
	publisher={IEEE}
}

@article{alparone2008multispectral,
	title={Multispectral and panchromatic data fusion assessment without reference},
	author={Alparone, Luciano and Aiazzi, Bruno and Baronti, Stefano and Garzelli, Andrea and Nencini, Filippo and Selva, Massimo},
	journal={Photogramm. Eng. Remote Sens.},
	volume={74},
	number={2},
	pages={193--200},
	year={2008},
	publisher={American Society for Photogrammetry and Remote Sens.}
}

@article{xie2021segformer,
	title={SegFormer: Simple and efficient design for semantic segmentation with transformers},
	author={Xie, Enze and Wang, Wenhai and Yu, Zhiding and Anandkumar, Anima and Alvarez, Jose M and Luo, Ping},
	journal={Adv. Neural Inform. Process. Syst.},
	volume={34},
	pages={12077--12090},
	year={2021}
}

@article{tucker1979red,
	title={Red and photographic infrared linear combinations for monitoring vegetation},
	author={Tucker, Compton J},
	journal={Remote Sens. Environ.},
	volume={8},
	number={2},
	pages={127--150},
	year={1979},
	publisher={Elsevier}
}

@article{xu2024infinite,
	title={Infinite-dimensional feature interaction},
	author={Xu, Chenhui and Yu, Fuxun and Li, Maoliang and Zheng, Zihao and Xu, Zirui and Xiong, Jinjun and Chen, Xiang},
	journal={Adv. Neural Inform. Process. Syst.},
	volume={37},
	pages={92381--92400},
	year={2024}
}

@article{zhang2023stp,
	title={STP-SOM: Scale-transfer learning for pansharpening via estimating spectral observation model},
	author={Zhang, Hao and Ma, Jiayi},
	journal={Int. J. Comput. Vision},
	volume={131},
	number={12},
	pages={3226--3251},
	year={2023},
	publisher={Springer}
}

@article{pereira2026multi,
	title={Multi-Head Attention Residual Unfolded Network for Model-Based Pansharpening},
	author={Pereira-S{\'a}nchez, Ivan and Sans, Eloi and Navarro, Julia and Duran, Joan},
	journal={Int. J. Comput. Vision},
	volume={134},
	number={2},
	pages={55},
	year={2026},
	publisher={Springer}
}

@inproceedings{yang2022memory,
	title={Memory-augmented Deep Conditional Unfolding Network for Pan-sharpening},
	author={Yang, Gang and Zhou, Man and Yan, Keyu and Liu, Aiping and Fu, Xueyang and Wang, Fan},
	booktitle={Proc. IEEE Conf. Comput. Vis. Pattern Recogn.},
	pages={1788--1797},
	year={2022}
}
\newpage
\clearpage

\begin{appendices} 

\setcounter{equation}{0}
\renewcommand{\theequation}{\thesection.\arabic{equation}} 
\renewcommand{\thesection}{A}
\section{}\label{sec:suppl_A}
\subsection{Full Row Rank of $\mathcal{T}$}\label{sec:suppl_A1}
\begin{proposition}
The mapping tensor:
\begin{equation}
	\mathbb{T}=[\widetilde{\Delta}_1||\Delta_2]\in\mathbb{R}^{B\times C},
\end{equation}
constructed through the Cayley parameterization is guaranteed to have full row rank, i.e.,
\begin{equation}
	rank(\mathbb{T}) = B
\end{equation}
\end{proposition}
\noindent\emph{Proof.} According to the Cayley transformation, the first block of the mapping tensor is defined as:
\begin{equation}
	\widetilde{\Delta}_1 = (I_{d} + (\Delta_1-\Delta_1^T))^{-1}(I_{d} - (\Delta_1-\Delta_1^T)), 
\end{equation}
The Cayley transform maps the skew-Hermitian matrix $A = \Delta_1-\Delta_1^T$, satisfying $A^T = -A$ to a unitary matrix. Therefore, we havr:
\begin{align}
	\!\!\!\!\widetilde{\Delta}_1^T\widetilde{\Delta}_1\!&=\!(I_d\!+\!A)(I_d\!-\!A)^{\!-\!1}(I_{d} \!+\! A)^{\!-\!1}(I_{d} \!-\! A) \\
	&=(I_d\!+\!A)(I_d\!+\!A)^{\!-\!1}(I_{d} \!-\! A)^{\!-\!1}(I_{d} \!-\! A)\\
	&=I_d,
\end{align}
where we utilize the property that:
\begin{equation}
	\!\!(I_d\!+\!A)(I_d\!-\!A)\! =\! I_d \!-\! A^2\!= \!(I_d\! -\! A)(I_d \!+ \!A),
\end{equation}
This indicates that all singular values of $\widetilde{\Delta}_1$ are equal to one. Hence,
\begin{equation}
	\sigma_{\min}(\widetilde{\Delta}_1) = 1 \ge 0,
\end{equation}
which implies that $\widetilde{\Delta}_1$ is nonsingular and:
\begin{equation}
	rank(\widetilde{\Delta}_1) = B.
\end{equation}
The final mapping tensor is obtained by column-wise concatenation:
\begin{equation} 
	\mathbb{T}=[\widetilde{\Delta}_1||\Delta_2]\in\mathbb{R}^{B\times C}. 
\end{equation}
Since $\widetilde{\Delta}_1$ is a submatrix of $\mathbb{T}$ containing $B$ linearly independent columns, the column space of $\mathbb{T}$ necessarily contains a $B$-dimensional independent subspace. Therefore,
\begin{equation}
	rank(\mathbb{T}) \ge rank(\widetilde{\Delta}_1) = B.
\end{equation}
Meanwhile, $\mathbb{T}\in\mathbb{R}^{B\times C}$ has at most $B$ independent rows, giving
\begin{equation}
	rank(\mathbb{T}) \le  B.
\end{equation}
Combining the above inequalities yields:
\begin{equation}
	rank(\mathbb{T}) =  B.
\end{equation}
Therefore, $\mathbb{T}$ is full row rank, ensuring a well-conditioned spectral-to-latent projection without rank deficiency. 
 
\setcounter{equation}{0}
\renewcommand{\theequation}{\thesection.\arabic{equation}} 
\renewcommand{\thesection}{B}
\section{}\label{sec:suppl_B}
\subsection{Latent Diffusion Bridge Model}\label{sec:suppl_B1}
\noindent{The generalized diffusion bridge~\cite{wang2025residual} can be formulated as:}
\begin{equation}\label{eq19}
	\begin{aligned}
		&\!\!\!\!d\mathbf{z}_t\! =\!\theta_t\coth(\overline{\theta}_{t:T})(\mathbf{z}_T\!-\!\mathbf{z}_t)dt \!+\! \sqrt{2\lambda \theta_t} dw_t,
	\end{aligned}
\end{equation}
where $\overline{\theta}_{s:t} = \int_s^t \theta_\tau d\tau$ and $\omega_t$ is the standard Wiener process. $\theta_t$ denotes the noise schedule and $\lambda$ is the stationary variance. The stochastic differential equation (SDE) describes the probability transition from initial state $z_0$ to terminate state $z_T$. To convert Eq.~(\ref{eq19}) into an analytical formula, we substitute $y_t = \mathbf{z}_t - \mathbf{z}_T$ and obtain:
\begin{equation}
	dy_t = -\theta_t\coth(\overline{\theta}_{t:T})y_tdt + \sqrt{2\lambda \theta_t} dw_t,
\end{equation}
Then we use $\Psi_t\! = \! \exp(\int_0^t \!\theta_s\!\coth(\overline{\theta}_{s:T})ds)$ as the integrating factor and expand $\Psi_ty_t$ by It$\hat{o}$ formula:
\begin{equation}
	d(\Psi_ty_t) = \Psi_tdy_t + y_t d\Psi_t + d\Psi_tdy_t.
\end{equation}
$\!\Psi_t\!$ is deterministic and satisfies $d\!\Psi \!=\! \Psi\!\theta_t\!\coth(\overline{\theta}_{t:T})$. $d\Psi dy_t$ produces $(dt)^2, dtdw_t$, which are the higher order infinitesimal of $dt$ and can be omitted. Thus, we obtain:
\begin{align}
		&d(\Psi_ty_t)  =\! \Psi_t(-\theta_t\coth(\overline{\theta}_{t:T})y_tdt \!+\!  \sqrt{2\lambda \theta_t} dw_t) \nonumber \\ & + \Psi_t \theta_t\coth(\overline{\theta}_{t:T})y_tdt= \Psi_t \sqrt{2\lambda \theta_t} dw_t, \label{subeq22} 
\end{align}
Furthermore, we integrate both sides of Eq.~(\ref{subeq22}):
\begin{equation}
	\Psi_t y_t = y_0 + \int_0^t\Psi_s\sqrt{2\lambda\theta_s}dw_s.
\end{equation}
Consequently, we have:
\begin{align}
	\mathbf{z}_t &= \mathbf{z}_T  + (\mathbf{z}_0-\mathbf{z}_T)\Psi_t^{e^{-\int_0^t \theta_s\coth(\overline{\theta}_{s:T})ds}} \nonumber\\&+ \int_0^t \sqrt{2\lambda\theta_s} e^{-\int_s^t \theta_z\coth(\overline{\theta}_{z:T})dz} dw_s. \label{eq24}
\end{align}
Considering the internal integral $\int_0^t\!\theta_s \! \coth(\overline{\theta}_{s:T})\!ds$, we set $u=\overline{\theta}_{s:T}$ satisfying $du = -\theta_sds$: 
\begin{align}
	&\int_0^t\theta_s \coth(\overline{\theta}_{s:T})ds = -\int_{\overline{\theta}_{0:T}}^{\overline{\theta}_{t:T}} \coth(u) du \nonumber\\&= -\ln \vert \sinh(u) \vert \bigg\vert_{\overline{\theta}_{0:T}}^{\overline{\theta}_{t:T}} = \ln \vert \frac{\sinh(\overline{\theta}_{0:T})}{\sinh(\overline{\theta}_{t:T})} \vert,
\end{align} 
Therefore, the analytical expression of $\Psi_t$ is:
\begin{equation}
	\Psi_{t} = \frac{\sinh(\overline{\theta}_{0:T})}{\sinh(\overline{\theta}_{t:T})}.
\end{equation}
Finally, we can compute the closed-form of $\mathbf{z}_t$ in Eq.~(\ref{eq24}):
\begin{equation}\label{eq27}
	\mathbf{z}_t = \mathbf{z}_T + (\mathbf{z}_0 - \mathbf{z}_T)\frac{\sinh(\overline{\theta}_{t:T})}{\sinh(\overline{\theta}_{0:T})}+\int_0^t \sqrt{2\lambda\theta_s} \frac{\sinh(\overline{\theta}_{t:T})}{\sinh(\overline{\theta}_{s:T})} dw_s.
\end{equation}
Eq.~(\ref{eq27}) preserves the properties of diffusion bridge models, whose initial state $\mathbf{z}_0$ and final state $\mathbf{z}_t$ are determined. The formulation of variance can be further simplified as follows:
\begin{align} 
	Var[\mathbf{z}_t] &= \int_0^t 2\lambda\theta_s (\frac{\sinh(\overline{\theta}_{t:T})}{\sinh(\overline{\theta}_{s:T})})^2 ds \nonumber\\
	&= 2\lambda\sinh^2(\overline{\theta}_{t:T})\int_{\overline{\theta}_{0:T}}^{\overline{\theta}_{t:T}}-\frac{du}{\sinh^2(u)}\nonumber\\
	&= 2\lambda\sinh^2(\overline{\theta}_{t:T}) \coth(u)\bigg\vert_{\overline{\theta}_{0:T}}^{\overline{\theta}_{t:T}}\nonumber\\
	& = 2\lambda\sinh^2(\overline{\theta}_{t:T})(\coth(\overline{\theta}_{t:T}) - \coth(\overline{\theta}_{0:T}))\nonumber\\
	& = 2\lambda\sinh^2(\overline{\theta}_{t:T})(\frac{\sinh(\overline{\theta}_{0:T} - \overline{\theta}_{t:T})}{\sinh(\overline{\theta}_{0:T})\sinh(\overline{\theta}_{t:T})})\nonumber\\
	& =  2\lambda\frac{\sinh(\overline{\theta}_{0:t})\sinh(\overline{\theta}_{t:T})}{\sinh(\overline{\theta}_{0:T})} 
\end{align}
The expectation and variance of Eq.~(\ref{eq27}) are summarized as follows:
\begin{equation}
	\!\mathbb{E}[\mathbf{z}_t] = \mathbf{z}_T \!+\! (\mathbf{z}_0\!-\!\mathbf{z}_T)\Theta_t, \quad \Theta_t \! \eqcolon \! \frac{\sinh(\overline{\theta}_{t:T})}{\sinh(\overline{\theta}_{0:T})}  \label{eq28}	
\end{equation}
\begin{equation}
	\!Var[\mathbf{z}_t]  = 2\lambda\frac{\sinh(\overline{\theta}_{0:t})\sinh(\overline{\theta}_{t:T})}{\sinh(\overline{\theta}_{0:T})} \coloneqq \Sigma_t^2. \label{eq29}\\
\end{equation}
A common forward process in this framework can be determined as follows:
\begin{equation}
	\mathbf{z}_{t-1} = \mathbf{z}_T + (\mathbf{z}_0 - \mathbf{z}_T)\Theta_{t-1} + \Sigma_{t-1}\epsilon_{t-1},
\end{equation}
\begin{equation}
	\mathbf{z}_{t} = \mathbf{z}_T + (\mathbf{z}_0 - \mathbf{z}_T)\Theta_{t} + \Sigma_{t}\epsilon_{t}.
\end{equation}
We assume the reverse process follows a Gaussian distribution:
\begin{equation}\label{eq65}
	\begin{aligned}
		& \mathbf{z}_{t-1} = \kappa_t \mathbf{z}_t + \eta_t \mathbf{z}_T + \gamma_{t} \mathbf{z}_0 + \dot{\sigma_t}\epsilon_{t}\\
		&= \kappa_t ( \mathbf{z}_T + (\mathbf{z}_0 - \mathbf{z}_T)\Theta_{t} + \Sigma_{t}\epsilon_{t}) + \eta_t \mathbf{z}_T + \gamma_t \mathbf{z}_0 + \dot{\sigma_t}\epsilon_{t}\\
		&=(\kappa_t + \eta_t - \kappa_t\Theta_{t})\mathbf{z}_T + (\kappa_t\Theta_{t}+\gamma_{t})\mathbf{z}_0 + (\kappa_t^2\Sigma_{t}^2 + \dot{\sigma_t}^2)^{\frac{1}{2}}\epsilon_{t},
	\end{aligned}
\end{equation}
we have:
\begin{align}
	\kappa_t + \eta_t - \kappa_t\Theta_{t}& = 1 - \Theta_{t-1},\\
	\kappa_t\Theta_{t}+\gamma_{t} &= \Theta_{t-1},\\
	\Sigma_{t-1}^2 &= \kappa_t^2\Sigma_{t}^2 + \dot{\sigma_t}^2.
\end{align}
By setting $\dot{\sigma_t} = 0$:
\begin{equation}
	\begin{aligned}
		\kappa_t = {\frac{\Sigma_{t-1}}{\Sigma_{t}}}&,\gamma_t=\Theta_{t-1} - \Theta_{t} {\frac{\Sigma_{t-1}}{\Sigma_{t}}}\\
		\eta_t & = 1 - \Theta_{t-1} - (1- \Theta_{t}){\frac{\Sigma_{t-1}}{\Sigma_{t}}},
	\end{aligned}
\end{equation}
substituting into Eq.~(\ref{eq65}):
\begin{align}
	\!\! &\mathbf{z}_{t\!-\!1} \!=\! \mathbf{z}_T \!+\! {\frac{\Sigma_{t\!-\!1}}{\Sigma_{t}}}(\mathbf{z}_t \!-\!\mathbf{z}_T) \!+\! (\Theta_{t\!-\!1} \!-\! \Theta_{t} {\frac{\Sigma_{t\!-\!1}}{\Sigma_{t}}})(\widehat{\mathbf{z}}_0^{t} \!-\!\mathbf{z}_T), \nonumber \\
	& \!=\! \mathbf{z}_T \!+\! \frac{\Theta_{t\!-\!1}}{\Theta_t}(\mathbf{z}_t \!-\! \mathbf{z}_T) \!-\! ({\frac{\Theta_{t\!-\!1}}{\Theta_{t}}\Sigma_{t} \!-\! \Sigma_{t\!-\!1}}) \widehat{\epsilon}_t.
\end{align}  
This concludes the derivations in Sec.III.

\setcounter{equation}{0}
\renewcommand{\theequation}{\thesection.\arabic{equation}} 
\renewcommand{\thesection}{C}  
\section{}\label{sec:suppl_C}
\subsection{Bridge Posterior Sampling}\label{sec:suppl_C1}
\begin{proposition}\label{Appendix:Sec_C-proposition2}
	\textit{Jensen gap upper bound}~\cite{gao2017bounds}. Define the absolute centered moment as $m_p := \sqrt[p]{\mathbb{E}[\|X-\mu \|^p]}$, and the mean as $\mu = \mathbb{E}[X]$. Assume that for $\alpha > 0$, there exists a positive number $K$ such that for any $x \in \mathbb{R}$,
	$|f(c) - f(\mu)| \leq K|x - \mu|^\alpha$. Then:
	\begin{equation}\label{Appendix:Sec_C-eq8}
		\begin{aligned}
			\mathbb{E}[|f(X) - & f(\mathbb{E}[X])|] \leq \int |f(x) - f(\mu)| dq(X)\\
			&\leq K \int |x - \mu|^\alpha \, dq(X) \leq Mm_\alpha^\alpha, 
		\end{aligned}
	\end{equation}
	where $M$ is an upper bound estimator constant that can be taken as $K$ or other constants related to the function $f$ and the distribution $q(X)$.
\end{proposition}
\begin{lemma}\label{Appendix:Sec_C_lemma2}
	Let $\phi(\cdot)$ be a univariate Gaussian density function with mean $\mu$ and variance $\sigma^2$, there exists a Lipschitz constant $L$ such that:
	\begin{equation}\label{Appendix:Sec_C-eq9}
		|\phi(x) - \phi(y)| \leq L|x - y|,
	\end{equation}
	where $L = \frac{1}{\sqrt{2\pi\sigma^2}} \exp\left(-\frac{1}{2\sigma^2}\right)$.
\end{lemma}
\textit{Proof.} As $\phi^{\prime}$ is continuous and bounded, we use the mean value theorem to get:
\begin{equation}\label{Appendix:Sec_C-eq10}
	\forall (x,y)\in \mathbb{R}^2, |\phi(x) - \phi(y)| \leq \|\phi^{\prime}\|_\infty |x - y|, 
\end{equation}
Since $L$ is the Lipschitz constant, we have that $L \le \|\phi^{\prime}\|_\infty$. Taking the limit $y \rightarrow x$ gives $\|\phi^{\prime}\| \le L$, and thus $\|\phi^{\prime}\|_\infty \le L$. Hence:
\begin{equation}\label{Appendix:Sec_C-eq11}
	L = \|\phi^{\prime}\|_\infty  = \| -\frac{x - \mu}{\sigma^2} \phi(x) \|_\infty.
\end{equation}
Since the derivative of $\phi^{\prime}$ is given as:
\begin{equation}\label{Appendix:Sec_C-eq12}
	\phi^{\prime\prime}=\sigma^{-2}(1 - \sigma^{-2}(x-\mu)^2) \phi(x),
\end{equation}
Setting $\phi^{\prime\prime}=0$, we get $\sigma^{-2}(1 - \sigma^{-2}(x - \mu)^2)=0$, which gives $x = \mu\pm\sigma$, and we have:
\begin{equation}\label{Appendix:Sec_C-eq13}
	L = \|\phi^{\prime}\|_\infty  = \frac{e^{-1/2\sigma^2}}{\sqrt{2\pi \sigma^2}}.
\end{equation}
\begin{lemma}\label{Appendix:Sec_C_lemma3}
	Let $\phi(\cdot)$ be an isotropic multivariate Gaussian density function with mean $\mu$ and variance $\sigma^2\boldsymbol{I}$. There exists a constant $L$ such that $\forall x,y\in \mathbb{R}^d$:
	\begin{equation}\label{Appendix:Sec_C-eq14}
		|\phi(x) - \phi(y)| \leq L|x - y|,
	\end{equation}
	where $L = \frac{d}{\sqrt{2\pi\sigma^2}} \exp\left(-\frac{1}{2\sigma^2}\right)$
\end{lemma}
\textit{Proof.}
\begin{equation}\label{Appendix:Sec_C-eq15}
	\begin{aligned}
		&\|\phi(x) - \phi(y)\| \le \max_{z} \|\nabla_z \phi(z)\|\cdot \|x-y\| \\&=  \frac{d}{\sqrt{2\pi\sigma^2}} \exp(-\frac{1}{2\sigma^2}) \cdot \|x-y\|,
	\end{aligned}
\end{equation}
each element of $d$ dimensions $\nabla_z \phi(z)$ are bounded by $\frac{1}{\sqrt{2\pi\sigma^2}} \exp(-\frac{1}{2\sigma^2})$.

\begin{proposition}\label{Appendix:Sec_C_theorem1}
	For the measurement model $\mathbf{z}_T = \mathbb{T} \mathcal{D} \mathbb{T}^{\ast}\mathbf{z}_0 +  \mathbf{n}$ in linear verse problem with $ mathbf{n}\sim \mathcal{N}(0,\sigma^2\boldsymbol{I})$, we have:
	\begin{equation}\label{Appendix:Sec_C-eq16}
		q(\mathbf{z}_T\vert \mathbf{z}_t) \simeq q(\mathbf{z}_T\vert \widehat{\mathbf{z}}_0),
	\end{equation}
	where the approximation error can be quantified with the Jensen gap, which is upper bounded by:
	\begin{equation}\label{Appendix:Sec_C-eq17}
		\!\!\!\mathcal{J}(\sigma,M) \! \le \! \frac{d}{\sqrt{2\pi\sigma^2}} \exp\left(-\frac{1}{2\sigma^2}\right)||\nabla_{u} \mathbb{T}(u)||_2 M,
	\end{equation}
	where $u:=\mathcal{D}  \mathbb{T}^{\ast} \widehat{\mathbf{z}}_0^t$ and $\|\nabla_u \mathbb{T}(u)\|:=\max_u \|\nabla_{u}\mathbf{z}_0 \mathbb{T}(u)\|$ is Lipschitz constant. $M:=\int \|\mathcal{D} \mathbb{T}^{\ast}\mathbf{z}_0 - \mathcal{D} \mathbb{T}^{\ast}\widehat{\mathbf{z}}_0\|q(\mathbf{z}_0\vert \mathbf{z}_t)d\mathbf{z}_0$.
\end{proposition}
\textit{Proof.} The measurement model $\mathbf{z}_T = \mathbb{T}(\mathcal{D}\mathbb{T}^{\ast}\mathbf{z}_0) + n$ can be formulated as:
\begin{equation}\label{Appendix:Sec_C-eq18}
	\begin{aligned}
		&q(\mathbf{z}_T \vert \mathbb{T}^{\ast}\mathbf{z}_0) \sim \mathcal{N}(\mathbb{T}\mathcal{D}\mathbb{T}^{\ast}\mathbf{z}_0,\sigma^2\boldsymbol{I}) \\&= \frac{1}{\sqrt{2\pi \sigma^2}} \exp\big(-\frac{(\mathbf{z}_T-\mathbb{T}\mathcal{D}\mathbb{T}^{\ast}\mathbf{z}_0)^2}{2\sigma^2}\big)\\
		& =  \frac{1}{\sqrt{2\pi \sigma^2}} \exp\big(\frac{\mathbb{T}\mathcal{D}\mathbb{T}^{\ast}\mathbf{z}_0}{\sigma^2}\mathbf{z}_T-\frac{(\mathbb{T}\mathcal{D}\mathbb{T}^{\ast}\mathbf{z}_0)^2}{2\sigma^2}\big)\exp\big(-\frac{\mathbf{z}_T^2}{2\sigma^2}\big) \\
		& = \bigg[ \frac{1}{\sqrt{2\pi \sigma^2}} \exp\big(-\frac{\mathbf{z}_T^2}{2\sigma^2}\big) \bigg] \exp\big(\frac{\mathbb{T}\mathcal{D}\mathbb{T}^{\ast}\mathbf{z}_0}{\sigma^2}\mathbf{z}_T-\frac{(\mathbb{T}\mathcal{D}\mathbb{T}^{\ast}\mathbf{z}_0)^2}{2\sigma^2}\big).
	\end{aligned}
\end{equation}

\begin{figure*}[t]
	\renewcommand\thefigure{S1}
	\centering
	\includegraphics[width=0.99\linewidth]{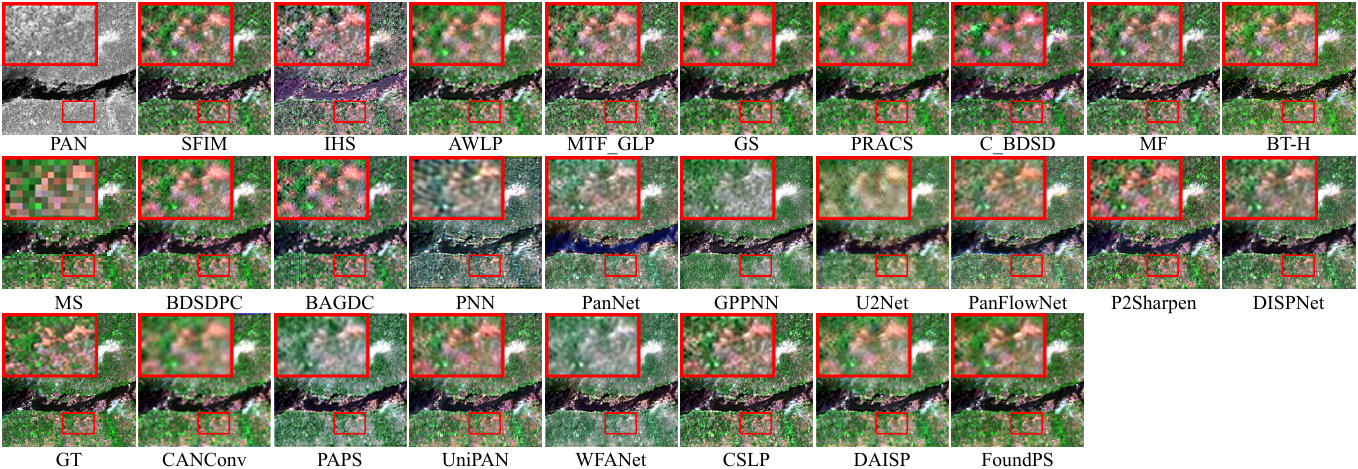}
	\caption{4-band visual comparisons of reduced scale on PSBench. Zoom in for the best.}
	\label{exp:append_reduced_4} 

	\renewcommand\thefigure{S2} 
	\centering
	\includegraphics[width=0.99\linewidth]{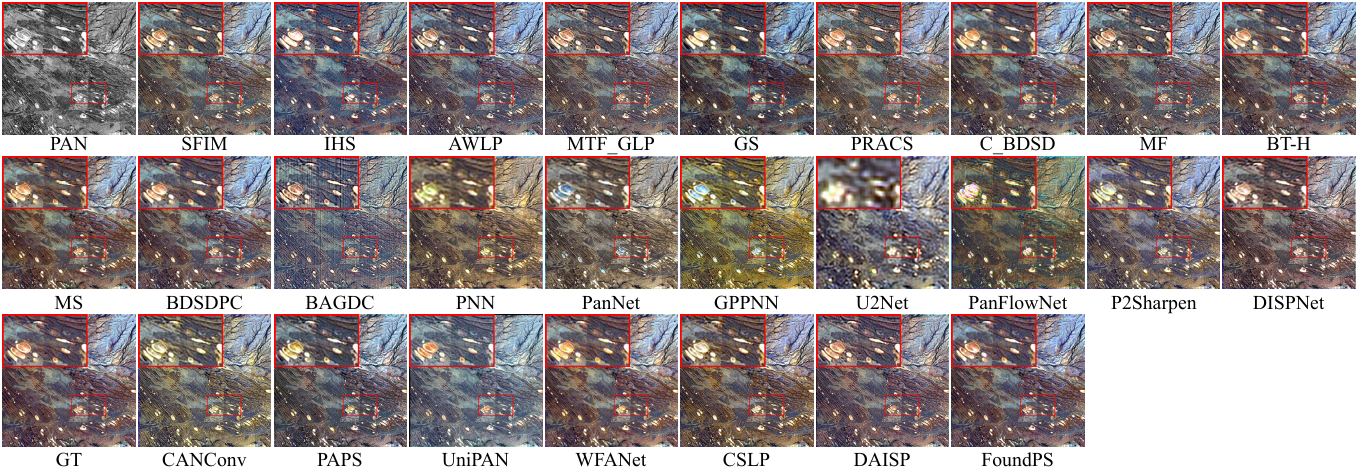}
	\caption{7-band visual comparisons of reduced scale on PSBench. Zoom in for the best.}
	\label{exp:append_reduced_7} 
\end{figure*}

\noindent In order to exploit the measurement model, we factorize $q(\mathbf{z}_T\vert \mathbf{z}_t)$ as follows:
\begin{equation}\label{Appendix:Sec_C-eq20}
	\begin{aligned}
		\!\!q(\mathbf{z}_T\vert \mathbf{z}_t) \!=\! \int\! q(\mathbf{z}_T\vert \mathbf{z}_0)q(\mathbf{z}_0\vert \mathbf{z}_t) d \mathbf{z}_0 \!=\! \mathbb{E}_{\mathbf{z}_0\sim q(\mathbf{z}_0\vert \mathbf{z}_t)}[f(\mathbf{z}_0)],
	\end{aligned}
\end{equation}
here, $f(\cdot):=h(\mathbb{T}(\cdot))$, where $\mathbb{T}$ is the measurement operator and $h(\cdot)$ is the multivariate normal distribution with mean $\mathbb{T}(\mathbf{z}_0)$ and the covariance $\sigma^2\boldsymbol{I}$. Therefore, we have:
\begin{align}
	&\!\!\mathcal{J}(f\!,\!q(\mathbf{z}_0\vert \mathbf{z}_t))  \!=\! \vert \mathbb{E}[f(\mathbb{T}^{\ast}\mathbf{z}_0)] \!-\! f(\mathbb{E}[\mathbb{T}^{\ast}\mathbf{z}_0]) \vert \! \\
	&= \vert \mathbb{E}[f(\mathbb{T}^{\ast}\mathbf{z}_0)] - f(\mathbb{T}^{\ast}\widehat{\mathbf{z}}_0) \vert, \label{Appendix:Sec_C-eq21}\\&= \vert \mathbb{E}[h(\mathbb{T}(\mathbb{T}^{\ast}\mathbf{z}_0))] - h(\mathbb{T}(\mathbb{T}^{\ast}\widehat{\mathbf{z}}_0)) \vert, \nonumber \\
	&\le\! \int \!\!\vert  h(\mathbb{T}(\mathbb{T}^{\ast}\!\mathbf{z}_0)) \!-\! h(\mathbb{T}(\mathbb{T}^{\ast}\widehat{\mathbf{z}}_0))  \vert\! d Q(\mathbb{T}^{\ast}\mathbf{z}_0\vert \mathbf{z}_t), \nonumber\\
	&\le \frac{d}{\sqrt{2\pi\sigma^2}}  e^{-\frac{1}{2\sigma^2}}  \int  \|\mathbb{T}\mathcal{D}(\mathbb{T}^{\ast}\mathbf{z}_0) - \mathbb{T}\mathcal{D}(\mathbb{T}^{\ast}\widehat{\mathbf{z}}_0)\| d Q(\mathbf{z}_0\vert \mathbf{z}_t), \nonumber\\
	&\le \frac{d}{\sqrt{2\pi\sigma^2}} e^{-\frac{1}{2\sigma^2}} \|\nabla_{u} \mathbb{T}(u) \| \int \|\mathcal{D}\mathbb{T}^{\ast}\mathbf{z}_0 - \mathcal{D}\mathbb{T}^{\ast}\widehat{\mathbf{z}}_0\| d Q(\mathbf{z}_0\vert \mathbf{z}_t), \nonumber\\
	&\le \frac{d}{\sqrt{2\pi\sigma^2}} e^{-\frac{1}{2\sigma^2}} \|\nabla_{u} \mathbb{T}(u) \| M,\nonumber\\
	&\le \frac{d}{\sqrt{2\pi\sigma^2}} e^{-\frac{1}{2\sigma^2}} \|\nabla_{\mathbf{z}_t} (\mathbb{T} \circ \mathcal{D}  \circ \mathbb{T}^{\ast}) (\widehat{\mathbf{z}}_0^t(z_t))\| \nonumber\\&\qquad\qquad\qquad\cdot \int \|\mathbf{z}_0 - \widehat{\mathbf{z}}_0\| d Q(\mathbf{z}_0\vert \mathbf{z}_t),\nonumber
\end{align} 
where $d Q(\mathbf{z}_0\vert \mathbf{z}_t) = q(\mathbf{z}_0\vert \mathbf{z}_t) d \mathbf{z}_0$, $\|\nabla_{u} \mathbb{T}(u) \|$ and $\|\nabla_{\mathbf{z}_t} (\mathbb{T} \circ \mathcal{D}  \circ \mathbb{T}^{\ast}) (\widehat{\mathbf{z}}_0^t(z_t))\|$ are Lipschitz constants. Note that $M$ is finite for most of the distribution in practice, which can be considered as the generalized absolute distance between the observed reference and estimated data. $\mathcal{J}(\sigma,M)$ can approach 0 as $\sigma\!\rightarrow\!0$ or $\infty$, suggesting that the approximation errors reduce when the measurement noise is extremely small or large. Specifically, if the predictions of $\widehat{\mathbf{z}}_0^t$ are accurate, the upper bound $\mathcal{J}(\sigma,M)\vert_{\sigma \!\rightarrow\! 0,M \!\rightarrow\! 0}$ shrinks due to the low variance and distortion. Oppositely, if the predictions lack accuracy, the upper bound of the $\mathcal{J}(\sigma,M)\vert_{\sigma \!\rightarrow\! \infty, M \!\rightarrow\! M_{max}}$ shrinks as well for large variance and limited distortion. For the inverse inference, we can employ bridge posterior sampling to guide the current state:
\begin{align}
	\widehat{\mathbf{z}}_0^t &:= \widehat{\mathbf{z}}_0^t + \eta_{z} \nabla_{\mathbf{z}_t} \|\mathbf{z}_T - \mathbb{T}(\mathbb{T}^{\ast}\widehat{\mathbf{z}}_0^t\downarrow)\|_2, \\
	\widehat{\epsilon}_t &:= \widehat{\epsilon}_t + \eta_{\epsilon} \nabla_{\mathbf{z}_t} \|\mathbf{z}_T - \mathbb{T}(\mathbb{T}^{\ast}\mathbf{z}_0\downarrow)\|_2.
\end{align} 
Building on these theoretical foundations, BPS functions as a training-free guidance mechanism that enables flexible control over the fusion effect via a tunable weight $\eta_{(\cdot)}$. 
\begin{figure*}[t]
	
	\renewcommand\thefigure{S3}
	\centering
	\includegraphics[width=0.99\linewidth]{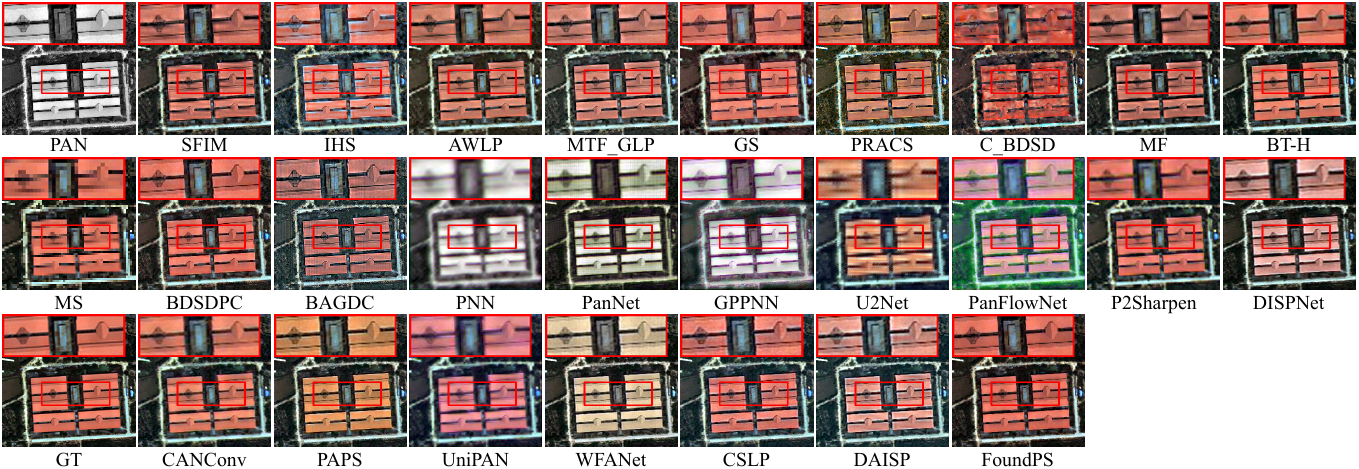}
	\caption{8-band visual comparisons of reduced scale on PSBench. Zoom in for the best.}
	\label{exp:append_reduced_8}
	
	\renewcommand\thefigure{S4}
	\centering
	\includegraphics[width=0.99\linewidth]{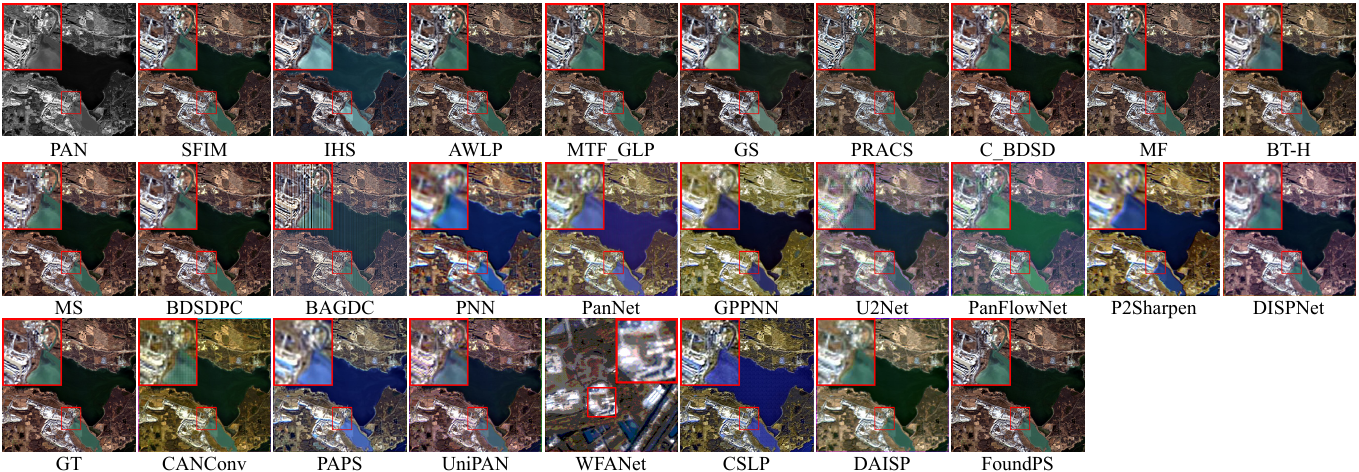}
	\caption{10-band visual comparisons of reduced scale on PSBench. Zoom in for the best.}
	\label{exp:append_reduced_10}

	\renewcommand\thefigure{S5}
	\centering
	\includegraphics[width=0.99\linewidth]{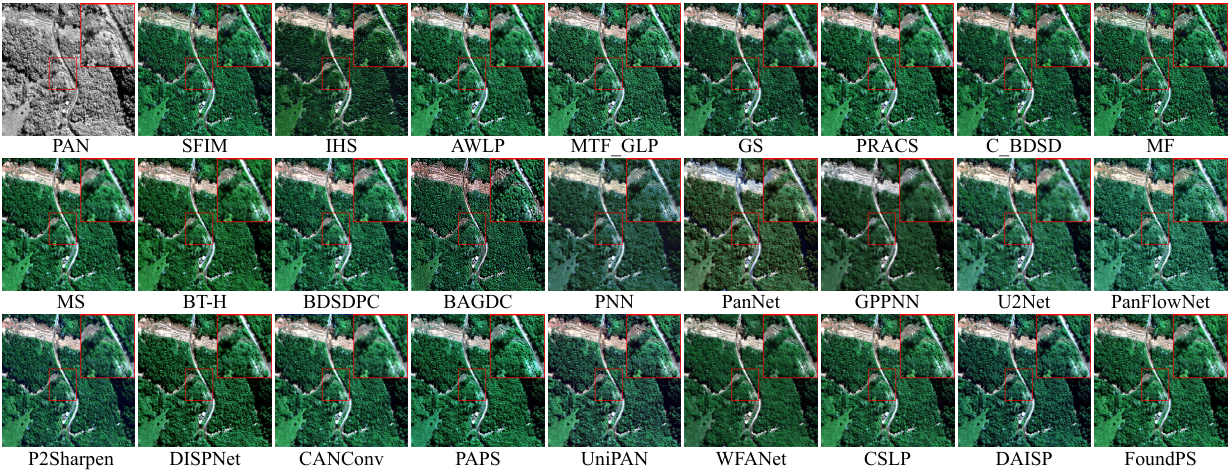}
	\caption{4-band visual comparisons of full scale on PSBench. Zoom in for the best.}
	\label{exp:append_full_4}
\end{figure*}

\begin{figure*}[t]
	\renewcommand\thefigure{S6}
	\centering
	\includegraphics[width=0.99\linewidth]{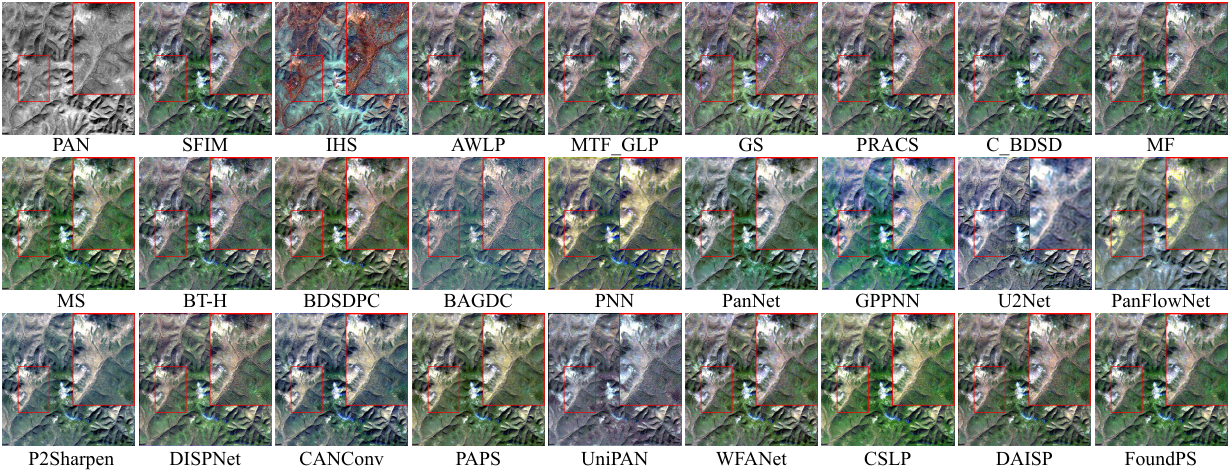}
	\caption{7-band visual comparisons of full scale on PSBench. Zoom in for the best.}
	\label{exp:append_full_7} 
	
	\renewcommand\thefigure{S7}
	\centering
	\includegraphics[width=0.99\linewidth]{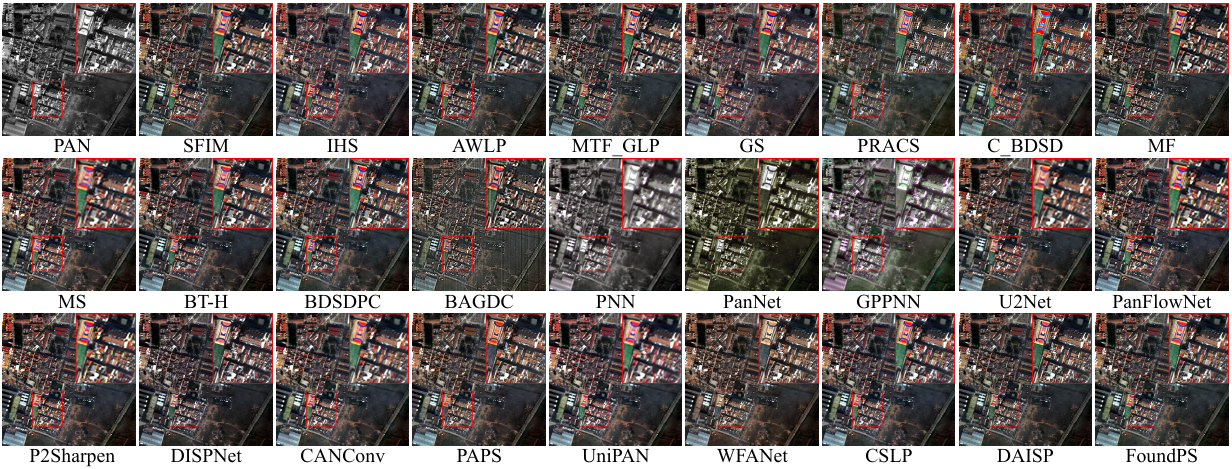}
	\caption{8-band visual comparisons of full scale on PSBench. Zoom in for the best.}
	\label{exp:append_full_8} 

	\renewcommand\thefigure{S8}
	\centering
	\includegraphics[width=0.99\linewidth]{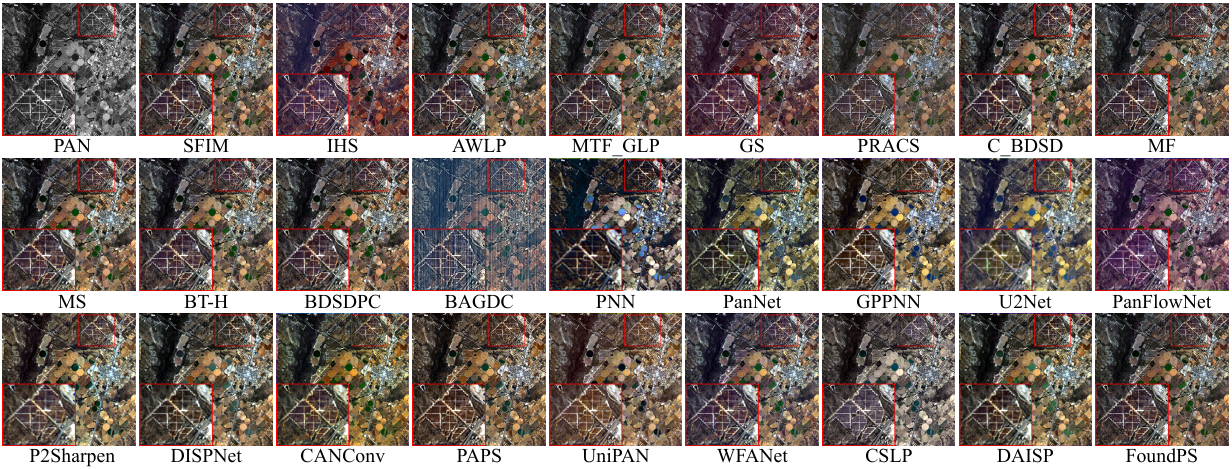}
	\caption{10-band visual comparisons of full scale on PSBench. Zoom in for the best.}
	\label{exp:append_full_10}
\end{figure*}

\begin{figure*}[t]
	\renewcommand\thefigure{S9}
	\centering
	\includegraphics[width=0.99\linewidth]{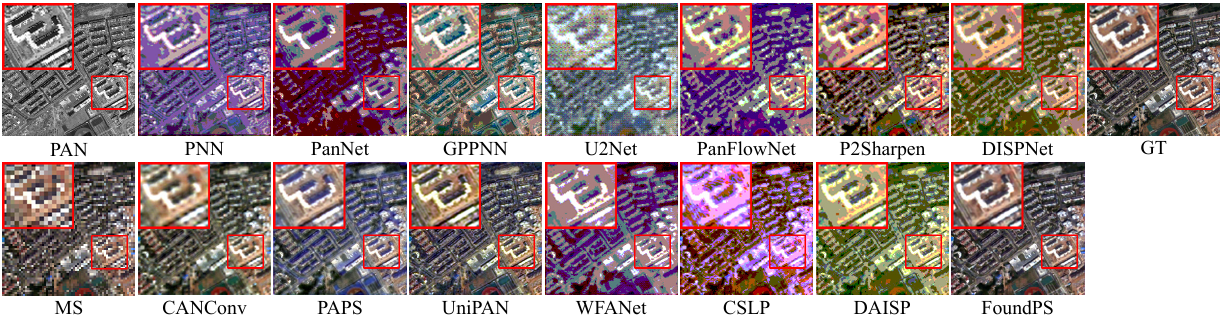}
	\caption{Visual comparisons of reduced scale on Quickbird.}
	\label{exp:genQB_reduced}
	
	\renewcommand\thefigure{S10}
	\includegraphics[width=0.99\linewidth]{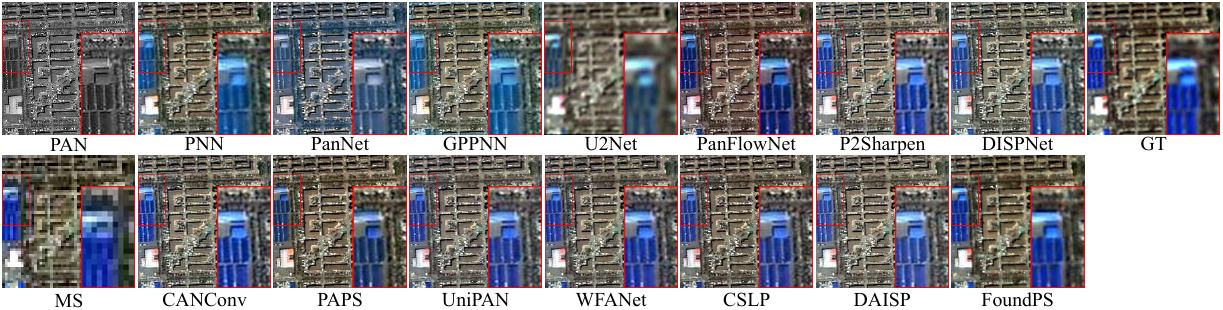}
	\caption{Visual comparisons of reduced scale on Jinlin.}
	\label{exp:genJL_reduced}
	
	\renewcommand\thefigure{S11}	
	\includegraphics[width=0.99\linewidth]{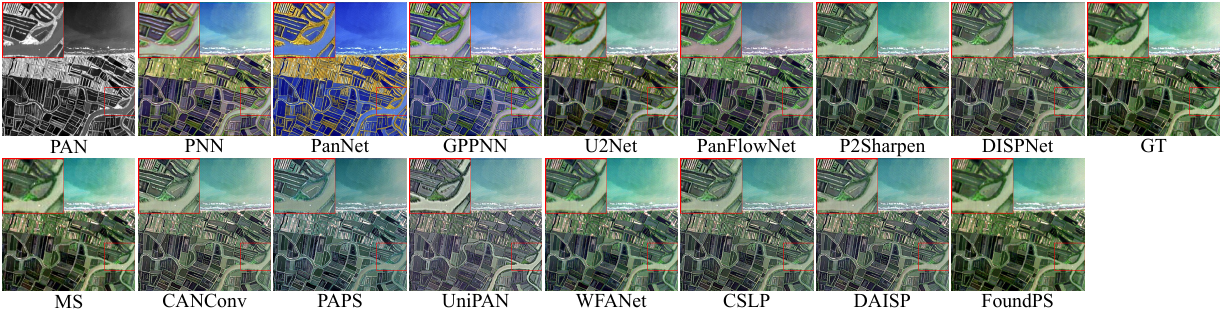}
	\caption{Visual comparisons of reduced scale on skysat.}
	\label{exp:genSS_reduced}
	
	\renewcommand\thefigure{S12}
	
	\centering
	\includegraphics[width=0.99\linewidth]{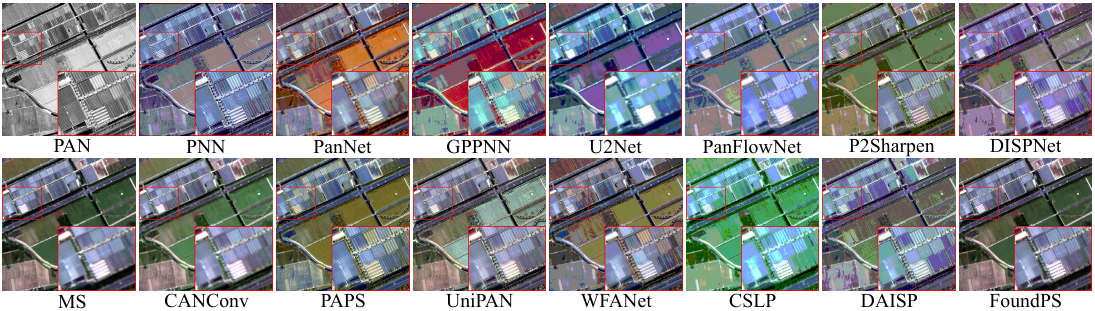}
	\caption{Visual comparisons of full scale on Quickbird.}
	\label{exp:genQB_full}
\end{figure*}

\begin{figure*}[t] 
	
	\renewcommand\thefigure{S13}
	
	\centering
	\includegraphics[width=0.99\linewidth]{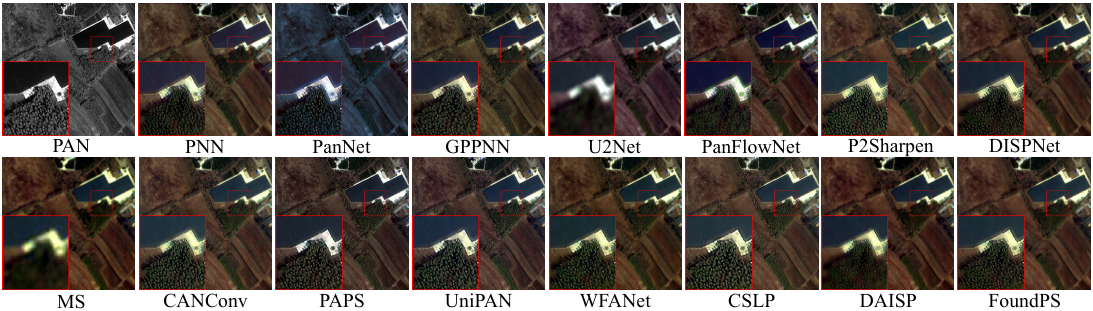}
	\caption{Visual comparisons of full scale on Jinlin.}
	\label{exp:genJL_full}
	
	\renewcommand\thefigure{S14}
	
	\centering
	\includegraphics[width=0.99\linewidth]{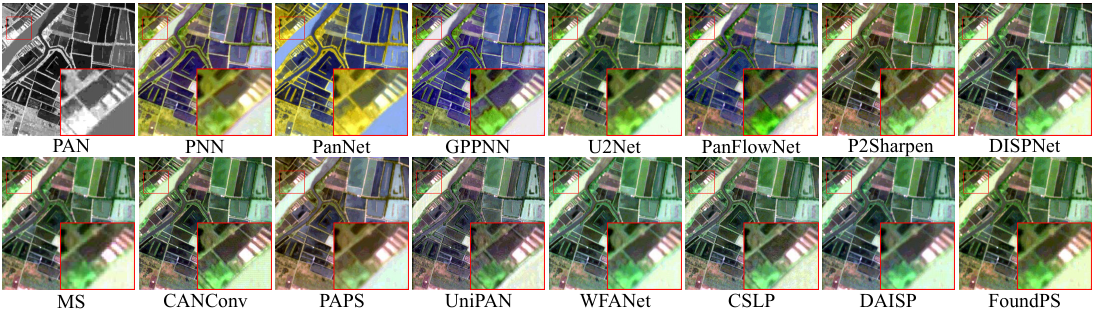}
	\caption{Visual comparisons of full scale on skysat.}
	\label{exp:genSS_full}
	
	\renewcommand\thefigure{S15}
	\centering
	\includegraphics[width=0.97\linewidth]{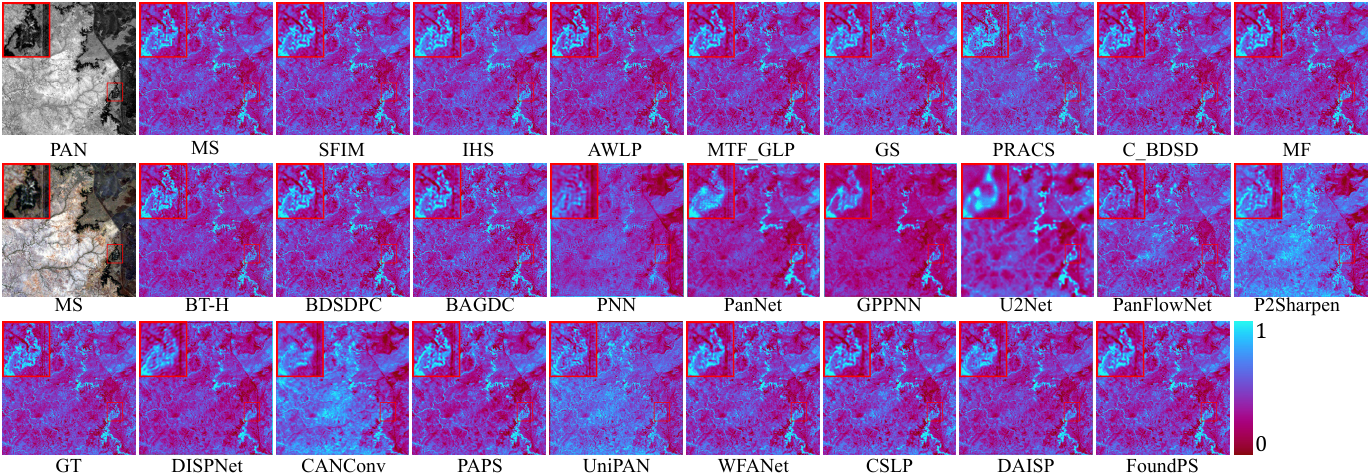}
	\caption{Remote sensing application of NDWI.}
	\label{exp:ndwi}
	\renewcommand\thefigure{S16}
	
	\centering
	\includegraphics[width=0.97\linewidth]{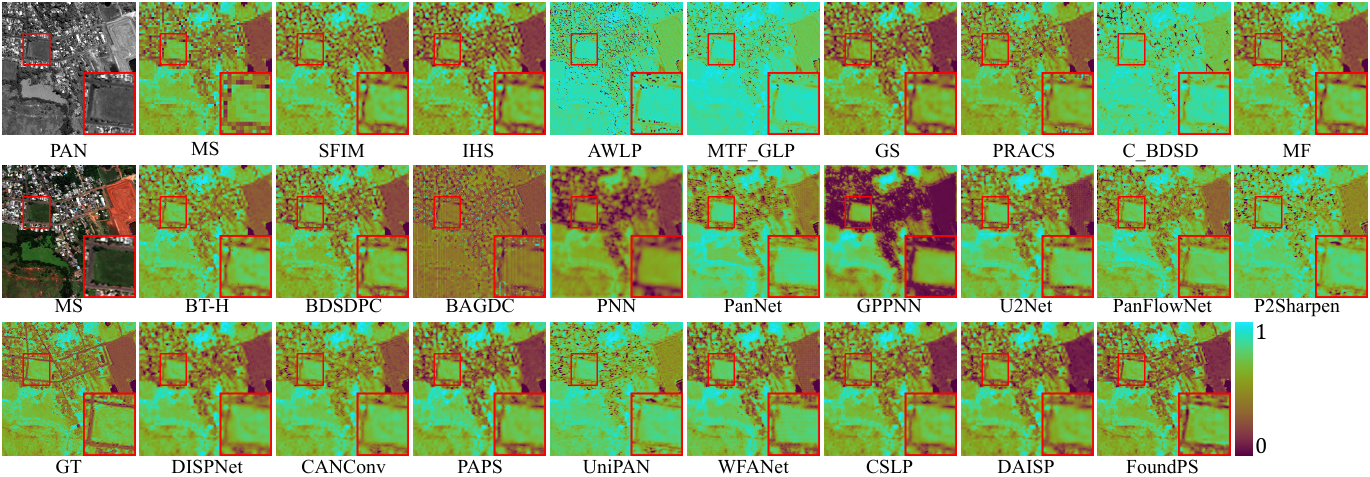}
	\caption{Remote sensing application of NDRE.}
	\label{exp:ndre}
\end{figure*}

\setcounter{equation}{0}
\renewcommand{\theequation}{\thesection.\arabic{equation}} 
\renewcommand{\thesection}{D}  
\section{}\label{sec:suppl_D}
\subsection{Visual Results on PSBench}\label{sec:suppl_1}
We show the visualization results of four representative pansharpening tasks on both reduced and full scale in Fig.~\ref{exp:append_reduced_4}, Fig.~\ref{exp:append_reduced_7}, Fig.~\ref{exp:append_reduced_8}, Fig.~\ref{exp:append_reduced_10}, Fig.~\ref{exp:append_full_4}, Fig.~\ref{exp:append_full_7}, Fig.~\ref{exp:append_full_8}, and Fig.~\ref{exp:append_full_10} to further demonstrate our superiority.

\subsection{Zero-shot Visual Results}\label{sec:suppl_2}
We conduct the zero-shot generalization experiments on unseen scenes and satellites. The additional visual results on reduced scale are presented in Fig~\ref{exp:genQB_reduced}, Fig~\ref{exp:genJL_reduced}, and Fig~\ref{exp:genSS_reduced}. And the corresponding results on full scale are presented in Fig~\ref{exp:genQB_full}, Fig~\ref{exp:genJL_full}, and Fig~\ref{exp:genSS_full}.

\subsection{NDXI Visual Results}\label{sec:suppl_3}
We carry out the verification on classical remote sensing applications. The additional visual comparisons of the NDWI and NDRE maps are presented in Fig.~\ref{exp:ndwi} and Fig.~\ref{exp:ndre}, respectively.

\end{appendices}


\end{document}